\documentclass{article}
\usepackage[numbers]{natbib}

\usepackage[utf8]{inputenc} 
\usepackage[T1]{fontenc}    
\usepackage{hyperref}       
\usepackage{url}            
\usepackage{booktabs}       
\usepackage{amsfonts}       
\usepackage{nicefrac}       
\usepackage{microtype}      
\usepackage{lipsum}
\usepackage{fancyhdr}       
\usepackage{graphicx}       
\graphicspath{{media/}}     
\usepackage{lscape} 
\usepackage{multirow}
\usepackage{subcaption}
\usepackage{amsmath}   
 \usepackage{wrapfig} %
 \usepackage[symbol]{footmisc}

\usepackage{graphicx}   
\usepackage[algoruled]{algorithm2e}

\usepackage{PRIMEarxiv}

\usepackage{dblfloatfix}
\usepackage{tabularx}
\usepackage{floatrow}

\usepackage{paralist}

\usepackage{amsthm} %
\usepackage{amsmath}
\usepackage{amssymb} 

\usepackage{graphicx}   

\usepackage{subcaption}
 
\usepackage[algoruled]{algorithm2e}

 \usepackage{paralist}

\usepackage{dsfont} %
 \usepackage{wrapfig} %
 \usepackage{lipsum} %
 \usepackage{floatrow} %
 \newfloatcommand{capbtabbox}{table}[][\FBwidth]

\usepackage{fancyhdr}
\usepackage{hyperref}
\let\svthefootnote\thefootnote
\newcommand\freefootnote[1]{%
  \let\thefootnote\relax%
  \footnotetext{#1}%
  \let\thefootnote\svthefootnote%
}

\usepackage{tikz}
\usetikzlibrary{shapes,decorations,arrows,calc,arrows.meta,fit,positioning}
\tikzset{
    -Latex, auto, node distance = 0.5 cm and 0.5 cm, semithick,
    state/.style = {circle, draw, minimum width = 0.7 cm},
    const/.style = {minimum width = 0.7 cm},
    inter/.style = {rectangle, draw, minimum width = 0.7 cm, minimum height = 0.7 cm},
    point/.style = {circle, draw, inner sep = 0.04cm, fill, node contents = {}},
    bidirected/.style = {Latex-Latex,dashed},
    el/.style = {inner sep=2pt, align=left, sloped}
}

\definecolor{asparagus}{rgb}{0.01, 0.75, 0.24}

\definecolor{darkspringgreen}{rgb}{0.09, 0.45, 0.27}
\definecolor{darkred}{rgb}{0.55, 0.0, 0.0}

\newcommand{\Gcal}{ {  \mathcal{G} } }
\newcommand{\Mcal}{ {  \mathcal{M} } }

\newcommand{\Vcal}{ {\mathcal{V}} }
\newcommand{\Ecal}{ {\mathcal{E}} }

\newcommand{\Rbb}{ {\mathbb{R}} }

\newcommand{\bld}[1]{\mathbf{#1}}

\newcommand{\Z}{ { \mathbf{Z} } }
\newcommand{\F}{ { \mathbf{F} } }
\newcommand{\A}{ { \mathbf{A} } }
\newcommand{\X}{ { \mathbf{X} } }
\newcommand{\U}{ { \mathbf{U} } }
\newcommand{\E}{ { \mathbf{E} } }

\newcommand{\etab}{ { \boldsymbol{\eta} } }
\newcommand{\dec}{ {dec } }
\newcommand{\enc}{ {enc } }

\newcommand{\Ical}{ { {\mathcal{I}}} }
\newcommand{\Ibb}[1]{ \mathbb{I}_{\{#1\}} }

\newcommand{\xF}{  \mathbf{x}^{F}  }
\newcommand{\xI}{ {{ \mathbf{x^{\Ical}} }} }
\newcommand{\xCF}{  \mathbf{x}^{CF}  }
\newcommand{\xCFh}{  \hat{\mathbf{x}}^{CF}}

\newcommand{\x}{ { \mathbf{x} } }

\newcommand{\XI}{ { { { \mathbf{X}_{ \Ical}} } } }

\newcommand{\ZIt}{ { \tilde{\Z}^{\Ical} } }
\newcommand{\ZFt}{ { \tilde{\Z}^{\F} } }
\newcommand{\Zt}{ { \tilde{\Z}} }

\newcommand{\blda}{ { {\boldsymbol \alpha} } }

\newcommand{\FI}{ { { { \mathbf{F}^{\Ical}}}}}
\newcommand{\fx}{ { { { \tilde{f}}}}}

\newcommand{\Fx}{ { { { \tilde{\mathbf{F}}}}}}
\newcommand{\FxI}{ { { { \tilde{\mathbf{F}}^{ \Ical}}}}}

\newcommand{\ani }{ { { { \mathrm{an}(i) } } }}
\newcommand{\ansi }{ { { { \mathrm{an}^{*}}(i) } } }
\newcommand{\pai }{ { { { \mathrm{pa}(i) } } }}
\newcommand{\pasi }{ { { { \mathrm{pa}^{*}}(i) } } }
\newcommand{\Ncal}{ {  \mathcal{N} } }

\newcommand{\Ncali}{ {  {\mathcal{N}_i} } }
\newcommand{\AI}{ { { { \mathbf{A}^{ \Ical}} } } }
\newcommand{\Aij}{ { { { {A}_{ij}} } } }

\newcommand{\qforall}{{ { { { \text{ }{\forall} } } } }}
 
\newcommand{\nin}{ { { \not \in}  } }

\newcommand{\Ex}{ { {\mathbb{E}} } }

\newcommand{\pt}{ { p_\theta } }
\newcommand{\qp}{ { q_\phi} }
\newcommand{\pZ}{ { p(\Z) } }

\newcommand{\fm}{ f^{m} }
\newcommand{\fa}{  f^{a} }
\newcommand{\fu}{ f^{u} }

\newcommand{\name}{VACA}
\newcommand{\nameb}{{(VACA) }} %
\newcommand{\mcvae}{{MultiCVAE}}
\newcommand{\carefl}{{CAREFL}}

\definecolor{olive}{rgb}{0.00, 0.30, 0.25}
\definecolor{myyellow}{rgb}{1.0, 0.75, 0.03}

\newcommand{\diameter}{ { \delta } }
\newcommand{\longpath}{ { \gamma } }
\newcommand{\hlayers}{ { N_h } }

\newcommand{\mre}{\text{MSE}} 
\newcommand{\sdre}{\text{SSE}} 
\newcommand{\mde}{\text{MeanE}} 
\newcommand{\sdde}{\text{StdE}} 

\newcommand{\lin}{{LIN}}
\newcommand{\nonlin}{{NLIN}}
\newcommand{\nonadd}{{NADD}}

\newcommand{\uf}{{ \textit{uf}} }
\newcommand{\acc}{{ \textit{f1}} }
\newcommand{\clf}{ { {h}} }
\newcommand{\full}{ { {h}_{\text{full}} }}
\newcommand{\unaware}{ { {h}_{\text{unaw}}} }
\newcommand{\fairx}{ { {h}_{\text{fair-x}}} }
\newcommand{\ourclf}{ { {h}_{\text{fair-z}}} }

\newcommand{\yF}{{y^{F}}}
\newcommand{\yhF}{{\hat{y}^{F}}}
\newcommand{\yhCF}{{\hat{y}^{CF}}}

\newcommand{\xh}{{\hat{\bld{x}}}}
\newcommand{\Xh}{{\hat{X}}}

\newcommand{\collider}{\textit{collider}}
\newcommand{\triangl}{{\textit{triangle}}}
\newcommand{\chain}{{\textit{chain}}}
\newcommand{\loan}{{\textit{loan}}}
\newcommand{\adult}{{\textit{adult}}}
\newcommand{\mgraph}{{\textit{M-graph}}}

\newtheorem{lemma}{Lemma}
\newtheorem{proposition}{Proposition}

\newtheorem{property}{Property}

\newtheorem{definition}{Definition}[section]

\definecolor{myblue}{rgb}{100,143,255}
\definecolor{myyellow}{rgb}{255,176,0}
\definecolor{myred}{rgb}{220,38,127}

\makeatletter
\newcommand{\printfnsymbol}[1]{%
  \textsuperscript{\@fnsymbol{#1}}%
}
\makeatother

\title{\name: Design of Variational Graph Autoencoders for Interventional and Counterfactual Queries}

\author{
  Pablo Sanchez-Martin\thanks{Equal contribution}\\
  Max Planck Institute for Intelligent Systems\\
  T\"ubingen, Germany \\
  \texttt{pablo.sanchez-martin@tuebingen.mpg.de} \\
  \And
  Miriam Rateike\printfnsymbol{1}\\
  Max Planck Institute for Intelligent Systems\\
  T\"ubingen, Germany \\
  \texttt{mrateike@tuebingen.mpg.de} \\
      \And
Isabel Valera\\
  Saarland University \\
  Saarb\"ucken, Germany\\
  \texttt{ivalera@cs.uni-saarland.de} \\
}

\begin{document}
\maketitle

\begin{abstract}
In this paper, we introduce \name, a novel class of variational graph autoencoders for causal inference in the absence of hidden confounders, when only observational data and the causal graph are available.
Without making any {parametric} assumptions, \name\ mimics the necessary properties of a \textit{Structural Causal Model} (SCM)   to provide a flexible and practical framework  for {approximating} interventions (\emph{do-operator}) and  \textit{abduction-action-prediction} steps. 
 As a result, and as shown by our empirical results,   \name\ {accurately} approximates the interventional and counterfactual distributions on diverse SCMs.
Finally, we apply \name\ to evaluate counterfactual fairness in fair classification problems, as well as  to learn 
{fair classifiers without compromising performance.}
\end{abstract}


\section{Introduction}\label{sec:introduction}

Graph Neural Networks (GNNs) are a powerful tool for graph representation learning and have been proven to excel in practical complex problems like neural machine translation \cite{bastings2017graph}, traffic forecasting \cite{derrow2021eta, yu2017spatio}, or drug discovery \cite{gilmer2017neural}. 

In this work, we investigate to which extent the inductive bias of GNNs--encoding the causal graph information--can be exploited to answer interventional and counterfactual queries. More specifically, to approximate the interventional and counterfactual distributions induced by interventions on a casual model. 
To this end, we assume i) causal sufficiency--i.e., absence of hidden confounders;  and, ii) access to  observational data and the true causal graph. We stress that the causal graph can often be inferred from expert knowledge~\cite{zheng2019using} or via one of the approaches for causal discovery~\cite{glymour2019review, vowels2021d}. 
With this analysis we aim to  complement the concurrent line of research that 
theoretically studies the use of Neural Networks (NN) \cite{xia2021causalneural},  and more recently GNNs  \cite{zevcevic2021relating}, for causal inference.

To this end, we describe the architectural design conditions that a variational graph autoencoder (VGAE)--as a density estimator that leverages a priori graph structure--must fulfill so that it can approximate causal interventions (\emph{do-operator}) and  \textit{abduction-action-prediction} steps~\cite{pearl2009causalitybook}.
The resulting Variational Causal Graph Autoencoder, referred to as \name, enables \emph{approximating} the observational, interventional and counterfactual distributions induced by a causal model with unknown structural equations.
We remark that parametric assumptions  on the structural causal equations are in general not testable, may thus not hold in practice \cite{peters2017elements} and may lead to inaccurate results, if misspecified. 
\name\ addresses this limitation  by including uncertainty, i.e., a probabilistic model, in the estimation of the causal-parent relationships.

We show in extensive synthetic experiments that \name\ outperforms competing methods \cite{karimi2020algorithmic, khemakhem2021causal}  on complex datasets at estimating not only the mean of the interventional/counterfactual distribution (as in previous work), but also the overall distribution (measured in terms of Maximum Mean Discrepancy \cite{gretton2012kernel}). Finally, we show a practical use-case in which \name\ is used to assess counterfactual fairness of different classifiers trained on the real-world German Credit dataset \cite{GermanData}, as well as to learn counterfactually fair classifiers without compromising performance.

\subsection{Related Work}
\label{sec:related_work}

Deep generative models are enjoying increasing attention for causal queries in complex data~\cite{moraffah2020can, parafita2019explaining}. Existing approaches for causal inference focus on  i) estimating the Average Treatment Effect (ATE)--a specific type of group-level causal queries--by assuming a fixed causal graph that includes a treatment variable \cite{kim2020counterfactual, louizos2017causal, rakesh2018linked, schwab2018perfect, vowels2020targeted, zhang2020causal}; ii)  discovering and intervening on the causal latent structure of the (e.g., image) data~\cite{kim2020counterfactual, parafita2019explaining, martinez2019explaining, shen2020disentangled, yang2020causalvae}; or iii) addressing interventional and/or counterfactual queries by fitting a conditional model for each observed variable given its causal parents~\cite{garrido2020estimating, karimi2020algorithmic, kocaoglu2017causalgan, parafita2020causal, pawlowski2020deep}.

Within the scope of causality, GNNs have predominantly
 been used for causal discovery~\cite{yu2019dag,zhang2019d} and only very recently, concurrent with us, exploited to answer interventional queries \cite{zevcevic2021relating}.

\citet{khemakhem2021causal} propose  CAREFL, an autoregressive normalizing flow for both  causal discovery and inference. The authors focus on (multi-dimensional) bi-variate graphs, but their approach can be extended to more general  directed acyclic graphs (DAGs) using e.g., neuronal spline flows \cite{durkan2019neural}. 
However, causal assumptions in a graph are modeled not only by the direction of edges, but also the absence of edges \cite{pearl2009causal}. For the task of causal inference \carefl\ is unable to exploit the absence of edges fully as it reduces a causal graph to its causal ordering (which may not be unique). Further, the authors only evaluate interventions in root nodes (which reduces to conditioning on the intervened-upon variable).

\citet{karimi2020algorithmic} answer interventional queries by fitting a conditional variational autoencoder (CVAE) to each conditional in the Markov factorization implied by the causal graph. 
As each observed variable is independently fitted, the mismatch between the true and generated distribution can cause errors that propagate to the distribution of its descendants. This can be problematic, especially for long causal paths. 
 \citet{pawlowski2020deep} propose an approach similar to \citet{karimi2020algorithmic}, and additionally propose an approach based on normalizing flows to approximate the causal parent-child effect.

In contrast, \name\ leverages i) GNNs to encode the causal graph information (inductive bias), ii) the GNN message passing algorithm to approximate the  effect of interventions (\textit{do-operator}~\cite{pearl2009causalitybook}) in the causal graph, and iii) jointly optimizes the observational distribution for all observed variables to avoid error propagation along the Markov factorization. We thoroughly evaluate the performance of \name\, and compare it with related work, at approximating {both interventional and counterfactual} distributions induced
by interventions on both root and non-root nodes in a wide variety of causal models.

\section{Background} \label{sec:background}

In this section, we first provide a brief overview on SCMs
and then introduce the main building block of \name, i.e., variational graph autoencoders. 

\subsection{Structural causal models}
\label{sec:SCMs}

An SCM $\Mcal=(p(\U), \Fx)$ determines how a set of $d$  endogenous (observed) random variables $\X:=\left\{X_1, \dots X_d\right\}$ is generated from a set of exogenous (unobserved) random variables $\U:=\left\{U_1, \dots U_d\right\}$ with prior distribution $p(\U)$ via the set of  \emph{structural equations} $\Fx=\{ X_{i}:=\fx_{i}\left(\X_{\pai}, U_{i}\right)\}_{i=1}^{d}$.
Here $\X_{\pai}$ refers to the set of variables directly causing $X_i$, i.e., parents of $i$.
Similarly to \cite{karimi2020algorithmic, khemakhem2021causal, pearl2009causal}, we consider SCMs
that are associated with a  directed acyclic \textit{causal graph} (although Section \ref{sec:heterogenous} relaxes this assumption).
We here denote the causal graph by $\Gcal:=(\Vcal, \Ecal)$, where each node $i \in \Vcal$ corresponds to an endogenous variable $X_i$. The set of directed edges $(j, i) \in \Ecal$ represent the causal parent-child relationship between endogenous variables~\cite{pearl2009causal}, i.e. $X_j$ is a parent of $X_i$. $\Ecal$ can be represented by the adjacency matrix $\A\in\{0,1\}^{d\times d}$, such that  $A_{ij}=1$ if $(j, i) \in \Ecal$ and $A_{ij}=0$, otherwise. We also define the set of neighbors, a.k.a. parents, of node $i$ as $\pai = \Ncal_i = \{j\}_{(j, i) \in \Ecal}$ and $\pasi := \pai \cup {i}$. 
Given an SCM, there are two types of causal queries of general interest:  interventional queries, e.g., ``What would happen to the population $\X$, if variable $X_i$ {would be} set  to a fixed value $\alpha$?'',  and counterfactual queries, e.g.,``What would have happened to a specific factual sample $\xF$, had $X_i$ been set to a value $\alpha $?''. 
In more detail,  \emph{interventional queries} aim to evaluate  {the effect at the population level (\textit{rung 2}) of a specific intervention on,}
or equivalently  manipulations of, a subset of the endogenous variables $\mathcal{I} \subseteq[d]:=\{1, \dots, d\}$. 
Interventions on an SCM $\Mcal$ are often represented with 
the {\textit{do-operator}} $do(X_i = \alpha_i)$ \cite{pearl2009causalitybook} and lead to a modified SCM $\Mcal^{\Ical}$ which induces a new distribution over the set of endogenous variables $p(\X\mid do(X_i = \alpha_i))$, which is referred to as the \textit{interventional distribution}. In $\Gcal$ an intervention removes incoming edges to node $i$ and sets $X_i = \alpha$ (see Figure~\ref{fig:SCM_int}).
A \textit{counterfactual query} for a given factual instance $\xF$ aims to estimate what would have happened had $\XI$ instead taken value $\blda$.
This effect is captured by the \textit{counterfactual distribution} $p(\xCF\mid \xF, do(X_\Ical =\blda))$, which can be computed using the \emph{abduction-action-prediction} procedure by \citet{pearl2009causalitybook}. Refer to Section~\ref{sec:SCMproperties} for further details on the computation of the interventional and counterfactual distributions within our framework.

\begin{figure}[h!]
 \centering
      \centering
     \begin{subfigure}[b]{0.3\textwidth}
         \centering
        \scalebox{0.9}{





\begin{tikzpicture}[every node/.style={inner sep=0,outer sep=0}]
        \node[state, fill=gray!60] (x1) at (0,0) {$X_1$};
        \node[state, fill=gray!60] (x3) [below right = 0.80 cm and 0.4 cm of x1 ] {$X_3$};
        \node[state, fill=gray!60] (x2) [left  = 1.2cm of x3] {$X_2$};

        \node[state] (u1) [right = 0.3cm of x1] {$U_1$};
        \node[state] (u2) [above = 0.3cm of x2] {$U_2$};
        \node[state] (u3) [right =  0.3cm of x3] {$U_3$};

        \path (x1) edge [thick, color=black!10!blue] (x3);
        \path (x2) edge [thick, color=red](x3);
        \path (x1) edge [thick, color=red](x2);

        \path (u1) edge (x1);
        \path (u2) edge (x2);
        \path (u3) edge (x3);

\end{tikzpicture}}
         \caption{$\Gcal$ without intervention}
         \label{fig:SCM_obs}
     \end{subfigure}
    \hspace{0.5 cm}
      \centering
     \begin{subfigure}[b]{0.3\textwidth}
         \centering
        \scalebox{0.9}{





\begin{tikzpicture}[every node/.style={inner sep=0,outer sep=0}]
        \node[state, fill=gray!60] (x1) at (0,0) {$X_1$};
        \node[state, fill=gray!60] (x3) [below right = 0.80 cm and 0.4 cm of x1 ]  {$X_3$};
        \node (x2) [left  = 1.2cm of x3] {$\alpha$ \textbullet
        };

        \node[state] (u1) [right = 0.3cm of x1] {$U_1$};
        \node[state] (u3) [right = 0.3cm of x3] {$U_3$};

        \path (x1) edge [thick, color=black!10!blue] (x3);
        \path (x2) edge (x3);

        \path (u1) edge (x1);

        \path (u3) edge (x3);
\end{tikzpicture}}
         \caption{$\Gcal$ with intervention}
         \label{fig:SCM_int}
     \end{subfigure}
    \caption {
    {\triangl\ SCM $\Mcal=\{p(\U), \Fx\}$, $\U\sim p(\U)$ with $d=|\X|=3$ endogenous variables 
    where $ X_1:= \fx_1(U_1)$, $X_2:= \fx_2(X_1, U_2)$, $X_3:= \fx_3(X_1, X_2, U_3)$ 
    with (a) the corresponding causal graph $\Gcal$ and (b) the causal graph corresponding to $\Mcal^{\Ical}$ after intervention $do (X_2 = \alpha)$}. 
    {Blue (red) arrows highlight the direct (indirect) causal path from $X_1$ to $X_3$ (via $X_2$).} }
  \label{fig:SCM_graph}
\end{figure}
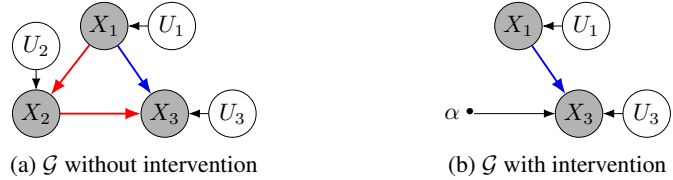

\subsection{Variational Graph Autoencoder and Graph Neural Networks}\label{sec:background_VGAE}

\textbf{Variational Autoencoders (VAEs).} VAEs \cite{kingma2013auto} are powerful latent variable models based on neural networks (NNs) for jointly i) learning expressive density estimators $p(\X) \approx \int p_{\theta}(\X \mid \Z) p(\Z) d \Z$, where the likelihood function (a.k.a.  \emph{decoder}) is parameterized using a NN with parameters $\theta$, and ii) performing approximate  posterior inference over the latent variables $\Z$ made possible via
a variational distribution (a.k.a. \emph{encoder}) $q_{\phi}\left(\Z \mid \X \right)$ parameterized using a NN with parameters ${\phi}$. 
The parameters $\theta$ and $\phi$ can be learned by maximizing a lower bound on the log-evidence~\cite{burda2015importance, nowozin2018debiasing, rainforth2018tighter, tucker2018doubly}.

 \textbf{Variational Graph Autoencoders (VGAEs).}~\citet{kipf2016variational} extend VAEs
to account for prior graph structure information on the data~\cite{yu2019dag}. VGAEs define a (potentially multidimensional) latent variable $Z_i$ per observed variable  $X_i$, i.e., $\Z:= \{Z_1, \dots, Z_d\}$.
Additionally, VGAEs rely on an adjacency matrix $\A$, which is used by two GNNs, one for the encoder and one for the decoder, to enforce structure  on  the posterior approximation  $q_{\phi}(\Z \mid \X, \A)$ and the likelihood $p_{\theta}(\X \mid \Z, \A)$. 
Hence, $\A$--given as prior--determines which variables $X_i$ influence $Z_j$ $\forall i, j \in [d]$. 

\textbf{Graph Neural Networks (GNNs).} In its most general form, a GNN is a composition of  message passing layers \cite{gilmer2017neural}, where each layer
updates the state of each node in $\mathcal{G}$. 
In particular, the state of node $i$ at the output of layer $l$,
i.e., $\bld{h}_i^l$, is specified as:
\begin{align} \label{eq:GNN}
    \bld{h}_i^l &  = \fu \left( \bld{h}_i^{l-1}, \fa  \left( \{ \bld{m}_{ij}^l\}_{j \in \Ncali} \right) ; \theta^l_u \right).
\end{align}
{\small %
}
First, node $i$ receives a message {\small $\bld{m}_{ij}^l = \fm(\bld{h}_i^{l-1},\bld{h}_j^{l-1}; \theta_m^l)$} from each of its neighbors  $ j \in \Ncal_i$. Then,  these messages are aggregated via $\fa$.
Finally, $\bld{h}_i^l$ is computed as a function $\fu$ of the node's previous state $\bld{h}_i^{l-1}$ and the aggregated message.
Note, if a GNN has $\hlayers$ hidden layers, then the output for node $i$ depends not only on its direct neighbors $\Ncali$, but also on its neighbors up to order $N_{h}+1$ (hops). 
For example, if $N_{h}=0$ ($N_{h}=1$) then the output for each node only depend on its direct neighbors, i.e., \emph{parents} (2-hop neighbors, i.e., \emph{grand-parents}).
For a detailed description of GNNs, please refer to Appendix~\ref{apx:GNN}.

\section{Observational, interventional and counterfactual distributions} \label{sec:SCMproperties}
\label{sec:distributions}
{In this section, we introduce the observational, interventional and counterfactual distributions (triggered by any intervention  of the form $do(\X_\Ical =\blda)$)} that are induced by an SCM $\mathcal{M}=\{p(\U), \Fx \}$.  Specifically, we summarize the main properties of an SCM that will allow us to propose a novel class of VGAEs,
namely \name, to compute accurate estimates of these distributions using observational data and a known causal graph. To this end, we assume the absence of hidden confounders, i.e., we assume that $p(\U) = \prod_{i=1}^d p(U_i)$. 

\textbf{Observational distribution.}
The SCM $\mathcal{M}$ 
determines the observational distribution $p(\X)$ over the set of endogenous variables $\X=\left\{X_1, \dots X_d\right\}$, which~satisfies causal factorization~\cite{scholkopf2019causality}, i.e.,  $p(\X) = \prod_{i=1}^d p(X_i \mid \X_{\mathrm{pa}(i)}).$
That is, after marginalizing out the exogenous variables $\U$, the distribution of each endogenous variable $X_i$ depends only on its parents, i.e., $\X_{\pai}$. 

The \textit{observational distribution} can alternatively  be written only in terms of the exogenous variables $\U$ as
\begin{align}\label{eq:SCMobs}
    p(\X) = \F_{\#}[p(\U)], 
\end{align}

i.e., $p(\mathbf{X})$ is the pushforward of $P(\mathbf{U})$ through $\mathbf{F}$. Here $\F: \U \rightarrow \X$ corresponds to the set of structural equations, which directly transform the exogenous variables $\U$ into the endogenous variables $\X$. This is equivalent to $\Fx$, which takes as input both the exogenous variable and the parent (endogenous) variables of a target endogenous variables to compute its value.

Let us denote by $\ani$ the set of indexes of the ancestors of $i$, and $\ansi:= \ani \cup \{ i\}$. 
Then, the causal factorization induced by  $\Mcal$  leads to the following property of $\F(\U)$: 
\begin{property}\label{prop:prediciton}
 Each endogenous variable $X_i$ can be expressed as a function of its exogenous variable $U_i$ and the ones of all its causal ancestors, i.e.,  $ \F(\U)${\small$=$}$\{X_i = f_i(\{U_j\}_{j \in \ansi}) \}$. This, together with the causal sufficiency assumption,  implies that  $X_i$ is statistically independent of $U_j, \forall j \notin \ansi$.
\end{property}

\textbf{Interventional distribution.}
As stated in Section~\ref{sec:SCMs}, interventions on a set of variables $\Ical$ can be performed using the \textit{do-operator}, which can be seen as a mapping $do(\XI = \blda): \Mcal \mapsto \Mcal^{\Ical} = (p(\U), \FxI)$ where $\FxI =  \{ \fx_{i}\}_{i \not \in \Ical} \cup \{ \alpha_i \}_{i \in \Ical}$. As above, we can represent the resulting set of \textit{intervened  structural equations} as $\FI= \{ f_i \}_{i \not \in \Ical} \cup \{ \alpha_i \}_{ i \in \Ical}$, and thus write the \textit{interventional distribution} as:
\begin{align}\label{eq:SCMint}
    p(\X \mid do(\XI = \blda)) =   \FI_{\#}[p(\U)].
\end{align}

\begin{property} \label{prop:action}

After an intervention $do(\XI = \blda)$ on $\Mcal$, all the causal paths from  $U_j \qforall j  \in \ansi$  to $X_i$ that include an intervened-upon variable in $\XI$ (i.e., the causal paths where $\XI$ is a mediator) are severed  in $\FI$, while the rest of causal paths remain untouched.    

\end{property}

The above property is illustrated in Figure~\ref{fig:SCM_graph}, where we can observe that after an intervention $do (X_2 = \alpha)$, the indirect causal path (in red) from $X_1$, and thus from $U_1$, to $X_3$ via $X_2$ is severed, while the direct path (in blue) remains. 

\textbf{Counterfactual distribution.}
 Assuming the SCM $\mathcal{M}=\{p(\U), \Fx \}$ to be known, 
 the following three steps defined by \citet{pearl2009causal} allow to compute counterfactuals $\xCF$:
 i) \textit{Abduction:} infer the values of the exogenous variables $\U$ for a factual sample $\bld{x}^{F}$, i.e.,  compute $p(\U\mid\xF)$; ii) \textit{Action:} intervene with $do(\XI = \blda): \Mcal \mapsto \Mcal^{\Ical}= (p(\U), \FxI)$; and iii) \textit{Prediction:} use the posterior distribution $p(\U\mid\xF)$ and the new structural equations $\FxI$ to compute $p(\xCF\mid\xF)$. The prediction step can alternatively be computed using the new set of structural equations $\FI$ defined in terms of the exogenous variables $\U$, so that we can write the  \textit{counterfactual distribution} as:
   \begin{align}\label{eq:SCMcf}
     {p(\xCF \mid \xF, do(\XI = \blda))} =   \FI_{\#}[p(\U\mid\xF)]. 
 \end{align}
Importantly, the posterior distribution {\small$p(\U\mid\xF)$} satisfies:

 \begin{property} \label{prop:abduction}
In the abduction step,  statistical independence implies that  conditioned on the endogenous variables of the factual sample $\xF$, each exogenous variable $U_i$ is independent of the factual value $x^F_j$ if $j\neq i$ and the variable $X_j$ is not a parent of $X_i$, i.e., $j \nin\pasi$.
 \end{property}

\vspace{-5pt}
\section{Variational Causal Autoencoder \nameb}\label{sec:CaVAE}

In this section, we present a novel variational causal graph autoencoder (\name) to approximate the observational \eqref{eq:SCMobs}, interventional \eqref{eq:SCMint} and counterfactual \eqref{eq:SCMcf} distributions. While the underlying SCM $\Mcal$ is unknown, we assume access to the true causal graph $\Gcal$  and observational data $\{\bld{x}_n\}^N_{n=1}$, i.e., i.i.d. samples of the observational distribution induced by $\Mcal$  (in the absence of hidden confounders).

\begin{definition}{\textit{\nameb.}}\label{def:VCAUSE}
Given a causal graph $\Gcal$ over a set of endogenous variables $\X = \{X_1, \ldots, X_d\}$, which  establishes the set of parents $\pai$ for each variable $X_i$ (including the $i$-th node), 
\name\ is defined by:
\begin{itemize}
\item A causal adjacency matrix $\A$, which is a $d\times d$ binary matrix with elements $A_{ij}=1$ if $j \in \pasi$, i.e., when $i=j$ or $j$ is a parent of $i$. Otherwise, $\Aij=0$.
\item A prior distribution $p(\Z) = \prod_i p(Z_i)$ over the set of latent variables $\Z = \{Z_1, \ldots, Z_d\}$. 
\item A decoder $\pt(\X \mid \Z, \A)$, which  is a GNN (parameterized by $\theta$) that takes as input the set of latent variables $\Z$ and the causal  adjacency matrix $\A$, and outputs the parameters of the likelihood $\pt(\X \mid \Z, \A)$. 
\item An encoder $q_{\phi}(\Z\mid\X, \A)$, which is a GNN (parameterized by $\phi)$ that takes as input the endogenous variables $\X$ and the causal adjacency matrix $\A$, and outputs the parameters of the posterior approximation $q_{\phi}(\Z\mid\X, \A)$. 
\end{itemize}
\end{definition}

Next, we discuss how to design \name\ such that it is able to capture the observational, interventional, and counterfactual distribution induced by an unknown SCM. 
Importantly, we derive the necessary conditions on the design of both the encoder and decoder GNNs such that \name\ can mimic the SCM properties introduced in Section \ref{sec:distributions}. 

\subsection{Observational distribution} 

\name\ approximates the \textit{observational distribution} in \eqref{eq:SCMobs} using the generative model as
\begin{align}\label{eq:CaVAEobs}
\begin{split}
    & p(\X) \approx \int \pt(\X \mid \Z, \A) \pZ d \Z,
\end{split}
\end{align}
\noindent where $\pt(\X \mid \Z, \A)=\prod_{i=1}^d \pt(X_i \mid \Z, \A)$. Figure~\ref{fig:obs} depicts this generative process.

\textbf{{Relationship between $\bld{Z}$ and $\bld{U}$}.} 
When comparing \eqref{eq:CaVAEobs}  with the true observational distribution in~\eqref{eq:SCMobs}, we observe that the latent variables $\Z$ play a similar role to the exogenous variables $\U$, and the decoder $\pt(\X \mid \Z, \A)$ plays a similar role to the structural equations $\F$. We remark that 
 $\Z$ do not need to correspond to the true exogenous variables (i.e., $p(\U) \neq p(\Z)$), and thus, the decoder does not aim to approximate the causal structural equations.   
 Yet, we assume  that there is one independent latent variable $Z_{i}$ for every observed variable $X_{i}$ capturing all the information of $X_{i}$ that cannot be explained by its parents. Thus, since $X_{i}$ is in turn a (deterministic) function of its parents $\pai$  and its exogenous variable $U_{i}$, the posterior $p\left(Z_{i} \mid X_{i}, \pai \right)$ aims to capture the information that $U_{i}$ contributes to $X_{i}$ (i.e., the information of  $X_{i}$ not contributed by
 its parents). 
{That is--similar to $p\left(U_{i} \mid X_{i}, \pai \right)$, the (true) posterior distribution--$p\left(Z_{i} \mid X_{i}, \pai \right)$  should depend only on $X_{i}$ and parents
$\pai$.}

 \textbf{{Observational noise}.} { \name\ has observational noise that is not present in the true SCM, where an observed variable is assumed to be a deterministic transformation of its exogenous variables and parents via the structural equations (SEs)}.
 As \name\ does not have access to the true SEs (nor to the true distribution of the exogenous variables), the \textit{noise} of the likelihood $p_{\theta}(\X \mid \Z, \A)$ can be interpreted as an estimate of the uncertainty on the estimated observational distribution (due to the uncertainty on the true SCM).

Here, we seek to ensure that the $p(\X)$ induced by \name\ complies with causal factorization (\textbf{Property~\ref{prop:prediciton}} in Section~\ref{sec:SCMproperties}).  To that end, the design of the decoder GNN must assure that $\pt(X_i \mid \Z, \A) = \pt(X_i \mid \Z_{\ansi})$. That is,  that $X_i$ depends only on $Z_j$ if $j=i$ or $X_j$ is an ancestor of $X_i$ in the causal graph.

\begin{proposition}{(Causal factorization).}\label{prop:CaVAEobs} 
\name\ satisfies causal factorization,  $\pt(\X \mid \Z, \A) = \prod_{i} p_{\theta_i}(X_i \mid \Z_{\ansi})$, if and only if the number of hidden layers in the decoder is greater or equal than $\diameter-1$, with $\diameter$ being the length of the longest shortest path between any two endogenous nodes.
\end{proposition}

The above proposition (proved in Appendix \ref{apx:VCAUSE_proof}) is based on the fact that, in a GNN with $N_h$ hidden layers (and $N_h+1$ layers in total), the output for the $i$-th node depends on its neighbors of up to $N_h+1$ hops. As an example, consider the following
\chain\ causal graph: $X_1 \rightarrow X_2 \rightarrow X_3$, such that $\delta =2$. 
 \ref{def:VCAUSE} 
In the decoder, the first layer yields a hidden representation for the $3$-rd node $ h_3^1:=f(f(Z_2), Z_3)$  
that only depends on $Z_2$ and $Z_3$.  Thus, we need a second  layer for its output $h_3^2:=f(h_2, Z_3) = f(f(f(Z_1), Z_2), Z_3)$  to depend on $Z_1$ (note that  $X_1$ is an ancestor of $X_3$).

\begin{figure}
     \centering
     \begin{subfigure}[b]{0.2\textwidth}
        \scalebox{0.7}{
         \begin{tikzpicture} [every node/.style={inner sep=0,outer sep=0}]
        \node[state, fill=gray!60] (x1) at (0,0) {$X_1$};
        \node[state, fill=gray!60] (x2) [below = 0.40 cm of x1] {$X_2$};
        \node[state, fill=gray!60] (x3) [below  = 0.40 cm of x2]  {$X_3$};

        \node[state] (z1) [left = 3cm of x1]  {$Z_1$};
        \node[state] (z2) [left = 3cm of x2]  {$Z_2$};
        \node[state] (z3) [left = 3cm of x3] {$Z_3$};
        
        \node (h1) [right = 1cm of z1]   {\textbullet $h_1$};
        \node (h2) [right = 1cm of z2]    {\textbullet  $h_2$};
        \node (h3) [right = 1cm of z3]    {\textbullet  $h_3$};

        \path (z1) edge [thick, color=black!10!blue] (h3);
        \path (z1) edge [thick, color=black!10!blue] (h1);
        \path (h3) edge [thick, color=black!10!blue] (x3);

        \path (h2) edge [thick, color=red](x3);
        \path (z1) edge [thick, color=red](h2);
        \path (h2) edge [thick, color=red](x3);
        
        \path (z2) edge [thick, color=gray](h2);
        \path (z3) edge [thick, color=gray](h3);
        \path (z2) edge [thick, color=gray](h3);
        
        \path (h1) edge [thick, color=gray](x1);
        \path (h1) edge [thick, color=gray](x2);
        \path (h2) edge [thick, color=gray](x2);
        
        \path (h1) edge [thick, color=gray](x3);
        \path (h1) edge [thick, color=black!10!blue] (x3);
\end{tikzpicture}}
         \caption{Original}
         \label{fig:GNN_obs}
     \end{subfigure}
    \hspace{0.5 cm}        
     \begin{subfigure}[b]{0.2\textwidth}
        \scalebox{0.7}{

        


        
        
        

\begin{tikzpicture} [every node/.style={inner sep=0,outer sep=0}]
        \node[state, fill=gray!60] (x1) at (0,0)  {$X_1$};
        \node[state, fill=gray!60] (x2) [below = 0.40 cm of x1]  {$X_2$};
        \node[state, fill=gray!60] (x3) [below  = 0.40 cm of x2]  {$X_3$};

        \node[state] (z1) [left = 3cm of x1]  {$Z_1$};
        \node[state] (z2) [left = 3cm of x2]  {$Z_2$};
        \node[state] (z3) [left = 3cm of x3]  {$Z_3$};
        
        \node (h1) [right = 1cm of z1]  {\textbullet $h_1$};
        \node (h2) [right = 1cm of z2]  {\textbullet $h_2$};
        \node (h3) [right = 1cm of z3]   {\textbullet $h_3$};

        \path (z1) edge [thick, color=black!10!blue] (h3);
        \path (h3) edge [thick, color=black!10!blue] (x3);
        \path (z1) edge [thick, color=black!10!blue] (h1);

        \path (h2) edge [thick, color=gray](x3);
        
        \path (z2) edge [thick, color=gray](h2);
        \path (z3) edge [thick, color=gray](h3);
        \path (z2) edge [thick, color=gray](h3);
        
        \path (h1) edge [thick, color=gray](x1);
        \path (h2) edge [thick, color=gray](x2);
        
        \path (h1) edge [thick, color=gray](x3);
        \path (h1) edge [thick, color=black!10!blue] (x3);
\end{tikzpicture}}
         \caption{Intervened}
         \label{fig:GNN_int}
     \end{subfigure}
     \caption{\name\ decoder (a) with and (b) without intervening on $X_2$. 
     Message passing in the GNN correspond to direct (blue) and indirect (red) causal paths in Figure~\ref{fig:SCM_graph}.
     }\label{fig:GNN}
\end{figure}
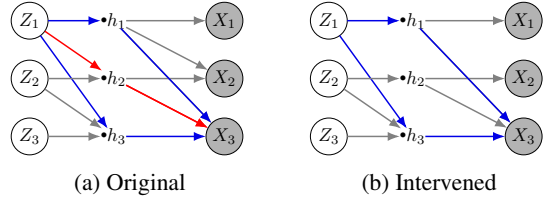

\begin{figure*}
     \centering
     \begin{subfigure}[b]{0.18\textwidth} 
         \centering
         \includegraphics[width=0.95\linewidth]{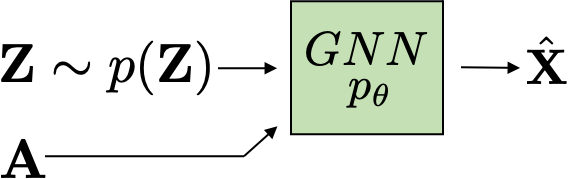}
         \caption{Observational}
         \label{fig:obs}
     \end{subfigure}
     \hspace{1mm}
     \begin{subfigure}[b]{0.38\textwidth} 
         \centering
         \includegraphics[width=1\linewidth]{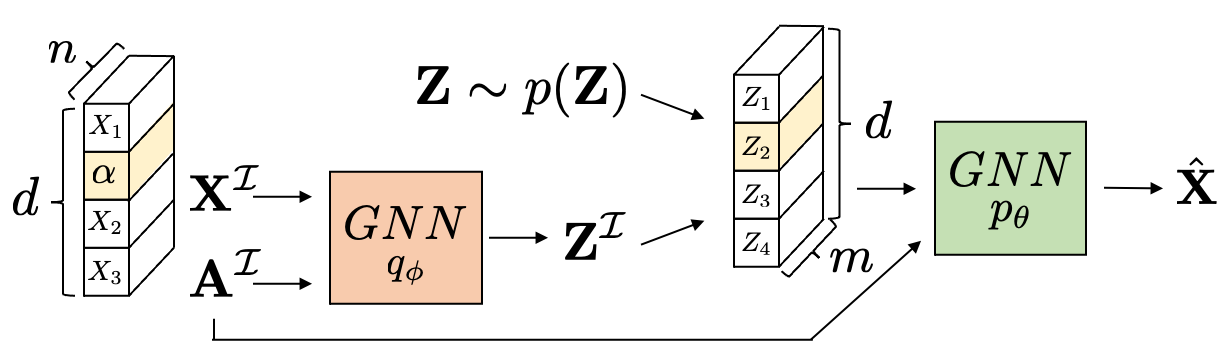}
         \caption{Interventional}
         \label{fig:int}
     \end{subfigure}
     \hfill
     \begin{subfigure}[b]{0.40\textwidth}
         \centering
         \includegraphics[width=1\linewidth]{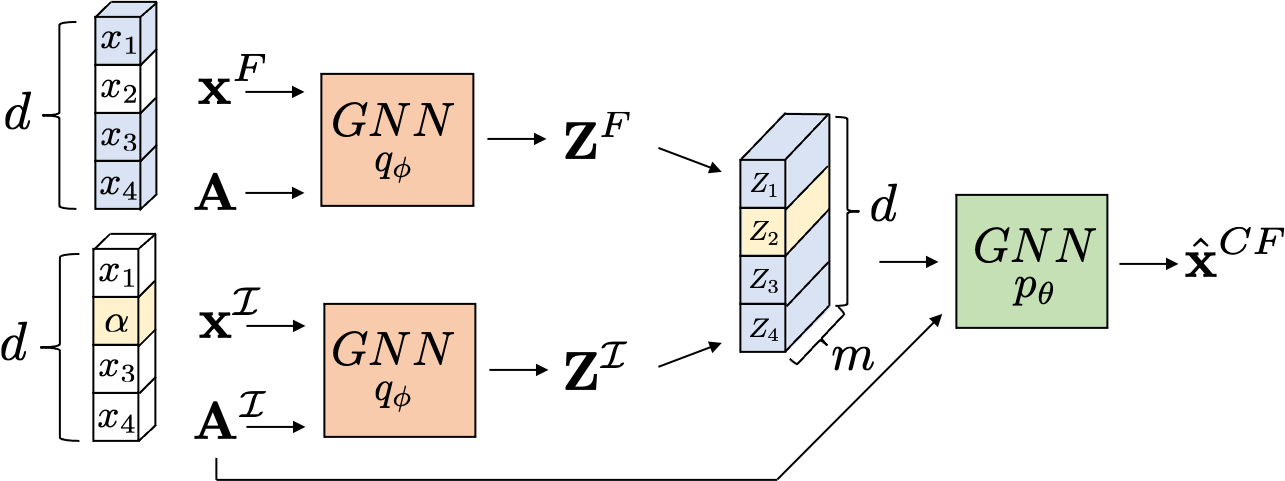}
         \caption{Counterfactual}
         \label{fig:coun}
     \end{subfigure}
         \caption{\name\ generation of (a) observational, (b) interventional, and (c) counterfactual samples. The `hat' in $\hat{\X}$ and $\hat{\x}^{CF}$ indicates that they are sample estimates of the true random variables. }
        \label{fig:three graphs}
\end{figure*}

\newpage
\subsection{Interventional distribution}

\name\  approximates the \textit{interventional distribution} in \eqref{eq:SCMint} as (illustrated in Figure~\ref{fig:GNN}):
\begin{align}\label{eq:CaVAEint} 
    p(\X \mid do(\XI &= \blda)) \approx \int \int\pt(\X \mid \Zt, \ZIt, \AI)  p(\Zt) \qp(\ZIt \mid \AI, \XI) d\Zt d\ZIt,
\end{align}
\noindent where $\ZIt =\{{Z}^{\Ical}_i\}_{i \in \Ical}$ is the subset of latent variables associated with the intervened-upon variables $\XI$, and  $\Zt = \{Z_i\}_{i \nin \Ical} $ {denotes  the subset of latent variables associated with the rest of the observed variables}. %
Importantly, here the \textit{do-operator} is performed on the causal adjacency matrix {as $do(X_\Ical = \blda): \A \mapsto \AI = \{ \Aij \}_{\qforall i \not \in \Ical, j} \cup \{ \Aij =0  \}_{\qforall i \in \Ical, j}$.} This ensures that  $X_i$ for $i \in \Ical$ is independent of $Z_j$ for all $j\neq i$. Note that in order for \eqref{eq:CaVAEint} to be able to approximate the interventional distribution in \eqref{eq:SCMint},  an intervention on 
\name\
should satisfy \textbf{Property~\ref{prop:action}}, i.e.:

\begin{proposition}{(Causal interventions).}\label{prop:CaVAEint}
\name\ captures causal interventions if and only if the number of hidden layers in its decoder is greater than or equal to $\gamma-1$, with $\gamma$ being the length of the longest path between any two endogenous nodes in $\Gcal$. 
\end{proposition}

To illustrate this, Figure~\ref{fig:GNN} depicts how messages are exchanged in a one-hidden-layer decoder GNN corresponding to the causal graph $\Gcal$ in Figure~\ref{fig:SCM_graph} (\emph{triangle} with $\longpath=2$), both (a) without and (b) with  an intervention on $X_2$. We highlight in blue the direct messages (sent via direct causal path in $\Gcal$), and in red the indirect  messages (sent via indirect causal path in $\Gcal$) from $Z_1$ to $X_3$. Observe that, similarly to Figure~\ref{fig:SCM_graph}, in (a) there is an indirect path (via $h_2$) from $Z_1$ to $X_3$; while in (b) this path is severed. Hence, the hidden layer $(h_1, h_2, h_3)$ allows to distinguish between direct and indirect paths and thus to capture interventional effects. 
As the condition in Proposition~\ref{prop:CaVAEint} is more restrictive than the one in Proposition~\ref{prop:CaVAEobs}, 
\name\ is able to approximate the observational and interventional distributions (as empirically validated in Appendix \ref{apx:vaca_analysis}) if: 

\textbf{Design condition 1 (necessary condition)} \emph{The decoder GNN of \name\ has at least as many hidden layers as $\longpath-1$, with $\longpath$ being the longest directed path in the causal graph $\Gcal$.}

\subsection{Counterfactual distribution} 

\name\ approximates the \textit{counterfactual distribution} in~\eqref{eq:SCMcf} as (illustrated in Figure~\ref{fig:coun}):
\small{
\begin{align} 
\begin{split}
    p(\xCF & \mid do(\XI  = \blda), \xF) \approx  \underbrace{\int \int  \underbrace{\pt(\X \mid \ZFt,\ZIt, \AI) \qp(\ZIt \mid  \xI, \AI)}_{action}  \underbrace{\qp(\ZFt \mid \xF, \A)}_{abduction} d\ZIt d\ZFt}_{prediction},
\end{split}
\end{align}
}

\noindent where $\xF$ represents a sample from $\X$ for which  we seek to compute the distribution over counterfactual $\xCF$ and  $\ZFt = \{{Z}^{F}_i\}_{i \nin \Ical}$. 
Note that two different passes of the encoder are necessary: 
one for the \emph{abduction} step of the factual instance $\qp(\ZFt \mid  \xF, \A)$;  and another one for the \emph{action} step (intervention) $\qp(\ZIt \mid \xI, \AI)$ with 
$x^\Ical_i = \alpha_i \qforall i \in \Ical$ (we remark that the rest of the values in $\xI$ do not affect the overall counterfactual computation). 

We then evaluate the likelihood making sure that the resulting counterfactual sample $\xCF$ only depends on the 
 $ \ZFt$ and $\ZIt$. Importantly, in order for \name\ to be able to approximate the counterfactual distribution,  we need its abduction (and action) step(s) to comply with \textbf{Property~\ref{prop:abduction}}, i.e.:
\begin{proposition}{(Abduction).}\label{prop:CaVAEcf} %
The abduction step of an observed sample $\x= \{x_1, \ldots, x_d\}$ in \name\ satisfies that for all $i$ the posterior of $Z_i$  is independent on the subset  $\{x_j\}_{j \nin \pasi} \subseteq \x$, if and only if the encoder GNN has no hidden layers. 
\end{proposition}

The above result (proved in Appendix \ref{apx:VCAUSE_proof}) can be shown by the message passing algorithm computed by the encoder GNN, and leads to 
the second design condition of \name:

\textbf{Design condition 2  (necessary condition):}
\emph{The encoder GNN of \name\ has no hidden layers.}

In other words, from the definition of the SCM, the posterior distribution of $U_{i}$ only depends on the parents, i.e $U_{i} \mid X_{\pai}$. In order for \name\ to mimic this property, the GNN that parameterized the encoder contains no hidden layers:
in the message passing algorithm, in the $k$-th iteration (layer), a node depends on its $k$-hop ancestors; requiring $k=1$ in the encoder refers to a GNN without hidden layers. 
Note that while the above condition may look restrictive and limiting the capacity of our encoder, we may choose arbitrarily complex NNs 
{for}
the message $\fm$ and update $\fu$ functions, as well as one or more aggregation functions $f^a$, e.g., sum or max, to model the encoder~\cite{corso2020principal}.

\subsection{Practical considerations}  \label{sec:heterogenous}

Next, we briefly discuss practical implementation considerations to  handle complex causal models, which often appear in real world applications~\cite{GermanData, AdultData} 
For further details on \name\ implementation, refer to Appendix~\ref{apx:VCAUSE}. 

\textbf{Heterogeneous causal nodes.} So far, we have modeled each endogenous variable $X_i$ as a node in the causal graph $\Gcal$,  and thus in the \name\ GNNs. Yet, in some application domains the relationships between a subset of $k_i$ variables may be unknown, or they may be affected by hidden confounders.
 In such cases, we assume  {that set of $k_i$ variables} to be correlated and model them as one multidimensional and potentially heterogeneous node $\X_i = \{X_{i1},  \dots, X_{ik_i}\}$ that share the same latent random variable $Z_i$. This allows us to deal with a large variety of graphs in practice.

\textbf{Heterogeneous endogenous variables.} Heterogeneous causal nodes require us to model different functions for each node, i.e. nodes may now contain a mix of continuous/discrete variables.
In general GNNs are parametrized such that the parameters of the message function $f^m$ and update function $f^u$ are shared for all the nodes and edges in the graph. However, similar to the structural equations $\F$, we can define a unique set of parameters  $\theta_{mij}$ for each $\fm_{ij}$  (see \eqref{eq:GNN}), so that we can model  a different function for every edge in the causal graph. %
Further, we can also assume different update functions $f^u_i$ for each node $i$, by introducing different update parameters $\theta_{ui}$. Note, however, that \name\ can be used with any GNN architecture \cite{corso2020principal, hamilton2017inductive}.

\section{Evaluation} \label{sec:Experiments}

In this section, we evaluate the potential of \name\ in approximating the outcomes of causal queries and compare it to two competing methods in synthetic experiments.
The synthetic setting allows us to have access the true SEs, which is necessary to evaluate interventional distributions and especially counterfactuals. We consider interventions of the form $do(x_i = \alpha_i)$ for several values of $\alpha_i$ on both root and non-root nodes. We compute all results over the same 10 random seeds and report mean and standard deviation. Refer to Appendix \ref{apx:training}  for a complete description of the experimental setup. Moreover, our code is publicly available at GitHub\footnote{\url{https://github.com/psanch21/VACA}}.

\begin{table*}[]
\setlength\tabcolsep{3pt
    \begin{tabular}{cc  l r rrr  rr r}
\toprule
 &  &    & \multicolumn{1}{c}{Obs.}   &  \multicolumn{3}{c}{Interventional}  &  \multicolumn{2}{c}{Counterfactuals}  & \\
  \cmidrule(r){4-4}   \cmidrule(r){5-7}  \cmidrule(r){8-9}
      \multicolumn{2}{c}{SCM} & Model   &MMD  &   MMD  &  \mde\  &    \sdde\    & \mre\  &    \sdre\   & Num. params \\
\cmidrule(r){1-10} 
 \multirow{6}{*}{\rotatebox[origin=c]{90}{ {\collider}}}  &  \multirow{3}{*}{\rotatebox[origin=c]{90}{\lin}} & MultiCVAE &   30.37$\pm$8.16 &   44.70$\pm$12.25 &   13.29$\pm$4.78 &   46.56$\pm$2.40 &   87.41$\pm$3.64 &  65.15$\pm$2.83 &    553  \\
       &              & CAREFL &    9.27$\pm$1.49 &     4.86$\pm$0.45 &    0.35$\pm$0.08 &   81.89$\pm$1.78 &    8.11$\pm$0.58 &   7.83$\pm$0.55 &   6420  \\
       &              & VACA &    1.50$\pm$0.67 &     1.57$\pm$0.41 &    0.75$\pm$0.31 &   41.99$\pm$0.30 &    9.86$\pm$0.74 &   7.06$\pm$0.38 &   5600  \\
      \cline{2-10}
         &  \multirow{3}{*}{\rotatebox[origin=c]{90}{\nonlin}} & MultiCVAE &    28.03$\pm$9.12 &   41.60$\pm$12.62 &   10.49$\pm$4.12 &   46.48$\pm$2.43 &   82.32$\pm$2.61 &  62.05$\pm$1.87 &    553  \\
       &              & CAREFL &   10.38$\pm$2.00 &     4.69$\pm$0.38 &    0.19$\pm$0.07 &   80.68$\pm$2.08 &    6.93$\pm$0.40 &   7.15$\pm$0.64 &   4308  \\
       &              & VACA &    0.95$\pm$0.27 &     0.97$\pm$0.23 &    0.26$\pm$0.12 &   42.20$\pm$0.24 &    5.01$\pm$0.73 &   4.08$\pm$0.54 &   1805  \\
\hline\hline
 \multirow{3}{*}{\rotatebox[origin=c]{90}{\loan}} &   & MultiCVAE &  90.38$\pm$11.31 &   213.65$\pm$5.38 &   12.24$\pm$1.33 &   65.78$\pm$1.13 &   40.98$\pm$0.35 &  15.12$\pm$0.16 &  33717 \\
       &         -     & CAREFL &   22.10$\pm$1.64 &    27.38$\pm$4.07 &    6.74$\pm$4.25 &   50.13$\pm$2.47 &   11.15$\pm$2.57 &   6.59$\pm$0.38 &  2880 \\
       &              & VACA &    2.22$\pm$0.25 &     6.87$\pm$0.66 &    4.35$\pm$0.35 &    3.83$\pm$0.08 &   10.30$\pm$0.40 &   6.41$\pm$0.11 &  30402 \\
\hline\hline
 \multirow{3}{*}{\rotatebox[origin=c]{90}{\adult}} &   & MultiCVAE &   140.15$\pm$6.37 &  155.52$\pm$5.93 &  12.18$\pm$2.36 &  63.52$\pm$4.05 &  39.96$\pm$0.36 &  16.37$\pm$0.65 &    6549 \\
      &     -   & CAREFL &   31.31$\pm$1.58 &   34.31$\pm$5.77 &  12.54$\pm$3.17 &  41.26$\pm$3.44 &   1.23$\pm$0.17 &   3.55$\pm$0.90 &  127420 \\
      &        & VACA &    4.51$\pm$0.45 &   12.68$\pm$1.95 &   1.65$\pm$0.23 &   3.37$\pm$0.09 &   5.33$\pm$0.27 &   5.67$\pm$0.20 &   63432 \\
\bottomrule
\end{tabular}
    \caption{Performance of different methods at estimating the observational, interventional and counterfactual distribution of different complex SCMs. Values are multiplied by 100. {All models have been cross-validated with a similar computational budget. The number of parameters of the best configuration is shown in the right column.} }\label{tab:results}
}
\end{table*}

\textbf{Datasets.} We consider 6 different synthetic causal graphs that differ in the number of nodes $d$, diameter $\diameter$, and  longest path $\longpath$. Here, we report the results for i) the  \collider\ ($d=3$, $\diameter=1$, $\longpath=1$) with  linear  (\lin) and non-linear (\nonlin) additive noise SEs, 
ii) the \loan\ from~\citet{karimi2020algorithmic}  ($d=7$, $\diameter=2$, $\longpath=3$), and iii) the \adult\ ($d=11$, $\diameter=2$, $\longpath=3$) graphs. Note that the two latter ones are synthetic  versions of the  German Credit dataset  \cite{GermanData} and the Adult datasets \cite{AdultData}, respectively.
See Appendix \ref{apx:training_datasets} for further details on the graphs and  Appendix \ref{apx:additional_results}  for the results with the remaining graphs. 

\textbf{Metrics.} We evaluate the observational distribution using the Maximum Mean Discrepancy (MMD) \cite{gretton2012kernel} as distance-measure between the true and estimated distributions, i.e., the lower the MMD the better the distributions match. For the interventional distribution, we additionally report the average estimation squared error of the mean (\mde) and of the standard deviation (\sdde) over all descendants of the intervened-upon variables. 
For the counterfactual distribution we report the mean squared error (\mre) as well as the standard deviation of the squared error (\sdre) between the true and the estimated counterfactual value. More details in Appendix \ref{apx:training_metrics}.

\textbf{Baselines.} We compare \name\ with  \mcvae\ \cite{karimi2020algorithmic}  and \carefl\ \cite{khemakhem2021causal} described in Section \ref{sec:related_work}. {For a fair comparison, all} model hyperparameters have been cross-validated using a similar computational budget (see Appendix \ref{apx:training_implementation}). In Table \ref{tab:results}, we report for each model and SCM the best configuration according to observational MMD. We also include a time-complexity analysis in Appendix \ref{apx:time_analysis}.

\textbf{Results.} Table \ref{tab:results} summarizes the results. We observe that, in general, \mcvae\ underperforms the other methods.
This may be explained by the fact that \mcvae\ trains each node independently, and thus the discrepancy between the true and generated distributions in one node may be amplified in  its descendants.
Comparing \name\ to \carefl, we first
observe that \name\  performs consistently better in terms of observational MMD, i.e., \name\ is able to generate observational samples that better resemble the true ones. 
Second, regarding the interventional distribution,
\carefl\ does a good job at fitting the mean (i.e., low \mde). However,\name\ performs consistently better both at approximating  the standard deviation (i.e., low \sdde) and the true samples (i.e., low MMD). This can be explained by \name\ i) leveraging the causal graph (contrary to  \carefl\ that relies on causal ordering), and ii) optimizing the log-evidence in  (\ref{eq:CaVAEobs}) jointly (contrary to the sequential optimization of \mcvae). 
Thus, \name\ approximates the distribution as a whole better, which is a desirable property for studying interventions on a population-level rather than just on average.
{Lastly, in the approximation of counterfactuals we observe that both \carefl\ and \name\ exhibit similar performance in terms of \mre\ and \sdre.
Note however that \carefl\ performs exact inference while \name\ is built on approximate inference and is trained on a lower bound on the log-evidence. Finding tighter bounds  could boost \name\ performance.}

\section{Use case: counterfactual fairness}\label{sec:fairness}

We finally show
two practical use-cases of our method: assessing counterfactual fairness and training counterfactually fair classifiers.
We use the public German Credit dataset ~\cite{GermanData} 
and
rely on the causal model proposed by~\citet{chiappa2019pathscounterffair} with 
the following 
random variables $\X$:
sensitive feature $S=\{ \textit{sex}\}$,  and non-sensitive features $C=~\{ \textit{age} \}$,  $R=\{\text{\textit{credit amount}, \textit{repayment history}\}}$ and $H = \{\textit{checking account}, \textit{savings}, \textit{housing}\}$. Then, we aim to predict the binary feature $Y=\{\textit{credit risk}\}$ from $\X$. 
See Appendix \ref{apx:fairness} for further details.

 \textbf{Counterfactual fairness.}  Let $S \subset \X$ be a sensitive attribute (e.g., gender), then the counterfactual unfairness~\cite{kusner2017counterfactualfairness} of a classifier $\clf:\X \rightarrow Y$ is measured $\forall \xCF, \alpha' \neq \alpha, y$ as: 
\begin{align}\label{eq:fairness_defi}
& \uf = \mid P(\clf(\xCF) = y \mid do(S = \alpha), \xF)  -  P(\clf(\xCF) =y \mid do(S = \alpha'), \xF) \mid 
 \end{align}

A classifier is counterfactually fair ($\uf=0$), if, given a factual $\xF$ with sensitive attribute $S=\alpha$, had its sensitive attribute been different $S=\alpha'$, the classifier prediction would remain the same.
We can use \name\ to generate counterfactual estimates to \emph{audit} the fairness level of a classifier.
 Following \cite{kusner2017counterfactualfairness}, we \textit{audit}:
i) a \textit{full} model  $\full: \X \rightarrow Y$ that takes as input the complete variable set; ii) an \textit{unaware} model $\unaware: \X \backslash S \rightarrow Y$ that takes as input all variables but the sensitive one; iii) and a \textit{fair} model $\fairx: \{X_i | S \nin \ansi\} \rightarrow Y$ that takes as input all non-descendant variables of the sensitive attribute. 
Moreover, we show that we can 
\emph{learn a fair classifier} $\ourclf: \Z \backslash Z_S \rightarrow Y$, which takes as input the latent variables generated by the \name\ encoder without the one of the sensitive attribute $Z_S$.

\begin{table}[t] 
\centering
\setlength\tabcolsep{3pt
{\small
    \begin{tabular}{l|lll|l}
\toprule
        Metric & 
        full & unaware & fair-x & fair-z\\
    \midrule
     \multirow{1}{*}{$\uparrow$ \acc } 
    & 71.67  & 69.49 & 59.50
    & 70.79 $\pm$ 5.15\\
\hline
    \multirow{1}{*}{$\downarrow$ \uf} 
    & 14.01 $\pm$ 2.26 & 13.27 $\pm$ 2.28 & 0.14 $\pm$ 0.02 
    &    0.51 $\pm$ 0.19\\
\bottomrule
\end{tabular}
}
 \caption{Counterfactual unfairness (\textit{uf}) and f1-score (\textit{f1}) of an SVM over 10 \name\ seeds. Values multiplied by 100.}\label{tab:fairness}
 }
\end{table}

\textbf{Fairness Auditing.} 
  Table~\ref{tab:fairness} summarizes the unfairness level and f1-score for 
  a support vector machine (SVM) classifier. See Appendix \ref{apx:fairness} for results of a logistic regression classifier.
As we do not have access to the true data generation process, we evaluate the \textit{auditing} task by the resulting ranking of the different classifiers according to their unfairness level.
Based on the counterfactual generation by \name\ the \emph{full} classifier is consistently less fair than  the \emph{unaware} and the \emph{fair-x} classifier, respectively.
{This ranking is consistent with the one in \cite{kusner2017counterfactualfairness}}.

\textbf{Fairness Classification}
Table \ref{tab:fairness} shows that for the \textit{fair-x} classifier fairness comes at the expense of accuracy compared to the \textit{full}  classifier. On the contrary, even though \name\ has been trained for representation learning without access to classification labels, \textit{fair-z} is a fair classifier (with comparable fairness level to the \textit{fair-x} one)
 while keeping the performance
 comparable to the unfair \textit{full}
 classifier. 
 \name, therefore, does not only allow us to audit counterfactual fairness but also provides a practical approach to train accurate and fair classifiers. 
%


\section{Conclusion, Limitations and Impact}\label{sec:conclusion}
In this work, we have proposed \name, a variational causal autoencoder based on GNNs that: i) is specially designed to capture the properties of SCMs; ii) inherently handles heterogeneous data; and iii) provides good approximations of interventional and counterfactual distributions as a whole for SCMs of different complexities.
As demonstrated by extensive {synthetic} experiments,  \name\ provides accurate results for a wide range of interventions in  diverse SCMs leading to more consistent results than competing methods~\cite{karimi2020algorithmic,khemakhem2021causal}. Finally, we have  applied  \name\, for  counterfactually fair classification. 

\textbf{Practical limitations.} The expressive power of \name\ to model complex  structural equations, e.g., in domains such as biology \cite{sachs2005causal}, is limited by the GNN architectures of the encoder and the decoder. As discussed  in the GNN literature \cite{corso2020principal}, especially aggregation functions may limit expressiveness. We expect \name\ to benefit from advances in the field.
Second,  long causal paths would
require \name\ to increase the number of layers in the decoder (see \textbf{Design condition 1}). However, the GNNs performance is known to deteriorate with depth \cite{gallicchio2020fast, gu2020implicit, li2018deeper}.

\textbf{Social impact.} Trusting counterfactuals is of great importance for decision making, e.g. in the political or medical  domain. We thus encourage anyone who uses \name\ (or any 
other ML method for causal inference) to  i) fully understand the model assumptions and to verify (up to the possible extend) that they are fulfilled;
as well as  ii) to be aware of the identifiability problem in counterfactual queries \cite{zevcevic2021relating, xia2021causalneural}.

\textbf{Future work.} 
First, it would be important to evaluate the sensitivity of \name\ to errors in  the assumed  causal graph, as well as to the presence of hidden confounders.
We plan to extend  \name\ to handle more complex causal models including, e.g., hidden confounders and non-DAG causal graphs. 
Second, it would be interesting to perform ablation studies on the limitations of available GNNs architectures~\cite{wu2020comprehensive} for the \name\ encoder and decoder; as well as on how the performance deteriorates as we increase the length of the causal path and thus the required number of hidden layers~\cite{li2018deeper}. 
Finally, it would be intriguing to apply \name\ to other causal questions recently discussed in the machine learning literature, such as privacy-preserving causal inference~\cite{kusner2016private} or explainable machine learning~\cite{karimi2020algorithmic}. 

\section*{Acknowledgments}
We would like to thank Adrián Javaloy Bornás, Amir Hossein-Karimi, 
Jonas Kleesen and Maryam Meghdadi Esfahani for  helpful  feedback  and discussions.  Moreover, a special thanks to Diego Baptista Theuerkauf for invaluable help with formalizing proofs. Moreover, the authors would like to thank Ilyes Khemakhem for helpful insights in how to generalize their CAREFL approach to arbitrary graphs.

 Pablo S\'anchez Mart\'in is supported by the German Research Foundation through the Cluster of Excellence “Machine Learning – New Perspectives for Science”, EXC 2064/1, project number 390727645. Miriam Rateike is supported by the German Federal Ministry of Education and Research (BMBF): Tübingen AI Center, FKZ: 01IS18039B. The authors thank for the generous funding support.
The authors also thank the International Max Planck Research School for Intelligent Systems (IMPRS-IS) for supporting Pablo Sa\'nchez Mart\'in.

\clearpage
\bibliographystyle{plainnat} %

\bibliography{refs}

\clearpage

\appendix
\section{Background details on message passing Graph Neural Networks} \label{apx:GNN}

A \textit{directed graph} with $|\Vcal| = d$ nodes can be represented as $\Gcal:=(\Vcal, \Ecal)$, where $\Vcal$ denotes the set of nodes and  $(j, i) \in \Ecal$ is the set of directed edges from $j$ to $i$. Additionally, we denote with  $\X \in \Rbb^{d \times F_X}$  the features of the nodes (the row index identifies the node, i.e., the $i$-th row contains the $F_X$-dimensional features of the node $i$)
Also, the adjacency matrix $\A \in \{0, 1\}^{d \times d}$ of $\Gcal$ is defined as $A_{ij} = 1$ if there is an edge from $j$ to $i$ and  $A_{ij} = 0$ otherwise. Then, a directed graph can be alternatively represented as   $\Gcal = (\X, \A)$.
Given a graph $\Gcal$, a \textit{Graph Neural Network (GNN) with parameters $\theta$} is a function $f_\theta:\Rbb^{d\times F_X}\times \{0,1\}^{d\times d}\rightarrow \Rbb^{d \times F_H}$  that takes into account the graph structure contained in the adjacency matrix $\A \in \{0,1\}^{d\times d}$ and transforms the node features $\X$  into  different features $\mathbf{H} \in \Rbb^{d \times F_H}$, i.e., $\mathbf{H} = f(\X, \A; \theta)$ (for readability we consider $f_\theta(\cdot) \equiv f(\cdot; \theta)$). 
Importantly, at the output of the GNN we have a graph $(\mathbf{H},\A)$ that preserves the structure of the input graph $(\X,\A)$.

\textbf{A GNN based on message passing} \cite{gilmer2017neural} is a type of spatial convolution GNNs \cite{wu2020comprehensive} in which information is passed following the message passing process: In each layer $l$ of the GNN  each node $i$ receives information from its neighbors $\Ncal_i$, a.k.a. the parents of $i$. In a message passing GNN, the feature vector of the $i$-th node at the output of a layer $l$---i.e., $h_i^l$---is computed in three steps:
\begin{enumerate}
    \item \texttt{Message.} The \textit{message from node $j$ to node $i$} is defined as $m_{ij}^l = \fm(h_i^{l-1}, h_j^{l-1}; \theta^l_m)$, where $h_i^{l-1}$ are the features of node $i$ at layer $l-1$,  $h_j^{l-1}$  are the features of node $j$ at the previous layer $l-1$, and $\fm$ is a neural network (usually a linear layer) parametrized by $\theta^l_m$.
    \item \texttt{Aggregator.}  The \textit{aggregator} is  a function in charge of combining all the incoming messages at each node $i$ into a single message, a.k.a. the aggregated message $M_i^l = \fa(\{m_{ij}^l \mid j \in \Ncal_i\})$. Notice that $\fa$ does not have any parameters. Choices of $\fa$ may be the sum, mean, standard deviation, max or min over the inputs, i.e., messages \cite{corso2020principal}.
    \item \texttt{Update.} The \textit{update function} $h_i^l = \fu(h_i^{l-1}, M_i^l; \theta_u^l)$ takes the aggregated message and the representation of node $i$ at layer $l-1$ and outputs the new representation for node $i$ at layer $l$. The function $\fu$ is defined as a neural network (usually a linear layer)  with parameters $\theta^l_u$.
\end{enumerate}

\noindent Putting the three steps together, we obtain the general form of a message passing  based GNN layer as $h_i^l  = \fu \left( h_i^{l-1}, \fa  \left( \{\fm(h_i^{l-1}, h_j^{l-1}; \theta^l_m) \mid j \in \Ncali\} \right) ; \theta^l_u \right)$.

Algorithm \ref{alg:gnn} describes the propagation of information (i.e., messages ) in a GNN with $L$ layers.

\begin{algorithm}[t]
	\SetAlgoLined
	\textbf{Input: } A directed graph $\Gcal$ with $d$ nodes, adjacency matrix $\A$ and node features $\X$.\\
	\textbf{Output: } $\mathbf{H} = \{h_i^L\}_{i=1}^d$. \\
	$h^0_i  = x_i$ $\forall i$ \\
	\For{$l=1, \dots, L$\tcp*{For each layer $l$}}  
	{
	    \For{$i=1, \dots, d$\tcp*{For node $i$}}  
	    {
		    $m_{ij}^l = \fm(h_i^{l-1},h_j^{l-1}; \theta^l_m)$ $\forall j \in \Ncal_i$\tcp*{Compute the messages}
		    $M_i^l = \fa(\{m_{ij}^l\}_{ j \in \Ncal_i})$ \tcp*{Compute the aggregated message}
		     $h_i^l = \fu(h_i^{l-1}, M_i^l; \theta_u^l)$ \tcp*{Compute the node features at layer $l$}
		}
	
	}
	\caption{Message passing GNN with $L$ layers decomposed in three operations}
	\label{alg:gnn}
\end{algorithm}

\section{Proofs}
\label{apx:VCAUSE_proof}

For completeness, this section first formalizes the meaning of causal factorization, interventions and the abduction step in \name.

\textit{\name\ causal factorization} refers to the factorization of the joint distribution as  
\begin{align*}
    \pt(\X \mid \Z, \A) = \prod_{i} p_{\theta_i}(X_i ; \etab_i),
\end{align*}
where the likelihood parameters $\etab_i = \etab_i(\Z_{\ansi})$ are a function of all (and only) the features of the ancestors of $i$ and the features of $i$.

\textit{A \name\ intervention} is performed  by removing all the edges towards the intervened node $i$, such that $\Ncal_i = \varnothing$, while the rest of the edges remains untouched.

 In a \textit{\name\  abduction step}, the posterior distribution factorizes as
\begin{align*}
    q_{\phi}(\Z \mid \X) = \prod_i q_{\phi_i}(Z_i ; \etab_i^\enc),
\end{align*}
where the distribution parameters $\etab_i^\enc =  \etab_i^\enc(\X_{\pasi})$ are a function of all (and only) the features of node  $i$ and the features of its the parents.

\paragraph{Notation. } Consider a \textit{causal graph} $\Gcal:= (\X, \A)$, which is a directed acyclic graph (DAG). 
Let us define a path of length $n$ from node $u$ to node $v$ in $\Gcal$ as $p(u, v)=(u, w_1, w_2,\dots ,  w_{n-1}, v)$, which is an ordered sequence of  unique nodes such that i) there exists an edge in $\Gcal$ between concurrent nodes, ii) the first node is $u$, iii) and the last node is $v$.
We refer to the length of the path as $|p(u,v)|$, i.e., the number of edges in the path, or alternatively, the number of nodes minus one. Let us define $P(u, v)$ as the set of unique paths connecting $u$ to $v$. 
Let us define the \textit{shortest path  from $u$ to $v$} as  $p^{-}(u, v)$ (i.e. the path with the minimum number of edges to go from $u$ to $v$) and its length as $d^-(u, v) = |p^{-}(u,v)|$.
Let us define the \textit{longest path} from $u$ to  $v$ as $p^{+}(u, v)$ (i.e. the path with the maximum number of edges to go from $u$ to $v$) and its length as  $d^+(u, v) = |p^{+}(u,v)|$.
Let us define the \textit{set of ancestors of node $i$} (i.e., $\ani$) as the set of nodes with paths to $i$, i.e., $\{j\mid  | P(j, i)| > 0\}$.
As for a GNN, we define the number of hidden layers (total number of layers minus one) as $\hlayers$.

Then, we define the \textit{diameter $\diameter$} of the graph $\Gcal$ to be the length of the longest shortest path and  $\longpath$ to be the length of the longest path of the graph, which we compute as
\begin{align*}
\diameter = \max_{u, v \in \Gcal} d^-(u, v) \quad \text{ and } \quad \longpath = \max_{u, v  \in \Gcal} d^+(u, v).  
\end{align*}

\begin{lemma} \label{lemma:gnn_paths}
A message passing Graph Neural Network (GNN) has at least $\hlayers$ hidden layers  if and only if  the output feature of every node $i$ (i.e., $h_i^{\hlayers + 1}$) receives information from any other node $j$ via paths $p(j, i)$ such that $|p(j, i)| \leq \hlayers + 1$.
\end{lemma}

\begin{proof}
\textbf{Step 1.} The statement is that  the feature of every node $i$ at the output of a GNN with $\hlayers$ hidden layers (i.e., $h_i^{\hlayers + 1}$) receives information from any other node $j$ via paths $p(j, i)$ such that $|p(j, i)| \leq \hlayers + 1$. We give a proof by induction on $\hlayers$ for an arbitrary node $i$, with input feature to the GNN $h_i^0$ and output feature $h_i^{\hlayers + 1}$.

\textbf{Base case:} The statement holds for $\hlayers =0$. By definition, a message passing GNN with one layer only exchanges messages between neighboring nodes. Hence, i) the output feature of node $i$ is $h_i^1  = f(\{h_i^0\} \cup \{h_j^0 | j \in \Ncal_i \} ; \theta)$, which is only a function of the $1$-hop ancestors  (i.e., parents); ii)  information is exchanged via paths that fulfill $|p(j, i)| \leq 1$.

\textbf{Inductive step:} We assume the statement holds for  $\hlayers=k-1$. In this case, i) the output feature of node $i$ is $h_i^{k}  = f(\{h_i^{k-1}\} \cup \{h_j^{k-1} | j \in \Ncal_i \} ; \theta)$, which is a function of the $k$-hop ancestors; ii)  information is exchanged via paths that fulfill $|p(j, i)| \leq  k$. For $\hlayers=k$, the output feature of node $i$ is $h_i^{k+1}  = f(\{h_i^{k}\} \cup \{h_j^{k} | j \in \Ncal_i \}; \theta )$. Since $h_j^{k}$ is a function of  $k$-hop ancestors of node $j$, it follows that $h_i^{k+1}$ is a function of the ancestors of $h_j^{k}$ and on  parents of node $i$, i.e., the output feature of node $i$ is a function of its $(k+1)$-ancestors. Then, it follows that for $\hlayers=k$, information is exchanged via paths that fulfill $|p(j, i)| \leq  k + 1$.

\textbf{Step 2.} We assume that the output feature of every node $i$  receives information from any other node $j$ via paths $p(j, i)$ such that $|p(j, i)| \leq \hlayers + 1$. 
Then, there exist node $i$ and $j$ such that $|p(j, i)| =  \hlayers + 1$. Then, by definition, a message passing GNN needs at least $\hlayers$ hidden layers to capture paths $p(j, i)$ with length $|p(j, i)| \leq  \hlayers + 1$.
\end{proof}

\textbf{Remark.} Lemma \ref{lemma:gnn_paths} implies that, for a given GNN with $\hlayers$ hidden layers, the set of paths through which the output feature of node $i$ is a function of node $j$ is
\begin{align}
\begin{split}
        P_{\text{GNN}}(j , i) = &\{p(j, i) \mid p(j, i) \in P(j, i) \\
    &\text{ and } |p(j, i)| \leq \hlayers + 1\}.
\end{split}
\end{align}

Additionally, if $|P_{\text{GNN}}(j , i)| = \varnothing$ then the output feature of node $i$ is not a function of node $j$.

\paragraph{Proposition 1 (Causal factorization).}
\textit{\name\ satisfies causal factorization,  $\pt(\X \mid \Z, \A) = \prod_{i} p_{\theta_i}(X_i \mid \Z_{\ansi})$, if and only if the number of hidden layers in the decoder is greater or equal than $\diameter-1$, with $\diameter$ being the length of the longest shortest path between any two endogenous nodes.}

\begin{proof}
Consider a causal graph $\Gcal:= (\X, \E)$ with diameter $\diameter$ and a GNN decoder with  $\hlayers$ hidden layers.
We assume \name\ to satisfy causal factorization, i.e., $\etab_i$ is a function $\Z_{\ansi}$ for all $i$. Therefore,  there exist node $i$ and $j$ such that  $d^-(j, i) =\diameter$ (notice that this implies that $j$ is an ancestor of $i$).  Thus, by Lemma \ref{lemma:gnn_paths}, the GNN decoder has $\hlayers \geq \diameter - 1$ hidden layers. The converse is true because  Lemma \ref{lemma:gnn_paths} is a bi-conditional statement.
\end{proof}

\paragraph{Proposition 2 (Causal interventions).}
\textit{\name\ captures causal interventions if and only if the number of hidden layers in its decoder is greater than or equal to $\gamma-1$, with $\gamma$ being the length of the longest path between any two endogenous nodes in $\Gcal$.}\\

A causal intervention involves to severe all the incoming edges to the intervened nodes. Thus, \name\ can only capture causal interventions, if it can model all the causal paths, i.e., $ P_{\text{GNN}}(j , i) = P(j, i) \qforall i, j$. Otherwise, severing some paths will have no effect in the resulting intervention, as we prove next. 
\begin{proof}
Consider a causal graph $\Gcal:= (\X, \E)$  with  length of the longest path between two nodes~$\longpath$ and a GNN decoder with $\hlayers$ hidden layers.
We assume that the GNN decoder models all the causal paths, i.e.,  $P_{\text{GNN}}(j , i) = P(j, i) \qforall i, j $.  By definition of $\longpath$, there exists at least one node $i$ with an ancestor $j$ such that $d^+(j, i) =\longpath$. Thus, by Lemma \ref{lemma:gnn_paths}, the GNN decoder has $\hlayers \geq \longpath - 1$ hidden layers. The converse is true because  Lemma \ref{lemma:gnn_paths} is a bi-conditional statement.
\end{proof}

\paragraph{Proposition 3 (Abduction).}
\textit{The abduction step of an observed sample $\x= \{x_1, \ldots, x_d\}$ in \name\  satisfies that for all $i$ the posterior of $Z_i$  is independent on the subset  $\{x_j\}_{j \nin \pasi} \subseteq \x$, if and only if the encoder GNN has no hidden layers.}

 \begin{proof}
Consider a causal graph $\Gcal:= (\X, \E)$ and a GNN encoder with $\hlayers$ hidden layers.
We assume the posterior of $Z_i$  is independent on the subset  $\{x_j\}_{j \nin \pasi} \subseteq \x$, i.e., the parameters  $\etab^\enc_i$ (the output of the GNN) is a function of  $\{x_j\}_{j \in \pasi}$. Then, the GNN only models paths $p(j, i)$ such that $d^+(j,i) = 1$. It follows, by Lemma \ref{lemma:gnn_paths}, that the number of hidden layers of the encoder GNN is $\hlayers=0$. The converse is true because  Lemma \ref{lemma:gnn_paths} is a bi-conditional statement.
\end{proof}

\begin{table*}[]
\centering
\centering
\setlength\tabcolsep{2.6pt}

    \begin{tabular}{c cc cc cc}
\toprule
    \multirow{2}{*}{$\hlayers$}     &      \multicolumn{2}{c}{\collider\ ($\diameter=1$, $\longpath=1$)} &  \multicolumn{2}{c}{\triangl\ ($\diameter=1$, $\longpath=2$)}  &  \multicolumn{2}{c}{\chain\ ($\diameter=2$, $\longpath=2$)} \\
          \cmidrule(r){2-3} \cmidrule(r){4-5} \cmidrule(r) {6-7}
 &  MMD Obs. (\%)  & MMD Inter.(\%)    &     MMD  Obs.(\%)          & MMD Inter.(\%)     &      MMD  Obs.(\%)   &  MMD Inter.(\%)     \\
\midrule
0 & 1.86 $\pm$ 0.84 & 1.16 $\pm$ 0.86 & 21.76 $\pm$ 5.80 & 48.47 $\pm$ 12.57 & 9.69 $\pm$ 1.92 & 17.24 $\pm$ 2.82 \\
1 & 1.31 $\pm$ 0.36 & 0.75 $\pm$ 0.17 & 9.17 $\pm$ 2.39 & 19.91 $\pm$ 4.75 & 6.44 $\pm$ 1.52 & 10.36 $\pm$ 2.67 \\
2 & 1.32 $\pm$ 0.41 & 1.09 $\pm$ 1.08 & 7.19 $\pm$ 3.14 & 14.10 $\pm$ 6.45 & 4.35 $\pm$ 1.77 & 6.55 $\pm$ 1.64 \\
\bottomrule
\end{tabular}
\caption{ Evaluation of the observational and interventional distributions generated by  \name\ with different numbers of hidden layers $\hlayers$ {in the decoder}. 
  {All metrics are multiplied by 100 (\%).}}
  \label{tab:obs_ours}
\end{table*}

\section{\name\ implementation details}\label{apx:VCAUSE}
In this section, we  extend Section \ref{sec:heterogenous} and provide further details about the implementation of \name\ for complex real-world datasets and causal graphs.

\subsection{Heterogeneous endogenous variables}
As described in Appendix \ref{apx:GNN}, each layer $l$ of a GNN uses the same parameters $\theta =\{ \theta_m^l, \theta_u^l\}$ (corresponding to $\fm$ and $\fu$) to update the features of every node, i.e., $h^l_i = f(\{ h_i^{l-1} \} \cup \{ h_j^{l-1} \mid j \in \Ncali \}; \theta)$ with $f_\theta$ being reused for all $i$. Nonetheless, the structural equations of an SCM define a unique function $f_i$ for each node (see \textbf{Property  \ref{prop:prediciton}}). 
To mimic this behavior, we will rely on \emph{port numbering}. In particular, for a given causal graph $\Gcal$, we uniquely identify each node with an index $i$  and each edge with the pair of indexes of the nodes it connects. \\

Then, we define a 
\emph{disjoint GNN layer} by the following characteristics:

\begin{itemize}
    \item The node indexes define unique update functions  $\fu_i$ for each node, with parameters $\theta_{ui}$.
    \item The edge indexes define unique message functions $\fm_{ij}$ for each edge, with parameters $\theta_{mij}$.
\end{itemize}

Consequently, parameters are not shared among nodes and we can mimic the diversity of the structural equations of an SCM and model heterogeneous endogenous variables. 
In our GitHub repository\footnote{\url{https://github.com/psanch21/VACA}}, we present a PyTorch Geometric implementation of the \emph{disjoint GNN layer}.

\subsection{Heterogeneous causal nodes}
Assume an SCM with $d$ endogenous variables. As described in Section \ref{sec:heterogenous}, is it possible to model an endogenous variable $\X_i$ of the SCM as a heterogeneous node, i.e.,  $\X_i = \{X_{i1},  \dots, X_{ik_i}\}$, where $k_i$ is the number of random variables in node $i$. In this section we describe the implications this has on the design of the encoder and decoder of \name.

\paragraph{Implications for the decoder.} Given the heterogeneous nature of the nodes, the likelihood of \name\ factorizes as follows
\begin{align*}
\begin{split}
     &p_{\theta}(\X \mid \Z) = \prod_{i=1}^d p_{\theta_i}(\X_i; \etab_i) =  \prod_{i=1}^d  \prod_{j=1}^{k_i} p_{\theta_i}(X_{ij} ; \etab_{ij}) \\
     &\text{ where } \etab_i = \etab_i(\Z_{\ansi}) \text{ and } \etab_{ij} = \etab_{ij}(\Z_{\ansi}).
\end{split}
\end{align*}

\begin{figure*}
    \centering
    \includegraphics[width=0.6\textwidth]{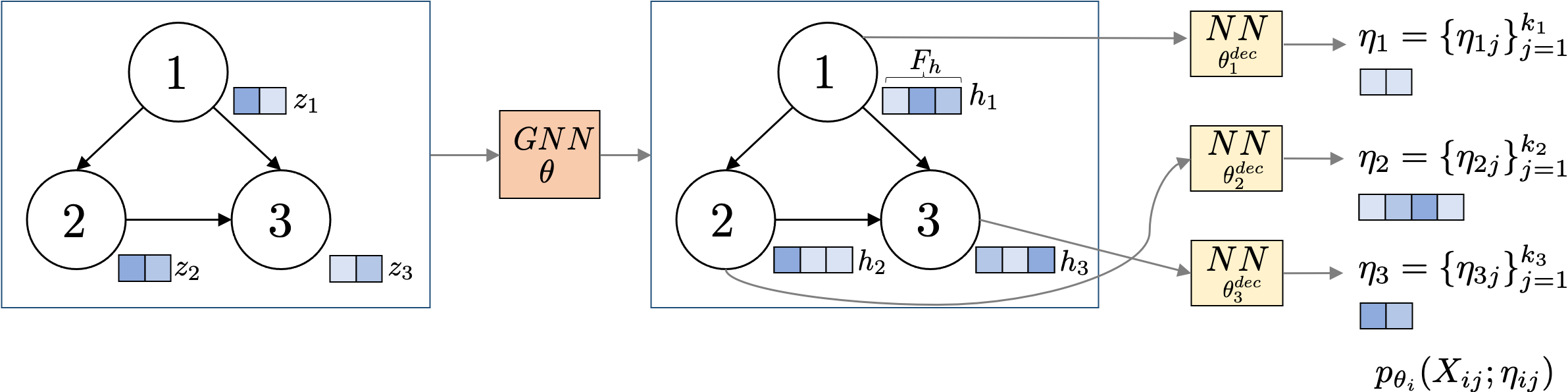}
    \caption{Heterogeneous  \name\ decoder architecture.}
    \label{fig:GNN_het_decoder}
\end{figure*}
\begin{figure*}
    \centering
    \includegraphics[width=0.6\textwidth]{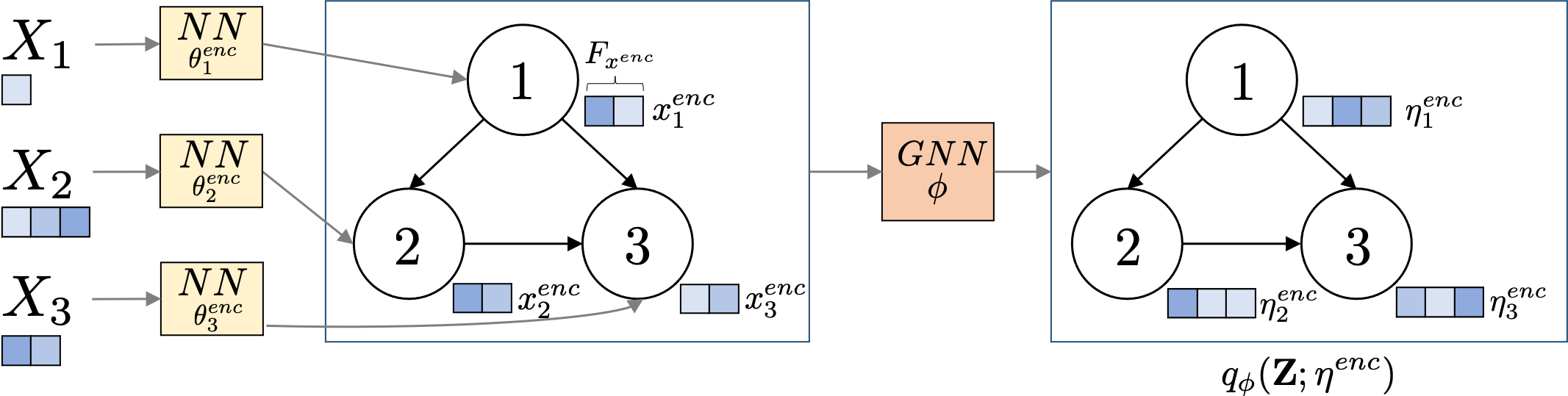}
     \vspace{-3pt}
    \caption{Heterogeneous \name\ encoder architecture. }
    \label{fig:GNN_het_encoder}
     \vspace{-3pt}
\end{figure*}

Note, each $p(X_{ij}; \etab_{ij})$ can be model with a different distribution, e.g., Gaussian or categorical.
This means that the likelihood parameters $\etab_{ij}$ for each random variable $X_{ij}$ may be different depending on its distribution type.
However, the decoder GNN  transforms the latent features into different features $\mathbf{H} \in \Rbb^{d \times F_h}$, where the feature vector of each node $i$ has the same dimensionality $F_h $. As a consequence, $\mathbf{H}$ cannot model the diversity in the likelihood parameters $\etab_i$.  
To overcome such limitation, we add at the output of the GNN decoder a neural network (NN) per node $i$ with parameters $\theta_i^\dec$. Such a NN transforms $h_i$, i.e. the output features of node $i$, into the set of likelihood parameters of each node $\etab_i=\{\etab_{ij}\}^{k_{i}}_{j=1}$, such that the likelihood parameters of each random variable $\etab_{ij}$ satisfy the constraints of the corresponding likelihood $p(X_{ij};\etab_{ij})$ (e.g., non-negativity of variance for a Gaussian distribution). See Figure~\ref{fig:GNN_het_decoder} for an illustration.

\paragraph{Implications for the encoder.} 
Due to the heterogeneous nature of nodes, each endogenous variable $\X_i = \{X_{i1},  \dots, X_{ik_i}\}$, can have a different number of random variables $k_i$ and thus the node $i$ corresponding to it in the GNN will have features of different dimensions. However, as described in Appendix \ref{apx:GNN}, a GNN takes as input in general a matrix feature $\X \in \mathbb{R}^{d\times F_{x^{enc}}}$. This implies, the features of every node share the same dimensionality $F_{x^{enc}}$. 
To overcome this limitation, we include for each node $i$ a neural network (NN) with parameters $\theta^{\enc}_i$ that transforms the corresponding heterogeneous random variable $\X_i$ into a feature vector with the dimension $F_{x^{enc}}$. See Figure~\ref{fig:GNN_het_encoder} for an illustration.

\section{Validating  \& understanding \name }\label{apx:vaca_analysis}
Here, we empirically validate the design conditions of \name\ and provide an analysis of the effect of the latent space dimension $\dim \bld{z}$ on the performance.

\paragraph{Validating \name\ design conditions.} 
In a first step we empirically validate our design choices for the \name\ encoder and decoder. 
We show how the number of hidden layers $N_h$ in the decoder affects the quality of the estimation of the observational and interventional distribution. We do so for three SCMs, with different values of longest shortest directed path $\delta$ and longest directed path $\gamma$. 
 Our observations in Table~\ref{tab:obs_ours} match our expectations. The \collider\ ($\delta= \gamma=1$) does not need any hidden layer to provide accurate estimate of both the observational and interventional distributions. In contrast, the \triangl\ ($\delta =1, \gamma=2$) -- in accordance  with \textbf{Proposition~\ref{prop:CaVAEint}} -- needs at least one hidden layer to get a more accurate estimate of the interventional distribution. Finally, as stated by \textbf{Propositions~\ref{prop:CaVAEobs}}~and~\textbf{\ref{prop:CaVAEint}}, the  \chain\  ($\delta= \gamma=2$) requires at least one hidden layer to accurately approximate both the observational and interventional distributions.

\paragraph{Analysis of the latent space dimension.} 

Here we present an analysis of the performance of \name\ with respect to the dimension of the latent space $\dim \bld{z}$. Without loss of generality, we focus the analysis on the \collider, \triangl\ and \chain\ graphs, and the \nonlin\ and \nonadd\ structural equations. The results are depicted in Figure \ref{fig:dim_z_in_VCAUSE}. All results are averaged over 10 different initializations. 
Regarding the observational distribution (left column),
we observe that \name\ overfits as we increase the $\dim \mathbf{z}$ for all SEMs and graphs under consideration. 
Specially, the performance degrades notoriously with $\dim \mathbf{z} = 32$. As shown in Table \ref{tab:num_params}, the number of parameters of \name\ increases linearly with $\dim \mathbf{z}$.  As for the interventional distribution (middle column), we observe similar behavior: for large values of  $\dim \mathbf{z}$ performance decreases and variance increases. For counterfactuals (right column), we observe a similar behavior but considerably less pronounced. 
In summary, we encourage practitioners to keep $\dim \mathbf{z}$ small to avoid overfitting and obtain better and more consistent performance.

 \begin{figure}[htbp!]
    \centering
    \includegraphics[width=0.4\textwidth]{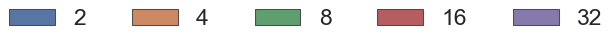}\\
    \includegraphics[width=0.23\textwidth]{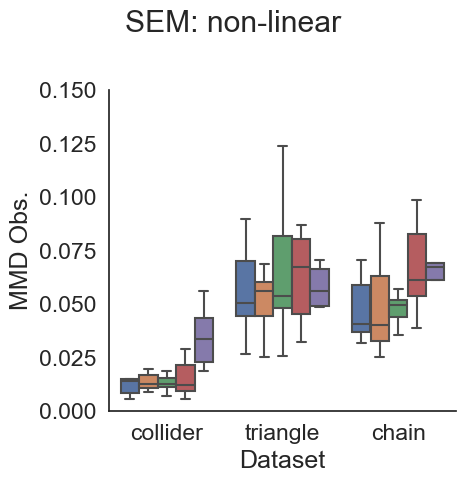}
        \includegraphics[width=0.22\textwidth]{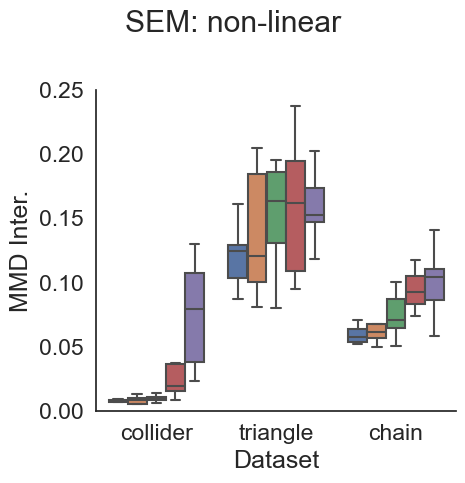}
            \includegraphics[width=.22\textwidth]{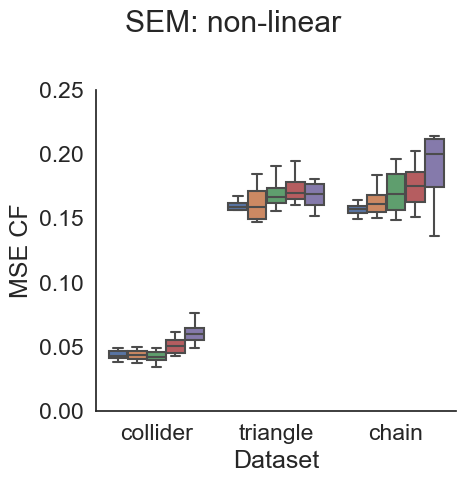}\\
        \vspace{0.2cm}
    \includegraphics[width=.23\textwidth]{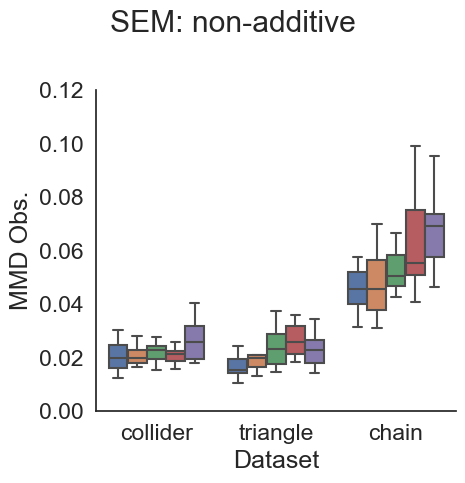}
    \includegraphics[width=.22\textwidth]{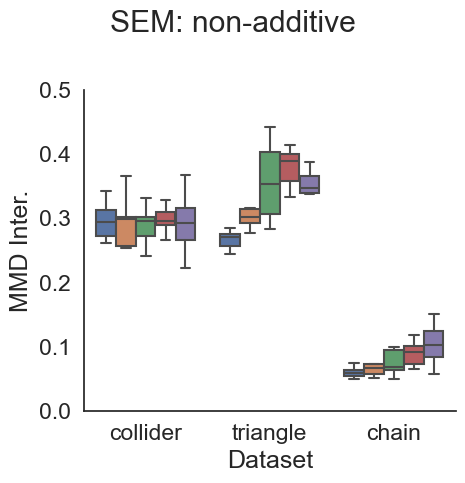}
     \includegraphics[width=.22\textwidth]{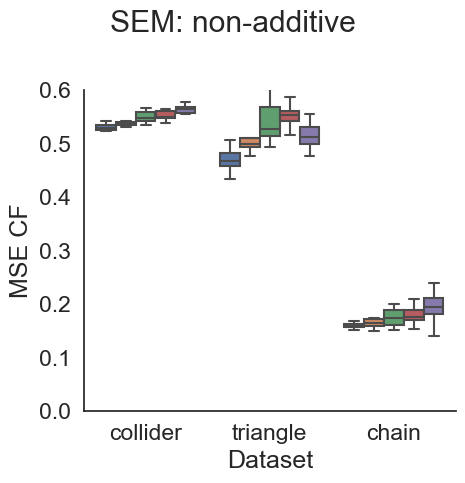}\\
    \caption{Performance of \name\ varying the dimension of the latent space $\dim \boldsymbol{z} \in \{2, 4, 8, 16, 32\}$ for the non-linear (top row) and non-additive (bottom column) SEMs of the \collider, \triangl\ and \chain\ graphs. We evaluate the quality of the observational (MMD Obs.) (left column), interventional (MMD Inter.) (middle column), and counterfactual (MSE CF) (right column) distributions.}
    \label{fig:dim_z_in_VCAUSE}
\end{figure}

\begin{table}[t]
    \centering
\small{
\setlength\tabcolsep{3pt}
    \centering
    \begin{tabular}{c   cc cc cc}
\toprule
 &  \multicolumn{2}{c}{\collider}  &  \multicolumn{2}{c}{\triangl}  & \multicolumn{2}{c}{\chain}  \\
   \cmidrule(r){2-3}   \cmidrule(r){4-5}   \cmidrule(r){5-7} 
  $\dim \mathbf{z}$  &  \nonlin & \nonadd  &  \nonlin & \nonadd  &  \nonlin & \nonadd    \\
  \cmidrule(r){2-7} 
  2 & 3785 & 3785 & 4542 & 4542 & 3785 & 3785 \\
  4 & 4445 & 4445 & 5334 & 5334 & 4445 & 4445 \\
  8 & 5765 & 5765 & 6918 & 6918 & 5765 & 5765 \\
  16 & 8405 & 8405 & 10086 & 10086 & 8405 & 8405 \\
  32 & 13685 & 13685 & 16422 & 16422 & 13685 & 13685 \\

\bottomrule
\end{tabular}\caption{Number of parameters of \name\ for different $\dim \mathbf{z}$ and  datasets}\label{tab:num_params}
}
\vspace{-5pt}
\end{table}

\newpage
\section{Experiments: setting, metrics  and further results}\label{apx:training}

This section provides a complete description of the experimental set-up, including the (semi-)synthetic datasets (Section \ref{apx:training_datasets}), training of \name, \mcvae\ \cite{karimi2020algorithmic} and  \carefl\ \cite{khemakhem2021causal} (Section \ref{apx:training_implementation}), metrics reported in the experiments (Section \ref{apx:training_metrics}), additional results (Section \ref{apx:additional_results}), complexity of the algorithm (Section \ref{apx:time_analysis}), and computing infrastructure (Section \ref{apx:infrastructure}).

\subsection{Datasets}\label{apx:training_datasets}
The following (semi-)synthetic datasets are taken from or inspired by \cite{karimi2020algorithmic}. The distribution of exogenous variables $p(\U)$  for \triangl, \chain\ and \collider\  follows Table \ref{tab:synth_exogen} with with $\operatorname{MoG}$ denoting a mixture of Gaussian distributions.

\begin{table*}[!htbp]
 \centering
    \begin{tabular}{llll}
\toprule
  SCM & $p(U_1)$ & $p(U_2)$ & $p(U_3)$\\
\midrule
\lin &  $ \operatorname{MoG}(0.5 \mathcal{N}(-2,1.5)+0.5 \mathcal{N}(1.5,1))$ & $\mathcal{N}(0,1)$ & $\mathcal{N}(0,1)$ \\ 
\nonlin & $ \operatorname{MoG}(0.5 \mathcal{N}(-2,1.5)+0.5 \mathcal{N}(1.5,1))$  &  $\mathcal{N}(0,0.1)$& $\mathcal{N}(0,1)$ \\
\nonadd & $ \operatorname{MoG}(0.5 \mathcal{N}(-2.5,1)+0.5 \mathcal{N}(2.5,1))$ & $\mathcal{N}(0,0.25)$ & $\mathcal{N}(0,0.0625)$\\ 
\bottomrule
\hfill
\end{tabular}
\caption{Distribution of exogenous variables $p(\U)$ for SCM \triangl, \chain, \collider.}
    \label{tab:synth_exogen}
\end{table*}

\begin{table*}[!htbp]
 \centering
    \begin{tabular}{lllll}
\toprule
& SCM  & $\fx_1:=X_1$ & $\fx_2:=X_2$ & $\fx_3:=X_3$ \\
\hline \hline
 \multirow{3}{*}{\rotatebox[origin=c]{90}{ {\collider}}} &
\lin &  $U_1 $ & $U_2$ & $0.05 X_1 + 0.25 X_2 + U_3$  \\ 
&\nonlin & $U_1$ &  $U_2$ & $0.05 X_1 + 0.25 (X_2)^2 + U_3$  \\
&\nonadd & $U_1$ & $U_2$ & $-1 + 0.1 \operatorname{sgn}(U_3) ((X_1)^2 + (X_2)^2) U_3$ \\ 
\hline \hline
 \multirow{3}{*}{\rotatebox[origin=c]{90}{ {\triangl}}}
 & \lin &  $U_1$ & $ -X_1 +U_2$ & $X_1 + 0.25 X_2 + U_3$ \\ 
& \nonlin & $U_1$ &  $-1 + \frac{3}{(1 + \operatorname{exp}(-2 X_1))} + U_2$ & $X_1 + 0.25 (X_2)^2 + U_3$  \\
& \nonadd & $U_1$ & $0.25 \operatorname{sgn}(U_2) * (X_1)^2 (1 + (U_2)^2)$ & $-1 + 0.1 \operatorname{sgn}(U_3) ((X_1)^2 + (X_2)^2) + U_3$\\ 
\hline \hline
 \multirow{3}{*}{\rotatebox[origin=c]{90}{ {\chain}}}& \lin &  $U_1$ & $-X_1 + U_2$ & $0.25 * X_2 + U_3$ \\ 
& \nonlin & $U_1$ &  $-1 + \frac{3}{(1 + \operatorname{exp}(-2 X_1))} + U_2$ & $0.25 * (X_2)^2 + U_3$  \\
& \nonadd & $U_1$ & $0.25 \operatorname{sgn}(U_2) (X_1)^2 (1 + (U_2)^2)$ & $-1 + 0.1 \operatorname{sgn}(U_3) ((X_2)^2) + U_3$  \\ 
\bottomrule
\hfill
\end{tabular}
    \caption{Structural equations $\Fx$ for different SCMs with $\U \sim p(\U)$ in Table \ref{tab:synth_exogen}.
    Function $\operatorname{sgn}(x)$ returns an element-wise indication of the sign of $x$.
    }
    \label{tab:3synth_eq}
\end{table*}

\paragraph{Collider.}
The \collider\ is a synthetic dataset, which consists of 3 endogenous variables. The 
structural equations are shown in Table~\ref{tab:3synth_eq}. Figure~\ref{fig:collider_graph} illustrates the corresponding causal graph with $d=|\X| = 3$ nodes, diameter $\diameter =1$ and  longest path $\longpath=1$.

\begin{figure}[htbp!]
    \centering



\begin{tikzpicture}
        \node[state, fill=gray!60] (x3) at (0,0) {$X_3$};
        \node[state, fill=gray!60] (x2) [above right = 0.60 cm of x3] {$X_2$};
        \node[state, fill=gray!60] (x1) [above left = 0.60 cm of x3] {$X_1$};

        \path (x1) edge [thick](x3);
        \path (x2) edge [thick](x3);

\end{tikzpicture}
    \caption{Causal graph for variables $\X$ of SCM \collider.} \label{fig:collider_graph}
  \end{figure}
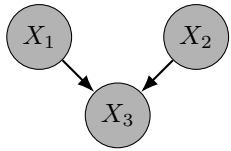

\pagebreak
\paragraph{Triangle.}
The \triangl\ is a synthetic dataset, which consists of 3 endogenous variables. 
The structural equations are shown in Table~\ref{tab:3synth_eq}. Figure~\ref{fig:triangle_graph} illustrates the corresponding causal graph with $d=|\X| = 3$ nodes, diameter $\diameter =1$ and longest path $\longpath = 2$. 

\begin{figure}[htbp!]
    \centering


\begin{tikzpicture}
        \node[state, fill=gray!60] (x1) at (0,0) {$X_1$};
        \node[state, fill=gray!60] (x3) [below right = 0.60 cm of x1] {$X_3$};
        \node[state, fill=gray!60] (x2) [below left = 0.60 cm of x1] {$X_2$};

        \path (x1) edge [thick](x3);
         \path (x1) edge [thick](x2);
        \path (x2) edge [thick](x3);

\end{tikzpicture}
    \captionof{figure}{Causal graph for variables $\X$ of SCM \triangl.} \label{fig:triangle_graph}
\end{figure}
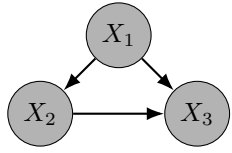

\paragraph{Chain.}
The \chain\ is a synthetic dataset, which consists of 3 endogenous variables. The structural equations are shown in Table~\ref{tab:3synth_eq}. Figure~\ref{fig:chain_graph} illustrates the corresponding causal graph with $d=|\X| = 3$ nodes, diameter $\diameter =2$  and  longest path $\longpath=2$.

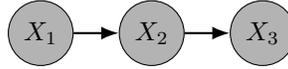
\begin{figure}[h!]
    \centering
        

        

\begin{tikzpicture}
        
        \node[state, fill=gray!60] (x1) at (0,0) {$X_1$};
        \node[state, fill=gray!60] (x2) [right = 0.60 cm of x1] {$X_2$};
        \node[state, fill=gray!60] (x3) [right = 0.60 cm of x2] {$X_3$};

        \path (x1) edge [thick](x2);
        \path (x2) edge [thick](x3);
        
\end{tikzpicture}
    \captionof{figure}{Causal graph for variables $\X$ of SCM \chain.} \label{fig:chain_graph}
\end{figure}

\paragraph{M-graph.}
The \mgraph\ is a synthetic dataset, which consists of 5 endogenous variables. Here, the distributions of exogenous variables follow $U_i \sim p(U_i) = \mathcal{N}(0,1) \qforall i \in {1 \dots 5}$. The structural equations are shown in Table~\ref{tab:mgraph_eq} and Figure~\ref{fig:mgraph_graph} illustrates the corresponding causal graph with $d=|\X| = 5$ nodes, diameter $\diameter =1$  and  longest path $\longpath=1$.

\begin{figure}[htbp!]
    \centering



\begin{tikzpicture}
        \node[state, fill=gray!60] (x4) at (0,0) {$X_4$};
        \node[state, fill=gray!60] (x2) [above right = 0.60 cm of x4] {$X_2$};
        \node[state, fill=gray!60] (x1) [above left = 0.60 cm of x4] {$X_1$};
        \node[state, fill=gray!60] (x5) [below right = 0.60 cm of x2] {$X_5$};
        \node[state, fill=gray!60] (x3) [below left = 0.60 cm of x1] {$X_3$};

        \path (x1) edge [thick](x3);
        \path (x2) edge [thick](x4);
         \path (x1) edge [thick](x4);
        \path (x2) edge [thick](x5);

\end{tikzpicture}
    \captionof{figure}{Causal graph for variables $\X$ of SCM \mgraph.} \label{fig:mgraph_graph}
\end{figure}
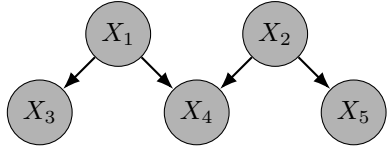

\begin{table*}[htbp!]
 \centering
\small{
    \begin{tabular}{llllll}
\toprule
 SCM  & $\fx_1:=X_1$ & $\fx_2:=X_2$ & $\fx_3:=X_3$ & $\fx_4:=X_4$ & $\fx_5:=X_5$ \\
\midrule
\lin &  $U_1$ & $U_2$ & $X_1 + U_3$ & $-X_2 + 0.5 X_1 + U_4$ & $-1.5 X_2 + U_5$\\ 
\nonlin & $U_1$ &  $U_2$ & $ X_1 + 0.5 (X_1)^2 + U_3$ & $-X_2 + 0.5 (X_1)^2 + U_4$ & $ -1.5 (X_2)^2 + U_5$\\
\nonadd & $U_1$ & $U_2$ & $X_1 * U_3$ & $(-X_2 + 0.5 * (X_1)^2) U_4$ & $(-1.5 (X_2)^2) U_5$\\ 
\bottomrule
\hfill
\end{tabular}
}
    \caption{Structural Equations $\Fx$ for SCM \mgraph\ with $U_i \sim p(U_i) = \mathcal{N}(0,1) \, \qforall i \in {1 \dots 5}$. }
    \label{tab:mgraph_eq}
\end{table*}

 \begin{figure*}[!htbp]
    \centering
    \begin{subfigure}{.3\textwidth}
    \centering
    \includegraphics[width=0.8\textwidth]{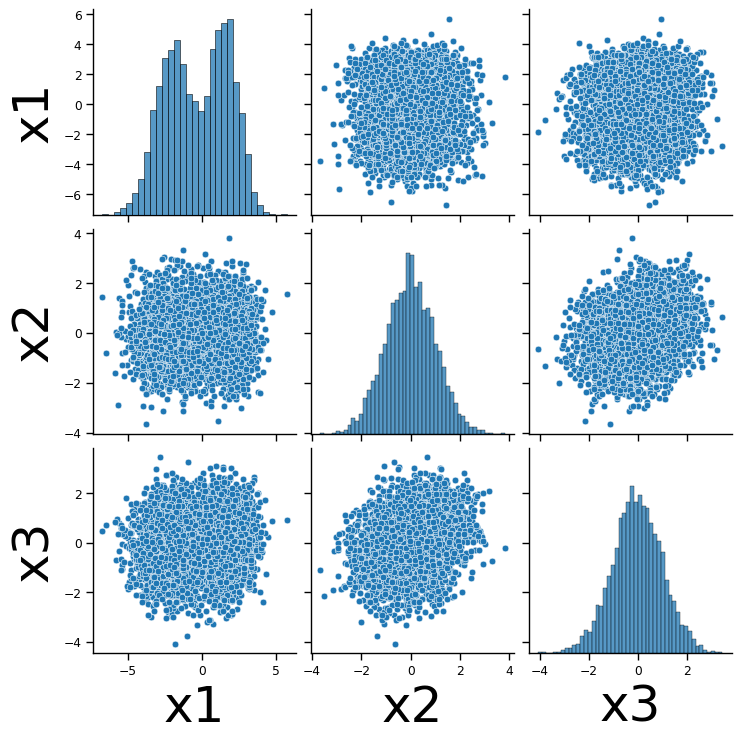}
    \caption{{\collider\  \lin}} \label{fig:collider_LIN_plot}
  \end{subfigure}
      \centering
 \centering
  \begin{subfigure}{.32\textwidth}
    \centering
    \includegraphics[width=0.8\textwidth]{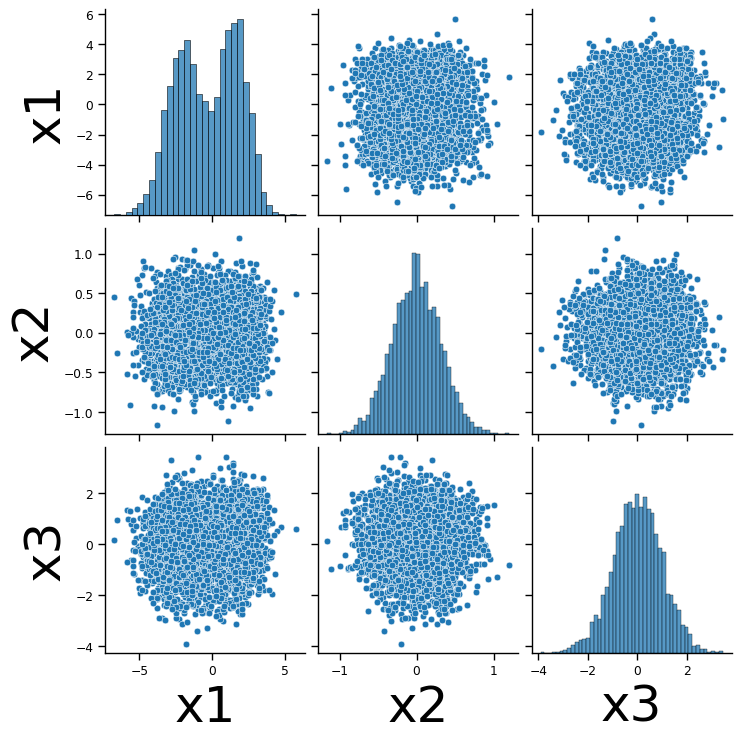}
    \caption{{\collider\ \nonlin}} \label{fig:collider_NLIN_plot}
  \end{subfigure}
   \centering
    \begin{subfigure}{.32\textwidth}
    \centering
    \includegraphics[width=0.8\textwidth]{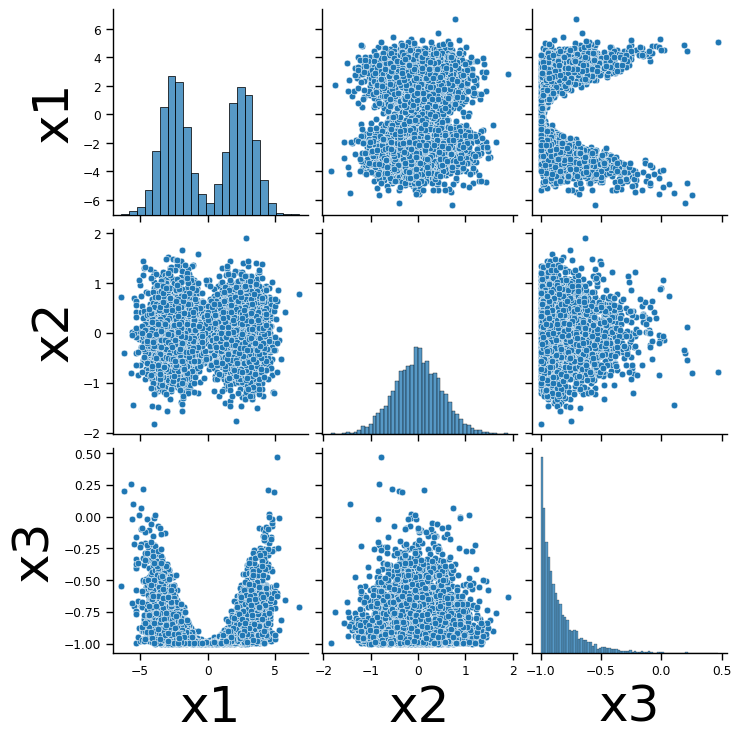}
    \caption{{\collider\  \nonadd}} \label{fig:collider_NADD_plot}
    \end{subfigure}
      \centering
    \begin{subfigure}{.3\textwidth}
    \centering
    \includegraphics[width=0.8\textwidth]{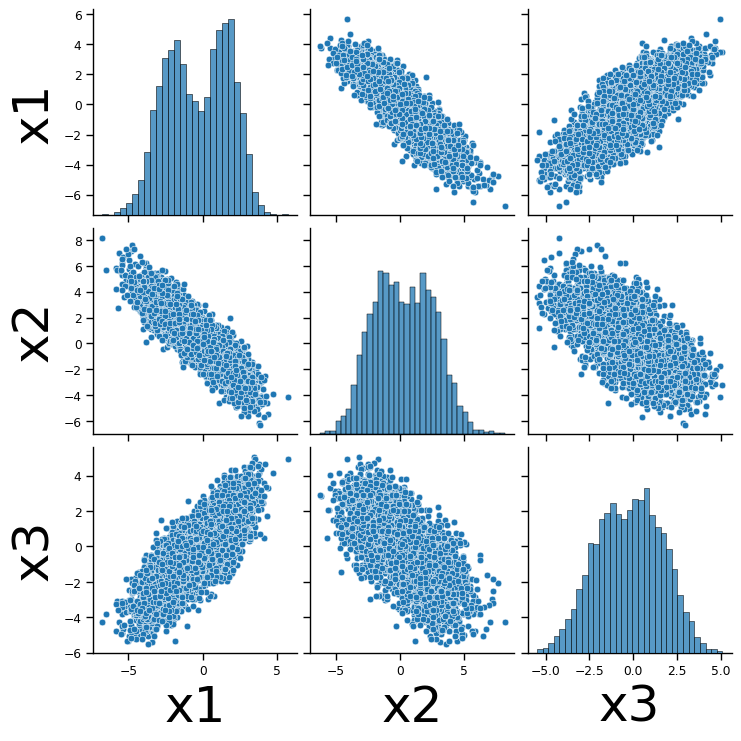}
    \caption{{\triangl\ \lin}} \label{fig:triangle_LIN_plot}
  \end{subfigure}
      \centering
 \centering
  \begin{subfigure}{.32\textwidth}
    \centering
    \includegraphics[width=0.8\textwidth]{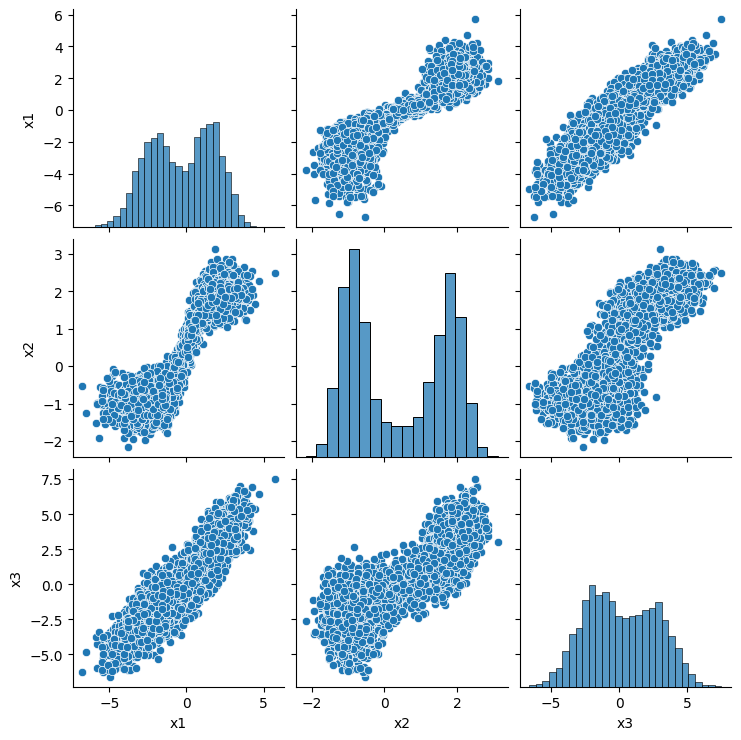}
    \caption{{\triangl\ \nonlin}} \label{fig:triangle_NLIN_plot}
  \end{subfigure}
   \centering
    \begin{subfigure}{.32\textwidth}
    \centering
    \includegraphics[width=0.8\textwidth]{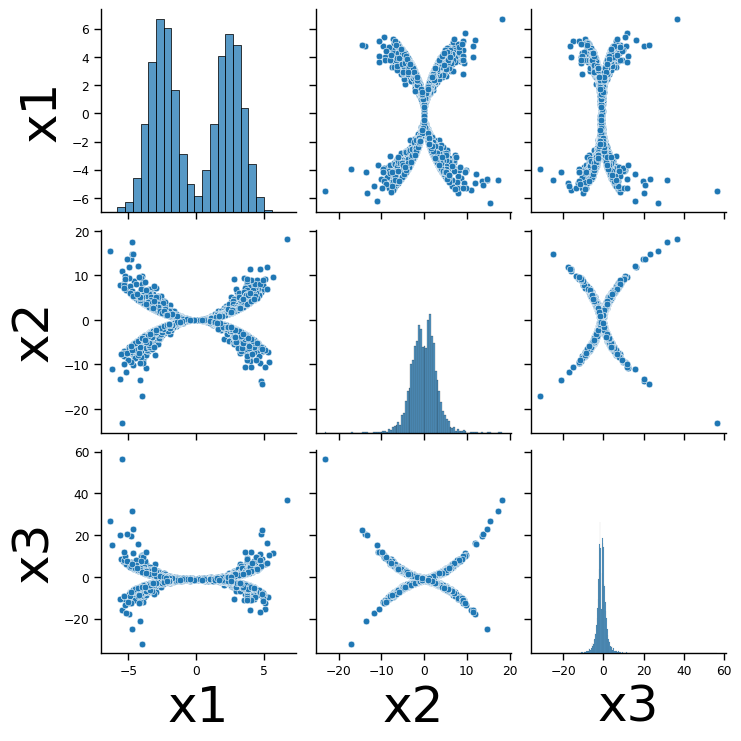}
    \caption{{\triangl\ \nonadd}} \label{fig:triangle_NADD_plot}
    \end{subfigure}
       \centering
    \begin{subfigure}{.3\textwidth}
    \centering
    \includegraphics[width=0.8\textwidth]{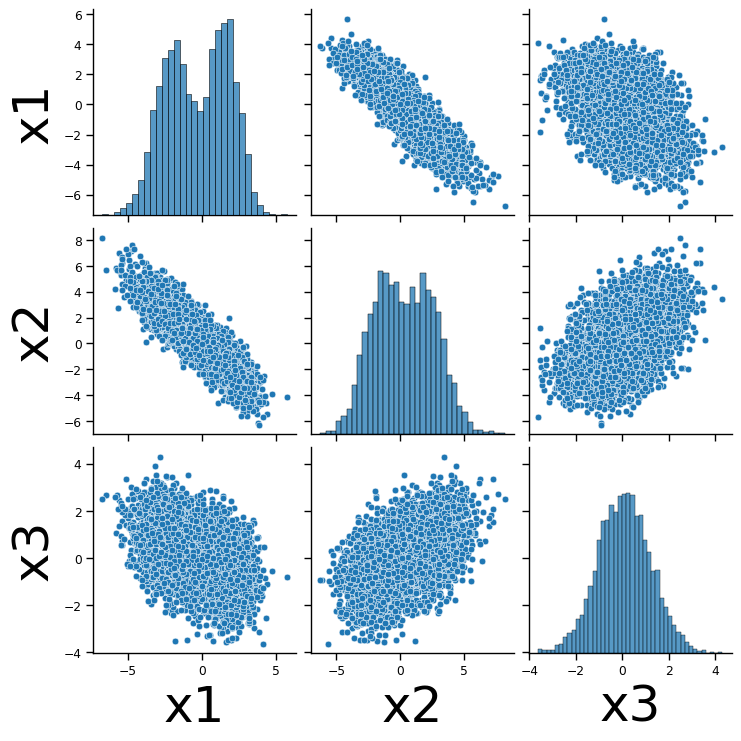}
    \caption{{\chain\ \lin}} \label{fig:chain_LIN_plot}
  \end{subfigure}
      \centering
 \centering
  \begin{subfigure}{.32\textwidth}
    \centering
    \includegraphics[width=0.8\textwidth]{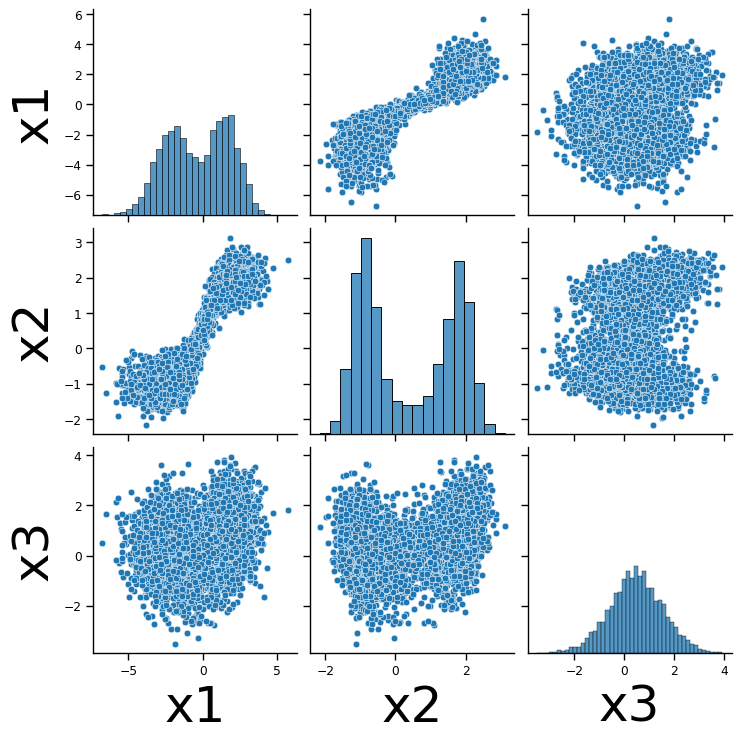}
    \caption{{\chain\  \nonlin}} \label{fig:chain_NLIN_plot}
  \end{subfigure}
   \centering
    \begin{subfigure}{.32\textwidth}
    \centering
    \includegraphics[width=0.8\textwidth]{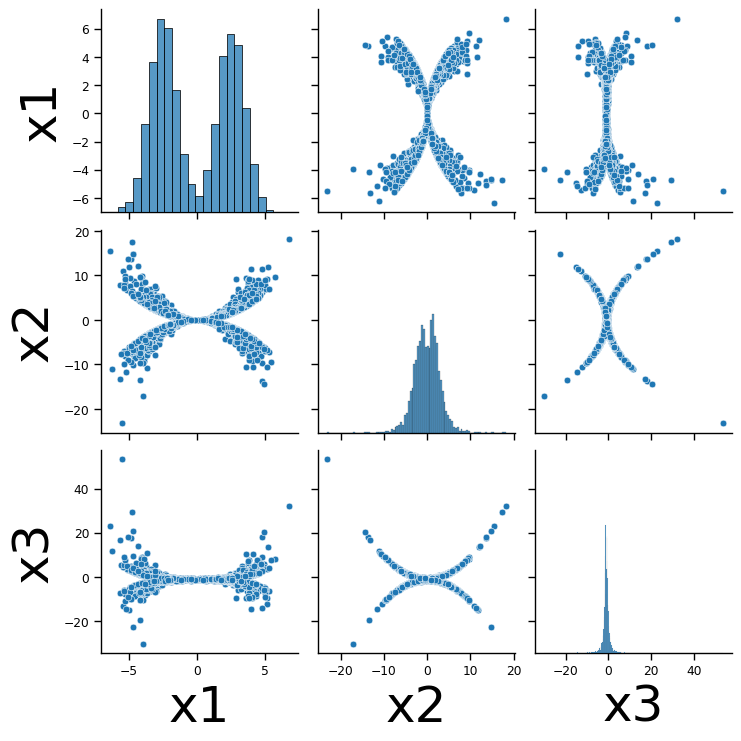}
    \caption{{\chain\ \nonadd}} \label{fig:chain_NADD_plot}
    \end{subfigure}
        \centering
    \begin{subfigure}{.3\textwidth}
    \centering
    \includegraphics[width=0.8\textwidth]{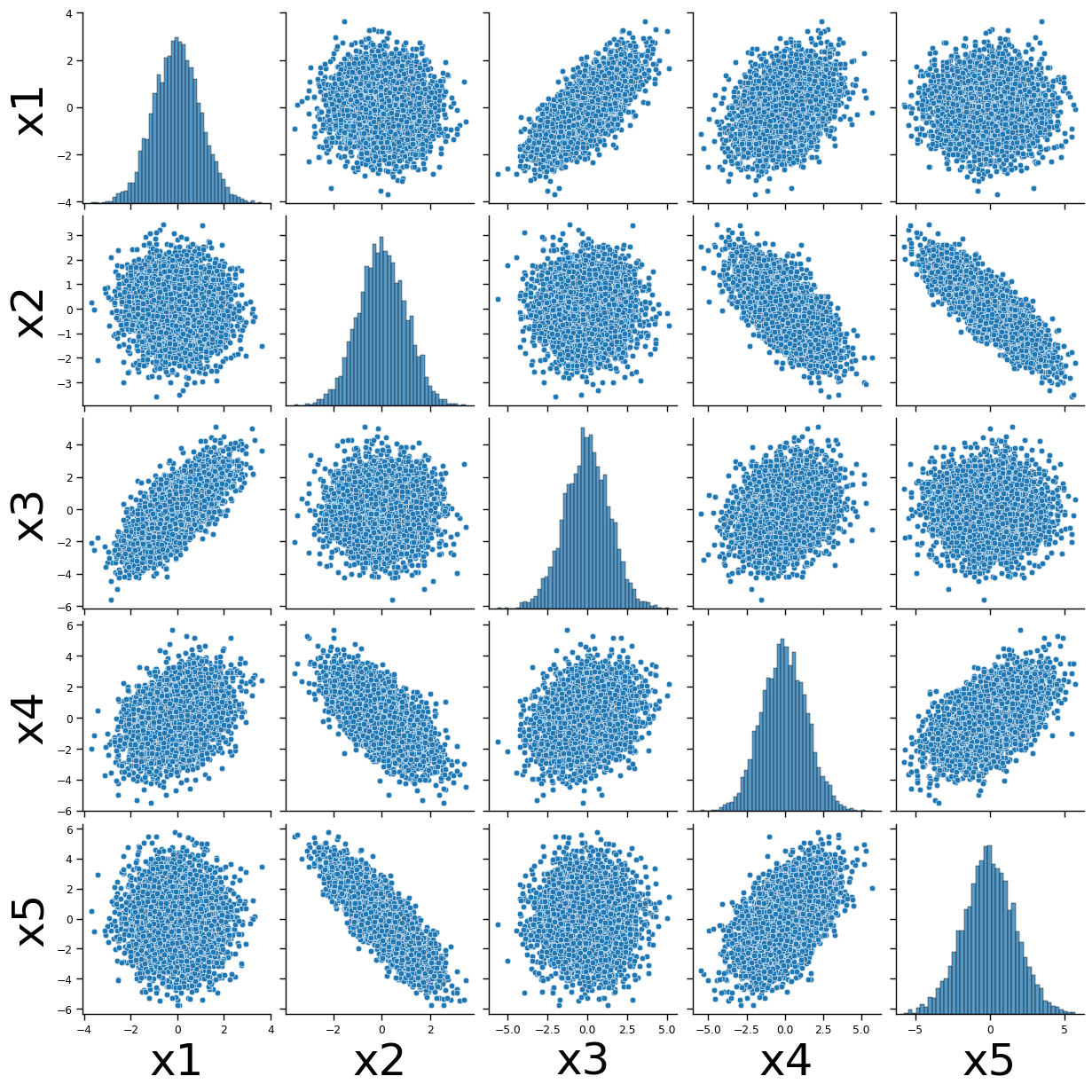}
    \caption{{\mgraph\ \lin}} \label{fig:m_graph_LIN_plot}
  \end{subfigure}
      \centering
 \centering
  \begin{subfigure}{.32\textwidth}
    \centering
    \includegraphics[width=0.8\textwidth]{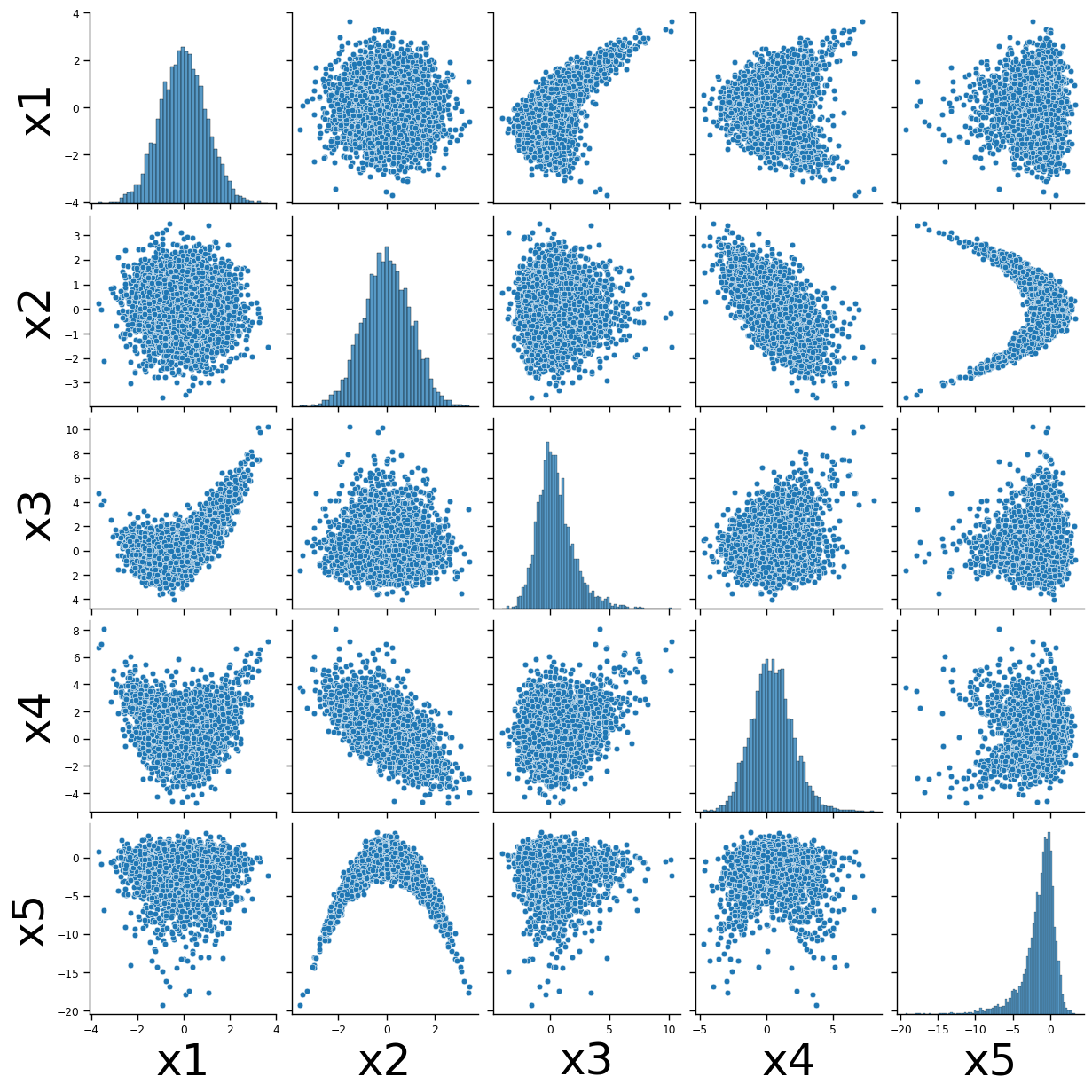}
    \caption{{\mgraph\ \nonlin}} \label{fig:m_graph_NLIN_plot}
  \end{subfigure}
   \centering
    \begin{subfigure}{.32\textwidth}
    \centering
    \includegraphics[width=0.8\textwidth]{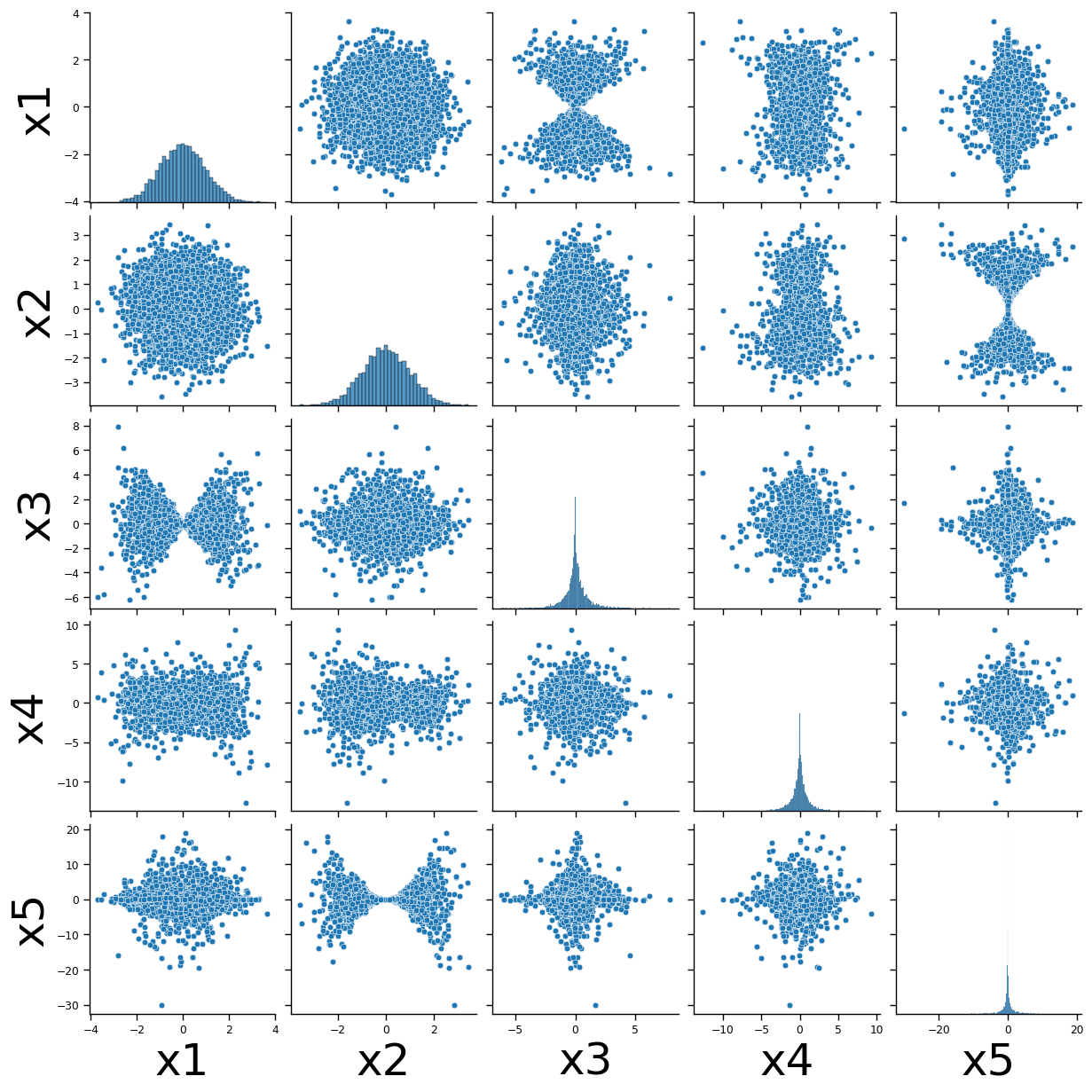}
    \caption{{\mgraph\ \nonadd}} \label{fig:m_graph_NADD_plot}
    \end{subfigure}
\caption{Histograms and scatter plots of pairwise feature relations for the \collider, \triangl, \chain\ and \mgraph\ datasets.}
 \end{figure*}

\newpage
\paragraph{Loan.}
The \loan\ is a semi-synthetic dataset from \cite{karimi2020algorithmic}, which reflects a loan approval setting in the real-world inspired by German Credit dataset \cite{GermanData}. It consists of 7 endogenous variables: \textit{gender} $G$, \textit{age} $A$,  \textit{education} $E$, \textit{loan amount} $L$, \textit{loan duration} $D$, \textit{income}  $I$ and \textit{savings} $S$ with the following structural equations and distributions of exogenous variables: 
\begin{align*}\label{eq:loan_eq}
    \begin{split}
    f_G : &G =U_{G} \\
    f_A : &A=-35+U_{A} \\
    f_E : &E=-0.5+\left(1+e^{+1-0.5 G-\left(1+e^{-0.1 A}\right)^{-1} -U_{E}}\right)^{-1} \\
    f_L : &L=1+0.01(A-5)(5-A)+G+U_{L} \\
    f_D : &D=-1+0.1 A+2 G+L+U_{D} \\
    f_I : &I=-4+0.1(A+35)+2 G+G E+U_{I} \\
    f_S : &S=-4+1.5 \mathbb{I}_{\{I>0\}} I+U_{S} 
    \end{split}
\end{align*}

\noindent with $ U_{G}~\sim \operatorname{Bernoulli}(0.5)$, $U_{A}~\sim \operatorname{Gamma}(10,3.5)$, 
$U_{E}~\sim \mathcal{N}(0,0.25)$, $ U_{L}~\sim~\mathcal{N}(0,4)$,  $U_{D}~\sim~\mathcal{N}(0,9)$, $U_{S}~\sim~\mathcal{N}(0,25)$, $U_{I}~\sim~\mathcal{N}(0,4)$.

Note, the authors model variables w.r.t. their relative meaning in terms of deviation from the mean. See \cite{karimi2020algorithmic} for further details. Figure~\ref{fig:loan_graph} illustrates the corresponding causal graph with $d=|\X| = 7$ nodes, diameter $\diameter =2$ and  longest path $\longpath=3$.

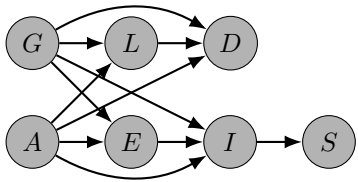
\begin{figure}[htbp!]



\begin{tikzpicture}
        \node[state, fill=gray!60] (G) at (0,0) {$G$};
        \node[state, fill=gray!60] (A) [below = 0.60 cm of G] {$A$};
        \node[state, fill=gray!60] (L) [right  = 0.60 cm of G] {$L$};
        \node[state, fill=gray!60] (D) [right  = 0.60 cm of L] {$D$};
        \node[state, fill=gray!60] (E) [right  = 0.60 cm of A] {$E$};
        \node[state, fill=gray!60] (I) [right  = 0.60 cm of E] {$I$};
        \node[state, fill=gray!60] (S) [right  = 0.60 cm of I] {$S$};

        \path (G) edge [thick](L);
        \path (G) edge [thick, bend left] (D) [right];
        \path (G) edge [thick](E);
        \path (G) edge [thick](I);
        \path (A) edge [thick](L);
        \path (A) edge [thick](D);
        \path (A) edge [thick](E);
        \path (A) edge [thick, bend right](I)[left];
        \path (I) edge [thick](S);
        \path (E) edge [thick](I);
        \path (L) edge [thick](D);

\end{tikzpicture}
    \caption{Causal graph for variables $\X$ of SCM \loan.}  \label{fig:loan_graph}
\end{figure}
  
\begin{figure}[htbp!]
     \centering
    \includegraphics[width=0.5\textwidth]{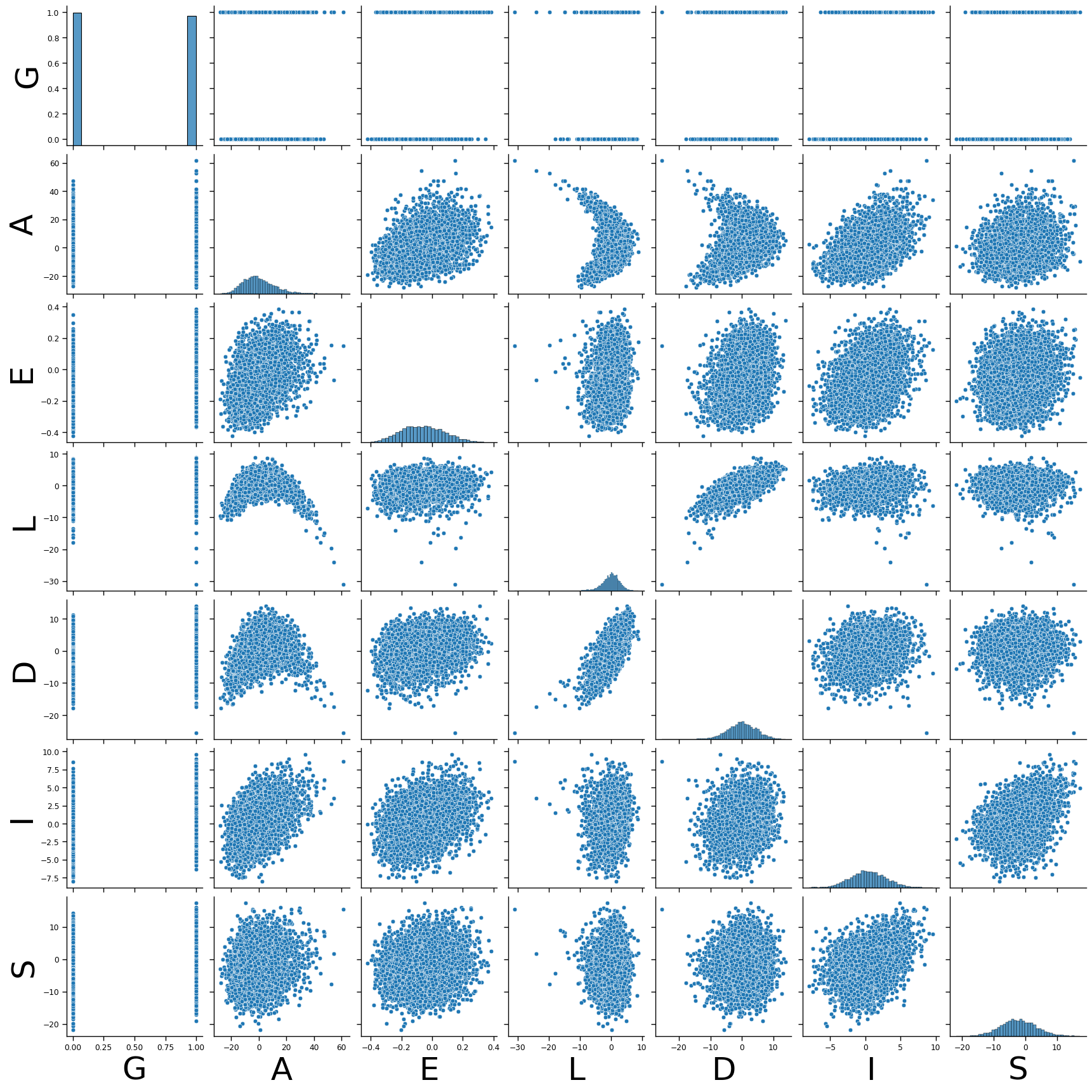}
    \caption{Histograms and scatter plots of pairwise feature relations for the \loan\ dataset. } \label{fig:loan_plot}
\end{figure}

 \newpage
\paragraph{Adult.}
We introduce the new semi-synthetic  \adult\ dataset which aims to reflect the relationship between the variables that affect the annual income of a person inspired by the real-world Adult dataset \cite{AdultData}. Figure~\ref{fig:adult_graph} shows the corresponding causal graph with $d=|\X| = 11$ nodes, diameter $\diameter =2$ and  longest path $\longpath=6$. We base the causal graph on \cite{chiappa2019pathscounterffair}. Our dataset consists of the following 11 endogenous variables:
\begin{itemize}
    \item \textit{Race} $R$ is an independent value.
    \item \textit{Age} $A$ is an independent value.
    \item \textit{Native country} ($N$) is an independent value.
    \item \textit{Sex} $S$ is an independent value.
    \item \textit{Education} level ($E$) depends on the sex, the native county and the race.
    \item \textit{Hours (worked) per week} $H$ depends additionally on the race, the native country and the sex. We consider this continuous variable to be between 0 and 80 hours.
    \item \textit{Work status} {$W$}, i.e., being unemployed or self-employed, is a discrete variable that depends on the number of hours worked per day, age, native country, and level of education.
    \item \textit{Martial status} $M$, i.e., married never married or separated, depends on age, race, work status, hours per week, native country and sex. Individuals under a certain age are never married.
    \item \textit{Occupation} {sector} $O$ is a discrete variable with values technology , social sciences and medicine. It depends on the race, age, education, martial status and sex.
    \item \textit{Relationship status} $L$ is a discrete variable, with values wife, own-child, husband, not-in-family, unmarried. It depends on martial status, education, age, native country, and sex.
    \item \textit{Income} $I$ depends on race, age, education, occupation, work status, martial status, hours per week, relationship status, native country, and sex.
\end{itemize}

These aspects are reflected in the following structural equations : 

\begin{figure}[htbp!]
     \centering
     \begin{subfigure}[b]{0.3\textwidth}
\begin{align*}
    \begin{split}
        &f_R:  R =U_{R} \\
        &f_A:  A =U_{A} + 17\\
        &f_N:   N =U_{N} \\
        &f_S:  S =U_{S} \\
        &f_E:  E = \exp \left( 2 \mathbb{I}_{\{R=0\}})+ \mathbb{I}_{\{R=1\}})+ \sigma(A - 30) \right) \\
        & \quad+ (0.5 \mathbb{I}_{\{S=0\}} + \mathbb{I}_{\{S=1\}})  (2\mathbb{I}_{\{N=1\}}+5\mathbb{I}_{\{N=2\}}\\
        & \quad +\mathbb{I}_{\{N=3\}}) + U_{E} \\
        &f_H:  H =((40 \mathbb{I}_{\{N=0\}} + 36  \mathbb{I}_{\{N=1\}}+ 50  \mathbb{I}_{\{N=2\}} \\
        & \quad + 30  \mathbb{I}_{\{N=3\}}) * \left(0.5\mathbb{I}_{\{R=0\}} + \mathbb{I}_{\{R=1\}} + 1.3 \mathbb{I}_{\{R=2\}}\right)  \\
          & \quad + 2 \exp \left( -(A - 30)^2 \right) + 5\mid \operatorname{tanh}(E-2)\mid  \\
        & \quad + 2 \mathbb{I}_{\{S=0\}} + U_{H}) \mathbb{I}_{\{A<70\}} \\
        & f_W:   W =w_2 \mathbb{I}_{\{w_2>=0\}}, \\
        & \quad w_2 = w_1 \mathbb{I}_{\{w_1<=3\}}+ 3\mathbb{I}_{\{w_1>3\}}\\
        & \quad w_1 =  \mathbb{I}_{\{5 \mid \operatorname{tanh}(E- 2)\mid + \sigma(H-30 + U_W) > 0.3\}} \\
        & \quad + \mathbb{I}_{\{\sigma(H-30 + U_W > 0.3\}} \mathbb{I}_{\{(A + 1.5 U_W) > 50\}} \\
        & \quad -\mathbb{I}_{\{N=0\}}+\mathbb{I}_{\{N=2\}} +3\mathbb{I}_{\{N=3\}}) \\
        & f_M : M = \operatorname{mode}(r_2, a_2, W, h_2, H, g_3) \\
        & \quad r_2 = 2*\mathbb{I}_{\{r_1=1\}} + \mathbb{I}_{\{r_1=2\}} \\
        & \quad r_1 = \operatorname{int}(R + 0.2U_M)\mathbb{I}_{\{R\in[0, 2]\}} + 2\mathbb{I}_{\{R>2\}}\\
        & \quad g_3= \mathbb{I}_{\{g_2=0\}} + 2*\mathbb{I}_{\{g_2=1\}}\\
        & \quad g_2 = 0*\mathbb{I}_{\{g_1<0\}} + \mathbb{I}_{\{g_1>1\}} + g_1\mathbb{I}_{\{g_1\in[0, 1]\}}\\
        & \quad g_1 = \operatorname{int}(G + 0.5U_M)\\
    \end{split}
\end{align*}
     \end{subfigure}
     \hfill
     \begin{subfigure}[b]{0.3\textwidth}
\begin{align*}
    \begin{split}
         & \quad a_2 = 2\mathbb{I}_{\{a_1 \in (20, 40]\}} + \mathbb{I}_{\{a_1 \in (40, 50]\}} + 2\mathbb{I}_{\{a_1 >= 50\}}\\
        & \quad a_1 = A + 2U_M\\
        & \quad h_2 = h_1\mathbb{I}_{\{h_1 <= 2\}} + 2\mathbb{I}_{\{h_1 > 2\}}\\
        & \quad h_1 = 3\operatorname{int}(\sigma(H-30))\\
        & f_L :  L = 0\Ibb{(M = 1)\wedge (c<-1)} + 1\Ibb{(M = 1)\wedge (c \geq -1)}\\
        & \quad+ 2\Ibb{(M \neq 1)\wedge (c\geq-1)} + 1\Ibb{(M \neq 1)\wedge (c<-1)} \\
        & \quad c = c_n + c_e + 2 \Ibb{A < 20} - 2 \Ibb{S = 0} \\
        & \quad c_n = U_O\Ibb{N=0} -U_O \Ibb{N=1} + 2U_O\Ibb{N=2} +  2 \Ibb{N=3} \\
        & \quad c_e = \sigma (E - 30) \\
        & f_O :  O = 0\Ibb{k<1} + 1\Ibb{1\leq k \leq 4} + 2\Ibb{4<k} \\
        & \quad k = R + k_a + k_e + W + 3M + 4S\\
        & \quad k_a = 2 e^{- (A + U_O - 20 )^2}\\
        & \quad k_e = - \sigma(E* U_O - 30 )\\
        & f_I  :  I = U_I \\
        & \quad + 10,000 \Ibb{R > 1.5}+ 20,000\Ibb{R < 1.5} \\
        & \quad + 3,000\Ibb{21 \leq A < 30}+ 8,000\Ibb{30 \leq A} \\
        & \quad + 5,000\Ibb{E < 2}+ 10,000\Ibb{2 \leq E < 10} + 30,000\Ibb{10 \leq E} \\
        & \quad + 5,000\Ibb{O = 1}+ 15,000\Ibb{O = 2} \\
        & \quad + 5,000\Ibb{W = 0}+ 7,000\Ibb{W = 1} \\
        & \quad + 1,000\Ibb{M = 0}+ 4,000\Ibb{M = 1} - 2,000\Ibb{M = 2} \\
        & \quad + 15,000\Ibb{H >45} + 10,000\Ibb{N \geq 2} \\
        & \quad + 4,000\Ibb{S = 1} + 3,000\Ibb{R \leq 1} \\
    \end{split}
\end{align*}
     \end{subfigure}
\end{figure}

\begin{figure*}[!htbp]
    \centering
    \begin{tikzpicture}
        \node[state, fill=gray!60] (R) at (0,0) {$R$};
        \node[state, fill=gray!60] (A) [below = 1.5 cm of R] {$A$};
        \node[state, fill=gray!60] (N) [below = 1.5 cm of A] {$N$};
        \node[state, fill=gray!60] (S) [below = 1.5 cm of N] {$S$};
        \node[state, fill=gray!60] (H) [right = 3 cm of A] {$H$};
        \node[state, fill=gray!60] (E) [right = 3 cm of N] {$E$};
        \node[state, fill=gray!60] (W) [right  = 6 cm of R, yshift=-1cm] {$W$};
        \node[state, fill=gray!60] (M) [right  = 6 cm of S, yshift=0.5cm] {$M$};
        \node[state, fill=gray!60] (L) [right = 6 cm of H] {$L$};
        \node[state, fill=gray!60] (O) [right = 6 cm of E, yshift=0.3cm] {$O$};
        \node[state, fill=gray!60] (I) [right = 3 cm of L] {$I$};
        \path (A) edge [thick, bend left](I); 
        \path (A) edge [thick](E); 
        \path (A) edge [thick](H); 
        \path (A) edge [thick](W); 
        \path (A) edge [thick, bend right](M); 
        \path (A) edge [thick](O); 
        \path (A) edge [thick, bend right](L); 
        \path (R) edge [thick, bend left](I); 
        \path (R) edge [thick](E); 
        \path (R) edge [thick](H); 
        \path (R) edge [thick](M); 
        \path (N) edge [thick](E); 
        \path (N) edge [thick](H); 
        \path (N) edge [thick](M); 
        \path (N) edge [thick](L); 
        \path (N) edge [thick, bend right](I); 
        \path (N) edge [thick, bend left](W); 
        \path (S) edge [thick](E); 
        \path (S) edge [thick](H); 
        \path (S) edge [thick, bend right](I); 
        \path (S) edge [thick](L); 
        \path (S) edge [thick](M); 
        \path (S) edge [thick](O); 
        \path (E) edge [thick](I); 
        \path (E) edge [thick](O); 
        \path (E) edge [thick](L); 
        \path (E) edge [thick](W); 
        \path (E) edge [thick](H); 
        \path (H) edge [thick](W); 
        \path (H) edge [thick](M); 
        \path (H) edge [thick, bend right](I); 
        \path (W) edge [thick](O); 
        \path (W) edge [thick](I); 
        \path (W) edge [thick](M); 
        \path (M) edge [thick, bend right](O); 
        \path (M) edge [thick, bend right](I); 
        \path (M) edge [thick](L); 
        \path (O) edge [thick](I); 
        \path (L) edge [thick](I); 

\end{tikzpicture}
    \captionof{figure}{Causal graph for variables $\X$ of SCM \adult.} \label{fig:adult_graph}
\end{figure*}

\begin{figure*}[!htbp]
    \centering
     \includegraphics[width=0.9\textwidth]{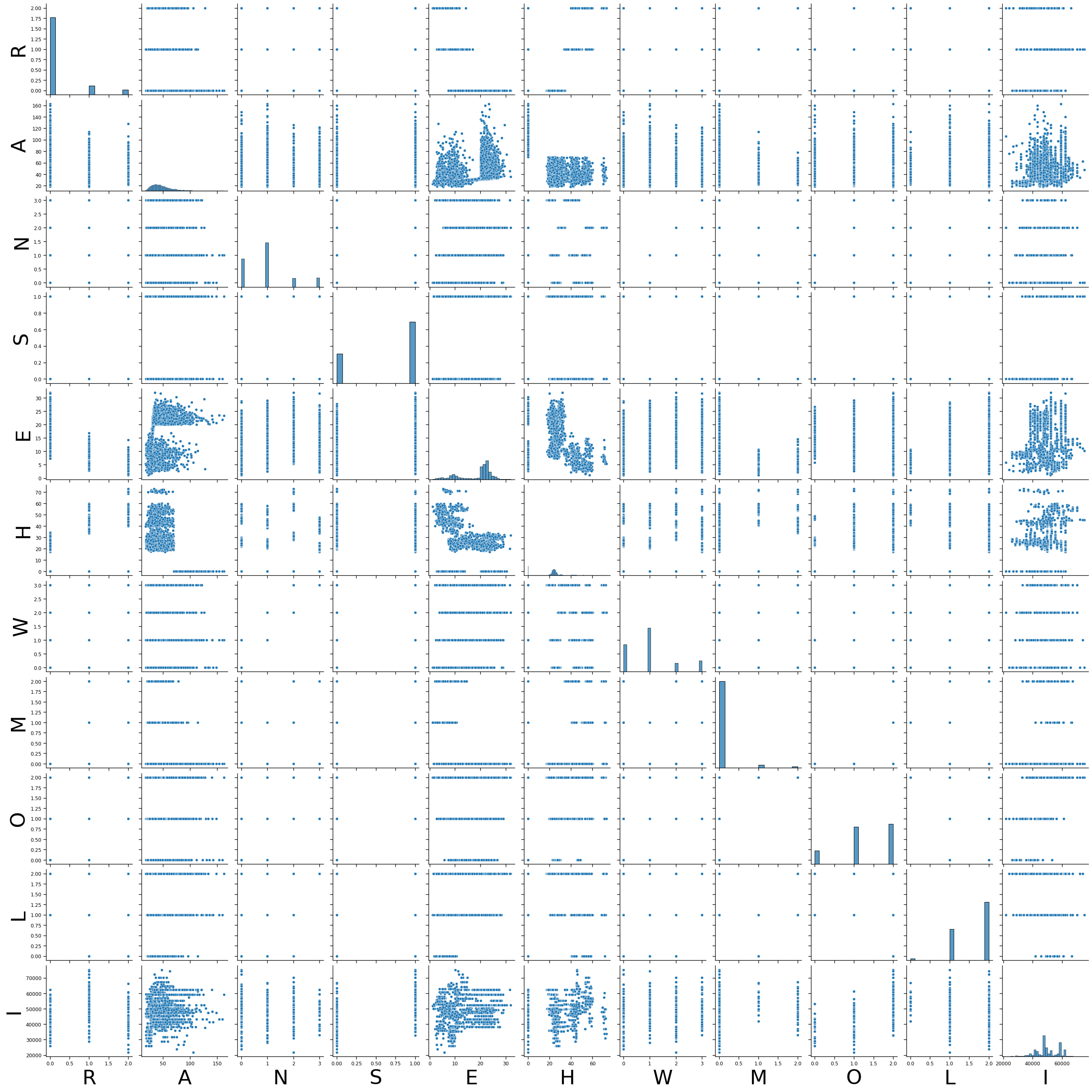}
    \caption{ Histograms and scatter plots of pairwise feature relations for the \adult\  dataset.} \label{fig:adult_plot}
  \end{figure*}
  
 \clearpage
\newpage
\subsection{Training and cross-validation}\label{apx:training_implementation}

This section details the hyperparameter cross-validation of \name, \mcvae\ \cite{karimi2020algorithmic}  and \carefl\ \cite{khemakhem2021causal}  for the experiments in Tables~\ref{tab:results} and Table~\ref{tab:interventional_apx}. Across experiments and models we generate synthetic datasets consisting of 5000 training samples, 2500 test samples and 2500 validation samples and we use a batch size of 1000. Al the models have been cross-validated using a similar computational budget, i.e. the number of combinations of hyperparameters cross-validated is the same for all the models (see Table \ref{tab:num_combinations}). Additionally, we run each configuration over 10 different random initializations. Regarding the number of iterations, we assumed a large enough number of epochs for training loss to converge, and implemented early stopping to prevent overfitting the training data. The best configuration has been chosen in terms of the best observational MMD.

\begin{table}[h]
 \centering
 \begin{tabular}{lcccc}
\toprule
 SCM &  \mcvae &   \carefl & \name  \\
 \midrule
\chain & 72  &  72 & 72 &  \\ 
\hline
\collider & 72  &  72 & 72 &  \\ 
\hline
\triangl & 72  &  72 & 72 &  \\ 
\hline
\mgraph & 72  &  72 & 72 &  \\ 
\hline
\loan & 40  &  24 & 24 &  \\ 
\bottomrule
\hfill
\end{tabular}
 \caption{Number of combinations of hyperparameters cross-validated for each of the models and datasets. Note that for each combination we run 10 different seeds.}
    \label{tab:num_combinations}
\end{table}

\paragraph{\name.} We optimize the ELBO \cite{kingma2013auto} and use the  IWAE \cite{burda2015importance}  with $K=100$ as the objective metric for the early stopping procedure. We use the Rectified Linear Unit (ReLU) as activation function. We trained with a learning rate $\eta = 0.005$ for a maximum of 500 epochs, or alternatively until the objective metric does not improve in 50 epochs. Also, we regularize the training of \name\ using a novel \emph{parents dropout:} randomly removing all incoming edges to the nodes with probability $p \in [0, 1)$. 
In our experiments, we observe that adding this regularization improves overall performance. 
We cross-validated the parents dropout rate with values $\{0.1, 0.2\}$, the number of hidden layers of the decoder with values $\{0,1, 2, 3, 4, 5\}$ (the specific values depend on the diameter of the graph) with 16 neurons each, and weather to use or not residual connections.  The best models are reported in Table~\ref{tab:VCAUSE_best}. 
We use a latent variable dimension of 4 and a Gaussian likelihood with a small variance $\sigma^2=\lambda_{KLD}/2$ with $\lambda_{KLD}=0.05$.

\begin{table*}[!htbp]
 \centering
    \begin{tabular}{lllrcc}
\toprule
 SCM & & Encoder Arch. & Decoder Arch. & Parents Dropout  & Residual \\
 \midrule
 \multirow{3}{*}{\rotatebox[origin=c]{90}{\chain}}  &  \lin&  $1 \times 16 \times 4$  &  $4 \times 16 \times 1$ & 0.1 & 1 \\ 
& \nonlin & $1 \times 16 \times 4$ &  $ 4 \times 16 \times 16 \times 1$ &  0.1 & 0  \\
& \nonadd  & $1 \times 16 \times 4$ &  $ 4  \times 64 \times 1$ &  0.1 & 1  \\
\hline
\hline
  \multirow{3}{*}{\rotatebox[origin=c]{90}{\collider}}&  \lin&  $1 \times 16 \times 4$  &  $4 \times 16  \times 16 \times 1$ & 0.2 & 1 \\ 
& \nonlin & $1 \times 16 \times 4$ &  $ 4 \times 16  \times 1$ &  0.1 & 0  \\
& \nonadd  & $1 \times 16 \times 4$ &  $ 4  \times 16  \times 16 \times 1$ &  0.1 & 1  \\
\hline
\hline
 \multirow{3}{*}{\rotatebox[origin=c]{90}{\triangl}}  &  \lin & $1 \times 16 \times 4$  &  $4 \times 16  \times 16 \times 1$ & 0.1 & 0 \\ 
& \nonlin & $1 \times 16 \times 4$ &  $ 4 \times 16  \times 16 \times 1$ &  0.2 & 0  \\
& \nonadd  & $1 \times 16 \times 4$ &  $ 4  \times 16  \times 16 \times 1$ &  0.1 & 0  \\ 
\hline
\hline
  \multirow{3}{*}{\rotatebox[origin=c]{90}{\scriptsize{\mgraph}}}  &  \lin & $1 \times 16 \times 4$  &  $4 \times 16   \times 1$ & 0.1 & 0 \\ 
& \nonlin & $1 \times 16 \times 4$ &  $ 4 \times 16  \times 16  \times 1$ &  0.1 & 0  \\
& \nonadd  & $1 \times 16 \times 4$ &  $ 4  \times 16  \times 16 \times 1$ &  0.1 & 0  \\ 
\hline
\hline
\loan  & - & $1 \times 16  \times 16  \times 4$  &  $4 \times 16  \times 16  \times 16 \times 1$ & 0.1
2 & 0 \\ 
\hline
\hline
\adult  & - & $1 \times 16  \times 16 \times 4$  &  $4  \times 8   \times 8  \times 8  \times 8 \times 1$ & 0.2 & 1 \\ 
\bottomrule
\hfill
\end{tabular}
 \caption{Hyperparemeter selection for our \name\ training for the SCMs on the synthetic datasets \triangl, \collider, \chain\ and \mgraph\ and on the semi-synthetic dataset \loan. Note, the encoder architecture refers to the layers in function $\fm$, while the decoder architecture refers to the different GNN layers. }
    \label{tab:VCAUSE_best}
\end{table*}

\begin{table*}[!htbp]
 \centering
    \begin{tabular}{llclrc}
\toprule
 SCM & & $\lambda_{\text{KLD}}$  & Encoder Arch. & Decoder Arch. & Dropout rate \\
 \midrule
 \multirow{3}{*}{\rotatebox[origin=c]{90}{\chain}}  &  \lin & 0.05 & $1 \times 16 \times 4$  &  $4 \times 128 \times 1$ & 0.0\\ 
& \nonlin & 0.05 &  $1 \times 16 \times 1$ &  $1 \times 128  \times 1$ &  0.0 \\
& \nonadd & 0.05 &  $1 \times 16 \times 4 $ & $ 4 \times 128 \times 1$&   0.0\\
\hline
\hline
  \multirow{3}{*}{\rotatebox[origin=c]{90}{\collider}}  &  \lin & 0.05 & $1 \times 16 \times 4 $  &  $ 4 \times 1$ & 0.0 \\ 
& \nonlin & 0.05 & $1 \times 16 \times 4 $ & $ 4 \times 1$&   0.0 \\
& \nonadd & 0.05 &  $1 \times 16 \times 4  $ & $ 4 \times 1$&   0.0\\
\hline
\hline
 \multirow{3}{*}{\rotatebox[origin=c]{90}{\triangl}}  &  \lin & 0.05 & $1 \times 16 \times 4 $  &  $ 4 \times 128 \times 1$ & 0.2 \\ 
& \nonlin & 0.05 & $1 \times 16 \times 1 $ & $ 1 \times 128 \times 1 $&   0.1\\
& \nonadd & 0.05 &  $1 \times 16 \times 4 $ & $ 4 \times 128 \times 1$&   0.2\\
\hline
\hline
  \multirow{3}{*}{\rotatebox[origin=c]{90}{\scriptsize{\mgraph}}}  &  \lin & 0.05  & $1 \times 16 \times 4 $  &  $ 4 \times 1$ & 0.2 \\ 
& \nonlin & 0.05 & $1 \times 16 \times 4 $ & $ 4\times 1 $&   0.1\\
& \nonadd & 0.05 &  $1 \times 16 \times 4 $ & $ 4 \times 32 \times 32 \times 1$&   0.2\\
\hline
\hline
\loan  & - & 0.05  & $1  \times 16 \times 16 \times 1 $  &  $1 \times 64 \times 64  \times 1$ & 0.0 \\ 
\hline
\hline
\adult  & - & 0.05  & $1  \times 16 \times 16 \times 2 $  &  $2 \times 16  \times 1$ & 0.0 \\ 
\bottomrule
\hfill
\end{tabular}
  \caption{Hyperparemeter selection for \mcvae\ \cite{karimi2020algorithmic} training for the SCMs on the synthetic datasets with three nodes (i.e., \triangl, \collider\ and \chain), \mgraph\ and for the semi-synthetic dataset \loan.}
    \label{tab:hyper_mcvae_loan}
\end{table*}

\paragraph{\mcvae.} \citet{karimi2020algorithmic} propose to train a conditional variational autoencoder (CVAE) for each endogenous variable that is not a root node in the causal graph, we refer to as \mcvae. Different from \cite{karimi2020algorithmic}, our implementation also models non-root nodes as CVAEs, since our goal is to model the joint distribution, while \cite{karimi2020algorithmic} target interventional and counterfactual distributions for algorithmic recourse only. Additionally, we perform the necessary modifications for training on normalized data. 

We cross-validated the number and size of hidden layers of the decoder, the dropout rate, and the dimension of the latent space. The best models (according to the observational MMD) are reported in Table \ref{tab:hyper_mcvae_loan}. As with \name,  we assume $\lambda_{KLD} = 0.05$ for all SCMs and CVAEs.

\paragraph{\carefl.}  \citet{khemakhem2021causal} propose \carefl, an autoregressive causal flows model for causal discovery, which also allows to answer interventional and counterfactual queries. The authors rely on real-valued non-volume preserving (real NVP) transformations,
since they mainly focus on the multivariate bi-variate case. As this flow architecture is not suited for general graphs, we use their framework with Neural Spline Autoregressive Flows \cite{durkan2019neural}.  We have cross validated the number of flows  $\{2, 3, 4, 5, 6\}$ and the number of hidden units of the neural networks $\{5, 10, 16,  32, 64, 96\}$.  The final configuration--i.e., the configuration with the lowest observational MMD--is displayed in Table~\ref{tab:CAREFL_best}.

\begin{table}
 \centering
\small{
 \begin{tabular}{lccc}
\toprule
 SCM & &  Flows & Hidden Units  \\
 \midrule
 \multirow{3}{*}{\rotatebox[origin=c]{90}{\chain}}  &  \lin&  4 & 5 \\ 
& \nonlin& 3 &  64\\
& \nonadd& 5 &  5 \\
\hline
\hline
 \multirow{3}{*}{\rotatebox[origin=c]{90}{\collider}}  &  \lin&  2 & 96 \\ 
& \nonlin& 2 & 64 \\
& \nonadd& 3 &  16 \\
\hline
\hline
 \multirow{3}{*}{\rotatebox[origin=c]{90}{\triangl}}  &  \lin&  4 & 64 \\ 
& \nonlin& 4 & 5\\
& \nonadd& 3 &  96 \\
\hline
\hline
\multirow{3}{*}{\rotatebox[origin=c]{90}{\scriptsize{\mgraph}}}  &  \lin &  4 & 64\\ 
& \nonlin& 4 & 96\\
& \nonadd& 5 &  96 \\
\hline
\hline
\loan  &  - &  4 & 64 \\ 
\hline
\hline
\adult  &  - &  2 & 96 \\ 
\bottomrule
\hfill
\end{tabular}
}
  \caption{Hyperparemeter selection for \carefl\ \cite{khemakhem2021causal} training for different SCMs.}
    \label{tab:CAREFL_best}
\end{table}


\subsection{Performance metrics} \label{apx:training_metrics}

In the following we describe the metrics used to evaluate the performance of \name\ in Section~\ref{sec:Experiments}.
In all experiments we use (semi-)synthetic datasets with access to samples from the ground truth distribution $\{\x_i\}^{n}_{i=0} \sim P$ as well as from the estimated distribution $\{\xh_i\}^{n}_{i=0} \sim Q$.

For the interventional and counterfactual distribution, we perform a set of interventions \mbox{$ \Ical=\{do(\X_{\Ical_j} =  \alpha_j )\}_{j}$}, where $\Ical_j \in [d]$ and $\alpha_j \in \{-1.0, -0.5, 0.0, 0.5, 1.0\} \times \sigma_{\Ical_j}$ with $\sigma_{\Ical_j}$ as {the empirical standard deviation of the intervened variable $\X_{\Ical_j}$ prior to intervention (i.e., in the observational distribution)}. Note that we only intervene on one variable at a time.  For each intervention in $\Ical$, we are interested in the estimated distribution of variables causally affected by the intervention $\{ \X_i | i \in des(\Ical_j)\}$, i.e., the set of descendants of the variable intervened. Note that $des(\Ical_j)$ refers to the set of indexes of the descendants. It follows that we do not intervene on leaf nodes.

\paragraph{Mean Maximum Discrepancy (MMD).} 
The Mean Maximum Discrepancy (MMD) \cite{gretton2012kernel} is a kernel-based distance-measure between two distributions $P$ and $Q$ on the basis of samples from both distributions.  
The smaller the MMD, the more likely it is that the sets of samples are drawn from the same distributions, i.e. the better distributions match. 

Without access to underlying distribution, we can compute an unbiased empirical squared MMD estimate %
using a kernel function $k$ as\cite{gretton2012kernel}:
\begin{align}
     &{\widehat{\mathrm{MMD}}}^2(\X, \hat{\X}) =  \frac{1}{n(n-1)} \left( \sum_{i=1}^{n} \sum_{j=1}^{n} k\left(\x_{i},\x_{j}\right) + \sum_{i=1}^{n} \sum_{j=1}^{n} k\left(\xh_{i}, \xh_{j}\right) -2  \sum_{i=1}^{n} \sum_{j=1}^{n} k\left(\x_{i}, \xh_{j}\right) \right).
\end{align}
In our implementation we use as kernel a mixture of RBF (Gaussian) kernels with different bandwidths and sample size $n=1000$.

\paragraph{Estimation squared error for the mean (\mde).}
For the interventional distribution, we compute the estimation squared error for the mean (\mde) as the average (across interventions) of the squared difference between the {empirical}  means of the true and estimated interventional distributions (for the descendants of the intervened variables):

\begin{align*}
    \operatorname{\mde} & =  \frac{1}{|\Ical|} \sum_{\Ical_j \in \Ical}  \frac{1}{|des(\Ical_j)|} \sum_{i \in des(\Ical_j) }  \left( E[X_i^{\Ical_j}] - E[\Xh_i^{\Ical_j}] \right)^2
\end{align*}

\paragraph{Estimation squared error for the standard deviation (\sdde).} 
For the interventional distribution, we compute the estimation squared error for the standard deviation (\sdde) as the average (across interventions) of the squared difference between the {empirical}  standard deviation of the true $\tilde{\sigma}(X_i^{\Ical_j}) $ and estimated  $\tilde{\sigma}(\Xh_i^{\Ical_j})$ interventional distributions (for the descendants of the intervened variables):

\begin{align*}
  \operatorname{\sdde} & =  \frac{1}{|\Ical|} \sum_{\Ical_j \in \Ical}  \frac{1}{|des(\Ical_j)|} \sum_{i \in des(\Ical_j) }  \left(  \tilde{\sigma}(X_i^{\Ical_j}) - \tilde{\sigma}(\Xh_i^{\Ical_j}) \right)^2  
\end{align*}

\paragraph{Mean squared error (\mre).} 
For the counterfactual distribution, we compute the mean squared error (\mre) as the average (across interventions) of the pairwise squared difference between true and estimated counterfactual values for  the descendants of the intervened variable. 
More in detail, let us define the random variable $T^{\Ical_j}$ as the \textit{Frobenius norm} of the difference between true $\xCF_{ des(\Ical_j)}$ and estimated $\xCFh_{ des(\Ical_j)}$ counterfactual values for the descendants of the intervened variable, i.e., 
\begin{align}
T^{\Ical_j} = ||\xCF_{ des(\Ical_j)}  - \xCFh_{des(\Ical_j)} ||^2_2,
\end{align}
{Thus, we can compute the counterfactual  \mre\ as:}
\begin{align}
\operatorname{\mre} = \frac{1}{|\Ical|} \sum_{\Ical_j \in \Ical}  \frac{1}{|des(\Ical_j)| } E \left[ T^{\Ical_j} \right]
\end{align}

\vspace{1cm}
\paragraph{Standard deviation of the squared error (\sdre).} 
Similarly, we can compute the average (across interventions) of the standard deviation of the counterfactual squared error as: 

\begin{align}
    \operatorname{\sdre} =  \frac{1}{|\Ical|} \sum_{\Ical_j \in \Ical}  \frac{1}{|des(\Ical_j)|} \sum_{i \in des(\Ical) }   \tilde{\sigma}\left( T^{\Ical_j} \right),
\end{align}
where $\tilde{\sigma}\left( T^{\Ical_j} \right)$ denotes the empirical standard deviation of $T^{\Ical_j} $.

\subsection{Additional results}\label{apx:additional_results}
\begin{table*}[]
    \centering
\small{
\setlength\tabcolsep{3pt}
    \centering
    \begin{tabular}{cc  l r rrr   rr r}
\toprule
 &  &    & Obs.  &  \multicolumn{3}{c}{Interventional}  &  \multicolumn{2}{c}{Counterfactuals}  & \\
  \cmidrule(r){4-4}   \cmidrule(r){5-7}  \cmidrule(r){8-9}  
      \multicolumn{2}{c}{SCM} & Model   &MMD &   MMD  &  \mde\ &    \sdde\  & \mre\ &    \sdre\  & Num. params \\
\cmidrule(r){1-10} 
 \multirow{9}{*}{\rotatebox[origin=c]{90}{ {\collider}}}  &  \multirow{3}{*}{\rotatebox[origin=c]{90}{\lin}} & MultiCVAE &   30.37$\pm$8.16 &   44.70$\pm$12.25 &   13.29$\pm$4.78 &   46.56$\pm$2.40 &   87.41$\pm$3.64 &  65.15$\pm$2.83 &    553  \\
       &              & CAREFL &    9.27$\pm$1.49 &     4.86$\pm$0.45 &    0.35$\pm$0.08 &   81.89$\pm$1.78 &    8.11$\pm$0.58 &   7.83$\pm$0.55 &   6420  \\
       &              & VACA &    1.50$\pm$0.67 &     1.57$\pm$0.41 &    0.75$\pm$0.31 &   41.99$\pm$0.30 &    9.86$\pm$0.74 &   7.06$\pm$0.38 &   5600  \\
      \cline{2-10}
         &  \multirow{3}{*}{\rotatebox[origin=c]{90}{\nonlin}} & MultiCVAE &    28.03$\pm$9.12 &   41.60$\pm$12.62 &   10.49$\pm$4.12 &   46.48$\pm$2.43 &   82.32$\pm$2.61 &  62.05$\pm$1.87 &    553  \\
       &              & CAREFL &   10.38$\pm$2.00 &     4.69$\pm$0.38 &    0.19$\pm$0.07 &   80.68$\pm$2.08 &    6.93$\pm$0.40 &   7.15$\pm$0.64 &   4308  \\
       &              & VACA &    0.95$\pm$0.27 &     0.97$\pm$0.23 &    0.26$\pm$0.12 &   42.20$\pm$0.24 &    5.01$\pm$0.73 &   4.08$\pm$0.54 &   1805  \\
      \cline{2-10}
         & \multirow{3}{*}{\rotatebox[origin=c]{90}{\nonadd}} & MultiCVAE &  29.72$\pm$8.90 &  117.67$\pm$46.20 &  26.23$\pm$11.09 &   39.75$\pm$1.34 &  73.93$\pm$10.44 &  51.44$\pm$6.06 &    553  \\
       &              & CAREFL &    9.79$\pm$1.29 &     9.13$\pm$1.45 &    1.58$\pm$0.94 &  102.13$\pm$3.16 &    9.92$\pm$0.51 &  32.66$\pm$0.86 &   1710  \\
       &              & VACA &    1.71$\pm$0.48 &    29.87$\pm$1.94 &   10.71$\pm$1.05 &   50.25$\pm$0.87 &   31.99$\pm$0.84 &  39.10$\pm$0.67 &   5600  \\
\hline\hline
 \multirow{9}{*}{\rotatebox[origin=c]{90}{ {\triangl}}} 
&  \multirow{3}{*}{\rotatebox[origin=c]{90}{\lin}} & MultiCVAE &   33.12$\pm$3.89 &  157.50$\pm$15.08 &   12.10$\pm$2.03 &   44.66$\pm$2.57 &   65.95$\pm$3.05 &  42.39$\pm$1.75 &   3243 \\
      &              & CAREFL &   13.50$\pm$1.83 &     8.88$\pm$0.54 &    0.62$\pm$0.25 &   80.11$\pm$2.44 &    7.05$\pm$1.09 &   8.91$\pm$1.43 &   8616 \\
      &              & VACA &    3.64$\pm$1.64 &    13.82$\pm$1.87 &    2.90$\pm$0.49 &   25.82$\pm$0.33 &   23.52$\pm$1.83 &  16.02$\pm$1.32 &   5334 \\
      \cline{2-10}
& \multirow{3}{*}{\rotatebox[origin=c]{90}{\nonlin}} & MultiCVAE &   46.65$\pm$9.05 &  218.23$\pm$30.38 &   17.50$\pm$4.85 &   27.55$\pm$1.98 &   52.34$\pm$4.50 &  32.45$\pm$2.22 &   1785  \\
       &              & CAREFL &   13.55$\pm$2.25 &    19.01$\pm$1.05 &     1.61$\pm$0.44 &  103.03$\pm$2.26 &       9.38$\pm$4.21 &   10.38$\pm$4.20 &    828  \\
       &              & VACA &    6.04$\pm$1.87 &     9.53$\pm$3.60 &       1.44$\pm$0.59 &     17.93$\pm$0.19 &   13.24$\pm$1.93 &     8.26$\pm$1.29 &   5334  \\
      \cline{2-10}
         & \multirow{3}{*}{\rotatebox[origin=c]{90}{\nonadd}} & MultiCVAE &   16.99$\pm$7.22 &  133.83$\pm$16.96 &    2.59$\pm$3.38 &  18.25$\pm$0.83 &   42.43$\pm$0.87 &   20.55$\pm$1.15 &   3243  \\
       &              & CAREFL & 13.58$\pm$1.69 &   74.51$\pm$14.18 &   1.31$\pm$0.84 &  148.03$\pm$6.96 &   11.62$\pm$1.50 &  53.71$\pm$0.98 &   9630  \\
       &              & VACA &   1.72$\pm$0.67 &     30.10$\pm$4.41 &  0.19$\pm$0.14 &  26.64$\pm$1.30 &   22.50$\pm$2.61 &   41.16$\pm$3.91 &   5334  \\
\hline\hline
 \multirow{9}{*}{\rotatebox[origin=c]{90}{ {\mgraph}}}  &  \multirow{3}{*}{\rotatebox[origin=c]{90}{\lin}} & MultiCVAE &  38.60$\pm$4.06 &   80.89$\pm$20.73 &   25.82$\pm$6.26 &   17.73$\pm$3.39 &   54.72$\pm$4.71 &  27.47$\pm$2.28 &    933  \\
       &              & CAREFL &   19.55$\pm$3.48 &    15.38$\pm$0.75 &    0.77$\pm$0.23 &   68.32$\pm$1.35 &    6.94$\pm$0.23 &   9.97$\pm$0.15 &  18200  \\
       &              & VACA &    1.76$\pm$0.34 &     4.60$\pm$0.81 &    1.86$\pm$0.23 &   11.79$\pm$0.11 &   12.83$\pm$0.69 &   9.10$\pm$0.51 &   3249  \\
      \cline{2-10}
         &  \multirow{3}{*}{\rotatebox[origin=c]{90}{\nonlin}} & MultiCVAE &    37.95$\pm$6.21 &   84.79$\pm$18.15 &   24.97$\pm$9.10 &   16.66$\pm$3.28 &   51.30$\pm$5.24 &  25.90$\pm$1.23 &    933  \\
       &              & CAREFL &   20.72$\pm$2.78 &    18.63$\pm$1.36 &    2.57$\pm$0.69 &   74.41$\pm$1.02 &    9.43$\pm$0.34 &  24.94$\pm$0.31 &  27160  \\
       &              & VACA &    1.95$\pm$0.40 &     6.78$\pm$0.95 &    2.27$\pm$0.41 &   12.55$\pm$0.43 &   16.26$\pm$0.95 &  17.65$\pm$1.12 &   8001  \\
      \cline{2-10}
         & \multirow{3}{*}{\rotatebox[origin=c]{90}{\nonadd}} & MultiCVAE &   19.97$\pm$5.03 &   80.25$\pm$14.12 &    0.92$\pm$0.60 &   32.94$\pm$1.77 &   38.28$\pm$1.00 &  32.61$\pm$1.11 &   7277  \\
       &              & CAREFL &   19.36$\pm$1.59 &    18.37$\pm$0.78 &    0.28$\pm$0.09 &   79.71$\pm$2.30 &   32.66$\pm$0.31 &  49.78$\pm$0.19 &  33950  \\
       &              & VACA &    3.06$\pm$0.84 &    17.80$\pm$1.78 &    0.09$\pm$0.05 &   21.22$\pm$0.56 &   41.06$\pm$0.31 &  48.39$\pm$0.35 &   8001  \\
\hline\hline
 \multirow{9}{*}{\rotatebox[origin=c]{90}{ {\chain}}}  &  \multirow{3}{*}{\rotatebox[origin=c]{90}{\lin}} & MultiCVAE &  29.38$\pm$2.96 &  131.24$\pm$12.04 &    6.75$\pm$1.66 &   38.96$\pm$2.17 &   59.15$\pm$1.61 &  41.37$\pm$1.07 &   3099  \\
       &              & CAREFL &   11.70$\pm$1.46 &    12.00$\pm$1.23 &    0.98$\pm$0.38 &   81.15$\pm$2.55 &    9.90$\pm$1.42 &  11.72$\pm$1.80 &    828  \\
       &              & VACA &    4.47$\pm$0.72 &     5.73$\pm$1.11 &    3.60$\pm$0.62 &   22.91$\pm$0.23 &   23.09$\pm$1.18 &  15.77$\pm$0.77 &   5648  \\
      \cline{2-10}
         &  \multirow{3}{*}{\rotatebox[origin=c]{90}{\nonlin}} & MultiCVAE &  30.33$\pm$4.53 &  159.00$\pm$11.81 &    8.68$\pm$4.51 &   43.49$\pm$2.91 &   60.34$\pm$1.87 &  40.88$\pm$1.08 &   1641  \\
       &              & CAREFL &   12.35$\pm$2.03 &    11.08$\pm$1.02 &    1.04$\pm$0.43 &   79.06$\pm$2.16 &   10.56$\pm$0.74 &  12.05$\pm$1.17 &   6462  \\
       &              & VACA &    4.66$\pm$1.06 &     6.17$\pm$1.24 &    3.69$\pm$0.46 &   22.41$\pm$0.45 &   23.17$\pm$2.04 &  15.69$\pm$1.25 &   4445  \\
      \cline{2-10}
         & \multirow{3}{*}{\rotatebox[origin=c]{90}{\nonadd}} & MultiCVAE &  29.65$\pm$2.89 &  132.74$\pm$11.58 &    7.10$\pm$1.75 &   40.59$\pm$2.23 &   59.21$\pm$1.66 &  41.26$\pm$1.08 &   3099  \\
       &              & CAREFL &   11.24$\pm$1.41 &    10.86$\pm$1.04 &    1.15$\pm$0.32 &   80.26$\pm$2.52 &   10.18$\pm$2.11 &  12.92$\pm$2.86 &   1035  \\
       &              & VACA &    4.67$\pm$0.91 &     6.09$\pm$1.16 &    3.98$\pm$0.72 &   23.14$\pm$0.20 &   23.58$\pm$1.33 &  15.96$\pm$0.81 &   5648  \\
\hline\hline
 \multirow{3}{*}{\rotatebox[origin=c]{90}{\loan}} &   & MultiCVAE &  90.38$\pm$11.31 &   213.65$\pm$5.38 &   12.24$\pm$1.33 &   65.78$\pm$1.13 &   40.98$\pm$0.35 &  15.12$\pm$0.16 &  33717 \\
       &         -     & CAREFL &   22.10$\pm$1.64 &    27.38$\pm$4.07 &    6.74$\pm$4.25 &   50.13$\pm$2.47 &   11.15$\pm$2.57 &   6.59$\pm$0.38 &  2880 \\
       &              & VACA &    2.22$\pm$0.25 &     6.87$\pm$0.66 &    4.35$\pm$0.35 &    3.83$\pm$0.08 &   10.30$\pm$0.40 &   6.41$\pm$0.11 &  30402 \\
\hline\hline
 \multirow{3}{*}{\rotatebox[origin=c]{90}{\adult}} &   & MultiCVAE &   140.15$\pm$6.37 &  155.52$\pm$5.93 &  12.18$\pm$2.36 &  63.52$\pm$4.05 &  39.96$\pm$0.36 &  16.37$\pm$0.65 &    6549 \\
      &        & CAREFL &   31.31$\pm$1.58 &   34.31$\pm$5.77 &  12.54$\pm$3.17 &  41.26$\pm$3.44 &   1.23$\pm$0.17 &   3.55$\pm$0.90 &  127420 \\
      &        & VACA &    4.51$\pm$0.45 &   12.68$\pm$1.95 &   1.65$\pm$0.23 &   3.37$\pm$0.09 &   5.33$\pm$0.27 &   5.67$\pm$0.20 &   63432 \\
\bottomrule

\end{tabular}
\caption{Performance of different methods at estimating the observational, interventional and counterfactual of different SCMs.
    {All metrics are multiplied by 100.}}\label{tab:interventional_apx}
}
\vspace{-10pt}
\end{table*}

 \begin{figure*}
    \centering
    \begin{subfigure}{.19\textwidth}
    \centering
    \includegraphics[width=.9\linewidth]{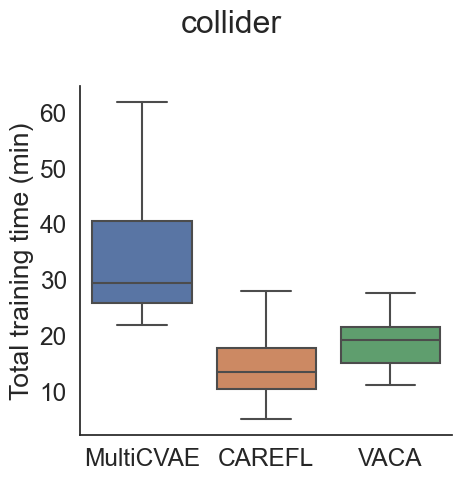}
    \end{subfigure}%
    \begin{subfigure}{.19\textwidth}
    \centering
    \includegraphics[width=.9\linewidth]{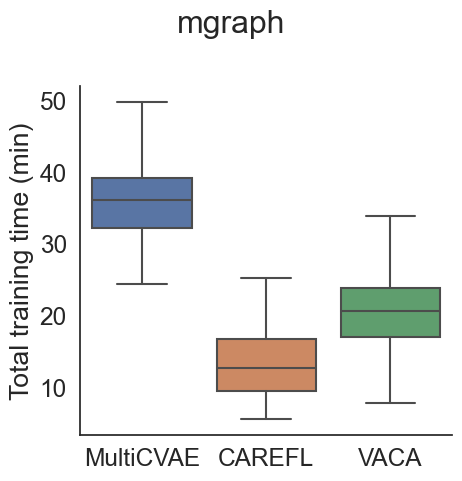}
    \end{subfigure}%
    \begin{subfigure}{.19\textwidth}
    \centering
    \includegraphics[width=.9\linewidth]{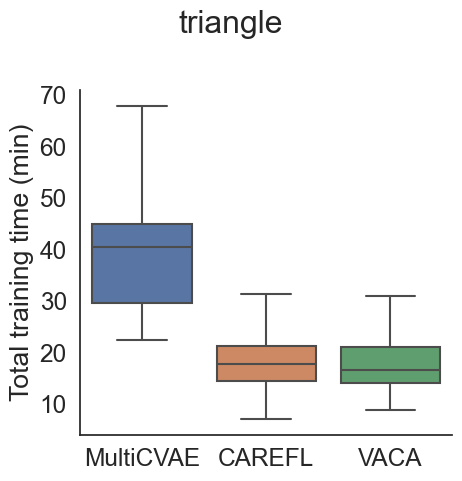}
    \end{subfigure}
    \begin{subfigure}{.19\textwidth}
    \centering
    \includegraphics[width=.9\linewidth]{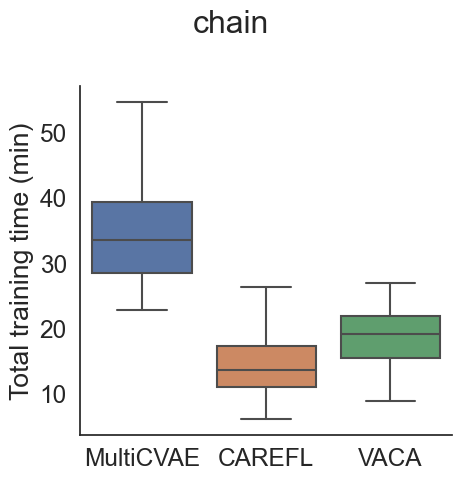}
    \end{subfigure}
    \begin{subfigure}{.19\textwidth}
    \centering
    \includegraphics[width=.9\linewidth]{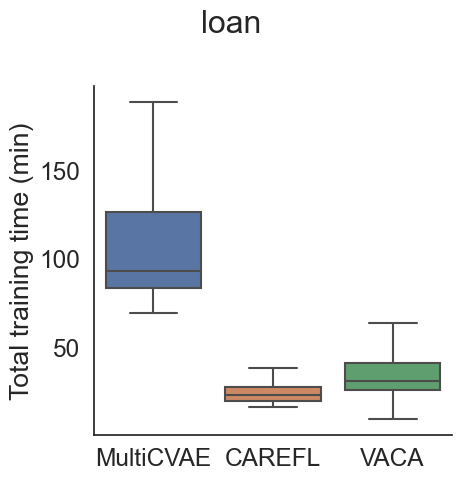}
    \end{subfigure}
    \caption{Training time (in minutes) for \mcvae, \carefl\ and \name\ for the different graphs.}
    \label{fig:time_analysis}
\end{figure*}

 \begin{figure*}
    \centering
    \begin{subfigure}{.19\textwidth}
    \centering
    \includegraphics[width=.9\linewidth]{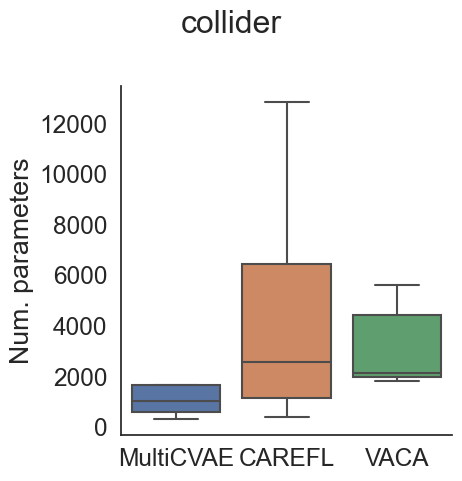}
    \end{subfigure}%
    \begin{subfigure}{.19\textwidth}
    \centering
    \includegraphics[width=.9\linewidth]{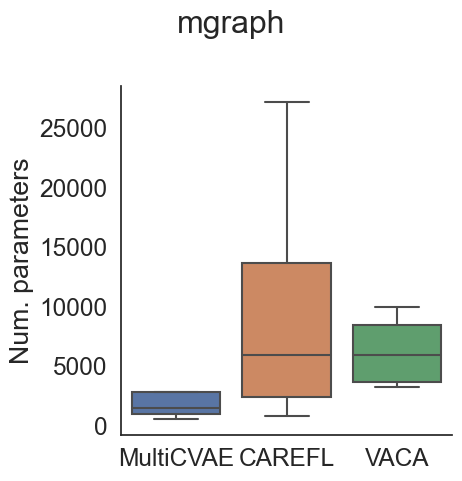}
    \end{subfigure}%
    \begin{subfigure}{.19\textwidth}
    \centering
    \includegraphics[width=.9\linewidth]{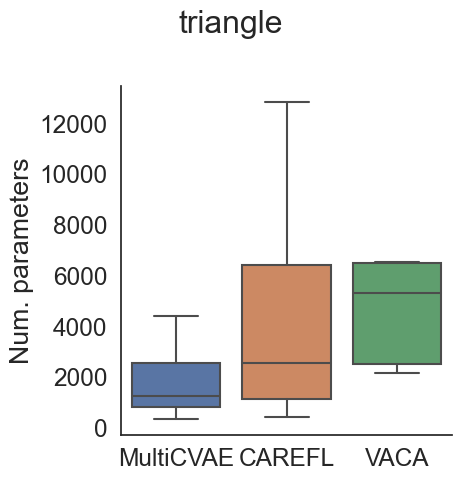}
    \end{subfigure}
    \begin{subfigure}{.19\textwidth}
    \centering
    \includegraphics[width=.9\linewidth]{images/num_params_per_model_triangle.png}
    \end{subfigure}
    \begin{subfigure}{.19\textwidth}
    \centering
    \includegraphics[width=.9\linewidth]{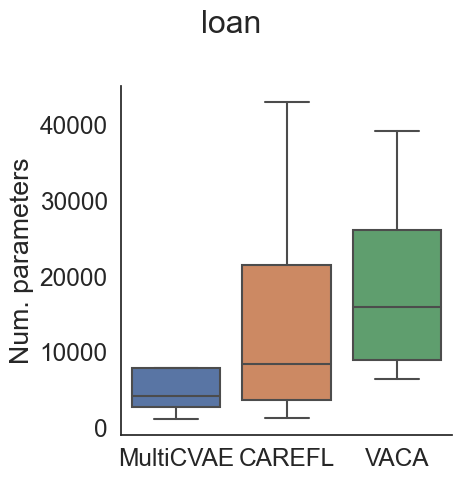}
    \end{subfigure}
    \caption{Number of parameters of \mcvae, \carefl\ and \name\ for the different graphs. }
    \label{fig:params_analysis}
\end{figure*}

In the following we present additional results that empirically show the potential of \name\ to model interventional and counterfactual queries. In particular, we report the results for the \triangl, \mgraph, and \chain\ graphs. We remark that the following results are consistent with the ones reported in the main manuscript for the \collider, \loan\ and \adult.

\paragraph{Results for interventional distributions.} Table \ref{tab:interventional_apx} (middle columns) reports the MMD, \mde, and \sdde\ for the interventional distribution. 
In accordance with the results shown in the main manuscript, we can observe that i) \name\ consistently outperforms other methods  in terms of MMD; ii) the three methods provide  comparable results in capturing the mean of the interventional distribution (\mde); and iii) CAREFL and \mcvae\ often fail to capture the standard deviation of the interventional distribution (\sdde),  while \name\ provides a more accurate estimate of the overall interventional distribution (as can be easily seen in the MMD).

\paragraph{Results for the counterfactuals.}  Table \ref{tab:interventional_apx} also reports the results for the counterfactual distribution, in terms of \mre\ and \sdre.
As reported in the main text, we observe that \carefl\ provides more accurate estimates  than \name\ and \mcvae\ in terms of \mre, which may be explained by the fact that CAREFL performs exact inference as opposed to the approximated inference of the other two approaches. 
However, \carefl\ presents high variance in its results (see \sdre).  
In contrast, \name\ leads to regularly lower values of \sdre, which suggests more consistent counterfactual estimations across factual samples and interventions.

\subsection{Complexity analysis}\label{apx:time_analysis}

In this section we compare the amount of time that it takes \name\ and the two baselines  to converge during training.
Figure~\ref{fig:time_analysis} shows the time (in minutes) it takes to train the models for the configurations of hyperparameters cross-validated (see Section \ref{apx:training_implementation}). Note we also show the results averaged over 10  different initializations. 
We can extract three main points from Figure~\ref{fig:time_analysis}. Firstly, we observe that most of the experiments take less than 60 minutes to converge, only some experiments with the \loan\ graph take longer. Secondly,  \mcvae\ takes longer to converge than \carefl\ and \name\ on average. This can be explained by the fact that nodes of the graph are trained sequentially. Thirdly, we observe that \carefl\ and \name\ take similar amount of time to converge on average.
We remark that experiments were run in a computer cluster, whose performance depends on its congestion, i.e. the number of people using the cluster and the amount of experiments. Thus, it is possible that part of the variance in the times is due to different congestion situations in the cluster.

Figure \ref{fig:params_analysis} shows a box-plot of the number of parameters of the configuration of hyperparameters cross-validated for the models and the datasets under study. Firstly, we observe that the configurations chosen for \mcvae\ contain less number of parameters than the  configurations of the other two methods. 
However, increasing the number of parameters leads to an increase in training time.  \mcvae\ already has the longest training time compared to the other methods, so the comparison is unfair in terms of time complexity.
Regarding \carefl\ and \name, we observe that the number of parameters cross-validated overlaps in all the cases.

\subsection{Computing infrastructure}
\label{apx:infrastructure}
All the experiments conducted in this work were executed in the same computer cluster based on the Linux OS. Each experiment was assigned to 2 CPUs (they could be assigned only to 1 CPU but the more CPUs, the faster the data-loading process is) and 8GB of memory. For information about the required software packages please refer to the official GitHub repository of \name\ \url{https://github.com/psanch21/VACA}.

 %

\section{Further details on the counterfactual fairness use-case}\label{apx:fairness}
In this section we provide further details on dataset, training, metrics and additional results for the use-case of counterfactual fairness in Section \ref{sec:fairness}.

\subsection{German Credit Dataset}
The German Credit dataset from the UCI repository \cite{GermanData} contains 20 attributes from 1000 loan applicants. We rely on the causal model   proposed by in~\cite{chiappa2019pathscounterffair} for the following subset of features as exogenous variables $\X$ (see Figure~\ref{fig:german_apx}): sensitive feature $S=\{ \textit{sex}\}$,  and non-sensitive features $C=~\{ \textit{age} \}$,  $R=\{\text{\textit{credit amount}, \textit{repayment history}\}}$ and $H = \{\textit{checking account}, \textit{savings}, \textit{housing}\}$. The causal graph in Figure~\ref{fig:german_apx} has a diameter $\diameter=1$ and longest path $\longpath=1$.
The goal of a classifier $h$ is to predict $Y=\{credit risk\}$ from $\X$. We load and pre-process the data using the $\operatorname{aif360}$ library such that the dataset contains binary outcome variable $Y$ (0-bad, 1-good) and a binary sensitive attribute $S$ (0-female, 1-male). Note that the dataset contains 700 labels $Y=1$ and 300 labels $Y=0$, i.e., it is imbalanced. It also contains 690 males $S=1$ and 310 females $Y=0$.
Note also that the causal model contains heterogeneous causal nodes ($R$ and $S$), as addressed in Section \ref{sec:heterogenous}. For example,  \cite{chiappa2019pathscounterffair} assume that the relationship between  $\textit{credit amount}$ and $\textit{repayment history}$ is unknown, or that it may be affected by hidden confounders. This leads to an undirected path between the random variables and they are grouped together in one multidimensional causal node $R$. This applies similarly to node $S$.

\begin{figure}[h]
    \centering



\begin{tikzpicture}
        \node[state, fill=gray!60] (S) at (0,0) {$S$};
        \node[state, fill=gray!60] (R) [below = 0.60 cm of S] {$R$};
        \node[state, fill=gray!60] (H) [right  = 0.60 cm of R] {$H$};
        \node[state, fill=gray!60] (C) [right  = 0.60 cm of S] {$C$};

        \path (S) edge [thick](R);
        \path (S) edge [thick](H);
        \path (C) edge [thick](R);
        \path (C) edge [thick](H);

\end{tikzpicture}
    \caption{Causal graph for variables $\X$ of the German Credit dataset~\cite{chiappa2019pathscounterffair}}
    \label{fig:german_apx}
\end{figure}
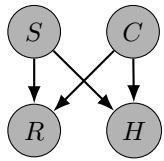

\subsection{Training}
In this section we provide further information on training \name\ on the German Credit dataset \cite{GermanData} and detail the different classifiers in Section~\ref{sec:fairness}. We use a 80\% training, 10\% validation, 10\% training data split.

\paragraph{\name.} Training for \name\ was performed on normalized data---performing normalization only on the continuous variables, i.e. r.v. $C$ and $R$ in Fig \ref{fig:german_apx}. We trained a heterogeneous \name\ as described in Section with a message passing function $\fm$ with one hidden layer of $16$ neurons, a decoder with one hidden layer of $16$ neurons and a latent variable with dimension $4$. We trained the model using the PIWAE \cite{rainforth2018tighter} approach with $\lambda_{KLD}=0.05$, specifically, the encoder with the IWAE \cite{burda2015importance} objective with $K=5$ and the decoder with a $\beta$-ELBO  with $\beta = 0.5$. We use a parents dropout rate (see Appendix~\ref{apx:training_implementation}) of $0.2$, learning rate of $0.005$ and batch size $100$.

\paragraph{Classifiers.} Logistic Regresion ($\operatorname{LR}$) and Support Vector Machine ($\operatorname{SVM}$) classifiers are taken from the scikit-learn library and trained with default parameters as well as $\operatorname{class\_weight} = \operatorname{balanced}$ due to the class imbalance of the dataset. 

\newpage
\subsection{Metrics}

In this section we detail the measures f1-score ($\operatorname{f1}$), unfairness ($\operatorname{uf}$) and accuracy ($\operatorname{acc}$) in Table \ref{tab:fairness2}.

\paragraph{f1-score.} Due to class imbalance, we measure classifier performance with the f1-score. The f1-score is the weighted average of the precision and recall and can assume values between 0 and 1; the higher the values the better. Our implementation relies on the $\operatorname{f1\_score}$ from the scikit-learn library. We compute in expectation over our training dataset: 

\begin{align}
    \operatorname{f1}=\Ex \left [2 \times \frac{\operatorname{precision} \times \operatorname{recall}} { \operatorname{precision} + \operatorname{recall}} \right].
\end{align}

where $\operatorname{precision}$ is the ratio $\frac{TP}{TP +  FP}$ with the number of true positives  $TP$ and the number of false positives $FP$ and $\operatorname{recall}$ is the ratio $\frac{TP}{TP +  FN}$ with the number of false positives $FP$.

\paragraph{Counterfactual (un)fairness.} We measure counterfactual unfairness  \cite{kusner2017counterfactualfairness} with counterfactual instances $\xCF$ and classifier prediction $h(\xCF)=\yhCF$ as expectation over our training dataset:

\begin{align}
       &\operatorname{uf} = \Ex \left [ \left| p(\yF)=1|\xF) -  p(\yhCF=1 | do(S = a'), \xF)\right| \right]
\end{align}

where $a' = 1-a$ as $S \in \{0,1\}$.

\paragraph{Accuracy.} In Table \ref{tab:fairness2}, we report additionally the prediction accuracy as performance measure of classifier $h$ with respect to factuals (samples) $(\xF, \yF)$ and prediction $h(\xF)=\yhF$ in expectation over our training dataset: 

\begin{align}
    \operatorname{acc}=\Ex \left [ \mathds{1}\left(\yF_{i}=\yhF_{i}\right) \right ].
\end{align}

Our implementation relies the $\operatorname{accuracy\_score}$ from the scikit-learn library.

\begin{table}
\centering
    \begin{tabular}{ll|rrr|r}
\toprule
        Metr. & 
        CLF & 
        full & unaware & fair-x & fair-z\\
    \midrule
         \multirow{1}{*}{$\uparrow$ \textit{acc} } 
    & LR & 61.00 & 62.00 &  43.00 & 63.60  $\pm$ 4.74 \\ 
    & SVM  & 66.00  & 64.00 & 51.00 & 64.20 $\pm$ 4.52\\
    \midrule
     \multirow{1}{*}{$\uparrow$ \acc } 
    & LR & 67.23 & 68.33 & 47.71 
    & 70.89  $\pm$ 4.11 \\ 
    & SVM 
    & 71.67  & 69.49 & 59.50
    & 70.79 $\pm$ 5.15\\
\hline
    \multirow{1}{*}{$\downarrow$ \uf} 
    & LR & 18.30 \footnotesize{$\pm$ 3.22} & 17.65 $\pm$ 3.20 & 0.16 $\pm$ 0.02 
    & 0.44 $\pm$ 0.24\\ 
    & SVM 
    & 14.01 $\pm$ 2.26 & 13.27 $\pm$ 2.28 & 0.14 $\pm$ 0.02 
    &    0.51 $\pm$ 0.19\\
\bottomrule
\end{tabular}
 \caption{Counterfactual unfairness (\textit{uf}), accuracy (\textit{acc}) and f1-score (\textit{f1}) of a LR and SVM classifier over 10 \name\ seeds. Values multiplied by 100.}\label{tab:fairness2}
\end{table}

\end{document}